\documentclass[12pt]{article}
\usepackage{enumitem}
\usepackage{amsmath}
\usepackage{amssymb}
\usepackage{algorithm}
\usepackage{algpseudocode}
\usepackage{placeins}

\DeclareMathOperator*{\argmin}{arg\,min}
\DeclareMathOperator{\denoiser}{D}
\DeclareMathOperator{\prox}{prox}
\newcommand{\R}{\mathbb{R}}
\usepackage{graphicx} 
\usepackage{placeins} 
\usepackage{makecell}
\usepackage{graphicx}
\usepackage{subcaption}
\usepackage{caption} 
\usepackage{tabularx}
\usepackage{array}
\usepackage{multirow}
\usepackage{amsmath}
\usepackage{multirow}
\usepackage{mathrsfs}
\usepackage{amsfonts}
\usepackage{indentfirst}
\usepackage{colortbl,dcolumn}
\usepackage{color}

\usepackage{amsmath}
\allowdisplaybreaks[4]
\usepackage{psfrag}

\usepackage{array}
\usepackage{booktabs}
\usepackage{cite}
\usepackage[hidelinks]{hyperref}
\usepackage{amsthm}
\allowdisplaybreaks
\numberwithin{equation}{section}
\usepackage[top=1in, bottom=1in, left=0.8in, right=0.8in]{geometry}
\renewenvironment{proof}%
  {\begin{trivlist}%
   \item[\hskip\labelsep]
   \ignorespaces}%
  {\qed\end{trivlist}}

\usepackage{caption}
\usepackage{appendix} 
\usepackage{graphicx}
\usepackage{subcaption}
\usepackage{subcaption}
\usepackage{float}
\usepackage{makecell}
\newcommand{\N}{\mathbb{N}}

\newcommand{\dd}{\text{d}}
\newtheorem{thm}{Theorem}[section]
\newtheorem{defn}[thm]{Definition}

\newtheorem{prop}[thm]{Proposition}

\newtheorem{rem}[thm]{Remark}

\begin{document}
\bibliographystyle{plain}
\title{HP-PnP-Flow: Hybrid Time Scheduling and Projection for Plug-and-Play Image Restoration via Flow Matching\footnotemark[2] {}}
	\author{
    Yong Li,
	Xiaojie Wang,
    Xiaoyan Zhang
		\\
		\footnotesize 
         School of Mathematics and Statistics,   HNP-LAMA, Central South University, Changsha 410083, Hunan, China
		%
		%
		%
	}

\maketitle
\footnotetext[2]{This work was supported by Natural Science Foundation of China (12471394, 12371417, 12071488) and Hunan Basic Science Research Center for Mathematical Analysis (2024JC2002). 
 E-mail addresses:
 15572158033@163.com, \
 x.j.wang7@csu.edu.cn, \
 x.y.zhang@csu.edu.cn.
 }

\begin{abstract}
%
%
Image restoration is a fundamental inverse problem that seeks to recover high-quality images from degraded observations. Plug-and-Play (PnP) methods provide a flexible framework for integrating pretrained denoisers into iterative reconstruction, whereas flow matching offers an efficient generative modeling approach based on continuous-time ODE dynamics.
Very recently, Martin et al.\ (ICLR, 2025) developed a PnP-Flow method, which integrates the PnP framework with flow matching through a time-dependent denoiser constructed from a pretrained flow matching model. 
%
By incorporating a different time schedule and an additional pixel-space projection operator, this paper proposes a novel efficient PnP-Flow algorithm (HP-PnP-Flow) that enjoys rigorous convergence guarantees under mild assumptions, without imposing restrictive conditions on iterates required by the aforementioned work.
Numerical experiments on CelebA and AFHQ-Cat demonstrate that the newly proposed HP-PnP-Flow algorithm significantly outperforms PnP-Flow on several restoration tasks such as denoising, deblurring, super-resolution, random inpainting, and box inpainting.
%
Notably, our algorithm achieves comparable or better reconstruction quality with substantially fewer iterations, reducing the iteration count from $500$ to $100$ for deblurring and super-resolution on AFHQ-Cat.
\\


\end{abstract}

\section{Introduction}
Image restoration has been a long-standing problem in a wide range of applications, including digital photography, medical imaging, remote sensing, surveillance and digital entertainment \cite{singhal2025study,wali2023recent}. 
Given a degraded observation, the objective is to recover the underlying clean image, where the degradation may arise from various factors during image acquisition, transmission, or processing
\cite{kamyab2022survey,9878048,fan2020neural,Elad2023image,jia2024variational}.
Such restoration tasks can typically be formulated as inverse problems under the observation model
\begin{equation}
    y = Hx + \zeta,
    \label{eq:forward_model}
\end{equation}
where $x\in\mathbb{R}^d$ denotes the unknown high-quality image, $y\in\mathbb{R}^m$ is the degraded observation, $H:\mathbb{R}^d\to\mathbb{R}^m$ denotes the degradation operator, and $\zeta$ describes the underlying additive noise.
The inherent ill-posedness and high dimensionality of the problem pose significant challenges to accurate and stable reconstruction.



From a Bayesian perspective \cite{TOPOLNICKI2026109913,latz2025deepgaussianprocesspriors,hein2025physics}, the unknown image $x$ and the observation $y$ are treated as random variables. 
By Bayes' theorem, the corresponding maximum a posteriori (MAP) estimator is then given by
\begin{equation}
\hat{x}_{\mathrm{MAP}}
=
\arg\max_{x\in\mathbb{R}^d}\log p(x\mid y)
=
\arg\min_{x\in\mathbb{R}^d}
\bigl\{-\log p(y\mid x)-\log p(x)\bigr\},
\end{equation}
where $\log p(y\mid x)$ represents the log-likelihood of observation $y$, $\log p(x)$ delivers the prior of clean image $x$ and is independent of degraded image $y$.
Nevertheless, the image prior $p(x)$ is generally unknown or intractable in this framework. To illustrate this point, the negative log-prior is typically replaced by an explicit regularization term, leading to the optimization problem
\begin{equation}
\label{eq:min-optimization}
    \arg \min_{x\in\mathbb{R}^d}
    \bigl\{f(x)+g(x)\bigr\},
\end{equation}
where $f(x)=-\log p(y\mid x)$ denotes the data-fidelity term.
In particular, under additive Gaussian noise
$\zeta \sim \mathcal{N}(0,\sigma^2 \,\mathrm{I}_d)$, the data-fidelity term takes the form $f(x)=\tfrac{1}{2\sigma^2} \|Hx-y\|^2.$



In recent years, the regularized minimization problem \eqref{eq:min-optimization} has received considerable attention.
Early works addressed this problem by using proximal splitting methods, which handle the data-fidelity and regularization terms separately.
For example, forward-backward splitting (FBS)
\cite{prox_sample_1,prox_sample_2} alternates between a proximal operation on $g$ and a gradient descent step on $f$, when $f$ is differentiable. 
In more general settings, half-quadratic-splitting (HQS) \cite{HQS}, alternating direction method of multipliers (ADMM) \cite{ADMM} or douglas-rachford splitting (DRS) \cite{douglas1956numerical} use proximal operations for both $f$ and $g$.\\
{\small\bfseries Plug-and-Play (PnP) method.}
Since designing an effective regularizer for complex priors is challenging, the authors of \cite{venkatakrishnan2013plug}
introduced the Plug-and-Play (PnP) methods, which replace the proximal operator with an off-the-shelf denoiser such as
DnCNN \cite{Zhang_DnCNN} or DRUNet \cite{zhang2021plug}.
%
Over the last few years, much attention has been received on various PnP methods and their convergence analysis (see, e.g., \cite{meinhardt2017learning,terris2020building,pesquet2021learning,ryu2019plug,hurault2022proximal} and references therein).
For example, the authors of \cite{Xu2026ANewConvergence} first provide a new convergence theory for plug-and-play proximal gradient descent (PnP-PGD) under prior mismatch where the denoiser is trained on a different data distribution to the inference task at hand. 
By incorporating neural network denoisers, PnP methods have achieved promising results across various image restoration tasks.
However, classical PnP methods typically rely on generic discriminative denoisers, for which an explicit regularization function may not be available. 
This limits both the optimization interpretation of the method and its ability to represent complex image distributions.
\\

{\small\bfseries PnP-Flow method.}
Generative models provide a natural alternative by modeling the underlying data distribution, and have achieved state-of-the-art performance in image generation, inverse problems and a broad range of related applications \cite{dhariwal2021diffusion,ho2020denoising,songs2021Denoising,chung2023diffusion,cardoso2024monte,lipman2023flow,liu2023Flow}. In particular, flow matching learns
a time-dependent vector field whose ODE flow transports
a base distribution to the data distribution.
Owing to simpler training objectives and more efficient sampling, flow matching has recently emerged as a powerful framework for continuous-time generative modeling.
%
These developments have motivated the use of generative models as image priors within the PnP framework. 
Very recently, Martin et al. \cite{pnp_flow} introduced PnP-Flow matching method by defining a time-dependent denoiser using a pretrained flow matching model.
%
Under regularity assumptions and bounded support of the base distribution, they established convergence under the additional assumptions that the iterates remain bounded and the time sequence satisfies $\sum_{k=0}^{\infty}(1-t_k)<\infty$.
%
%
Despite its good reconstruction performance,
PnP-Flow still has several limitations.
%
On the one hand, the condition imposed on the time sequence is difficult to be satisfied in practice.
%
For instance, the linear schedule $t_k=k/N$ used in \cite{pnp_flow} does not satisfy the required uniform boundedness with respect to the iteration
budget $N$. 
%
On the other hand, the convergence guarantee is obtained based on bounded iterates, a property that is not explicitly guaranteed by the original algorithm.
%
Motivated by this we ask the following question:
\\

{
\it 
Can one design an improved PnP-Flow algorithm that enjoys rigorous convergence guarantees under mild assumptions, while achieving improved reconstruction efficiency and quality?
%
%
}
\\

{\small\bfseries An improved HP-PnP-Flow algorithm.}
In this work, we address this problem by proposing a hybrid-projection PnP-Flow (HP-PnP-Flow) algorithm, based on a different time schedule and an additional
pixel-space projection operator.
More precisely, the hybrid time schedule given by \eqref{eq:hybrid_schedule} combines geometric and linear phases, while the projection operator defined by \eqref{eq:def-Projection} explicitly constrains the iterates to the prescribed pixel range.
Under mild and reasonable assumptions, 
we establish the convergence of HP-PnP-Flow (see Theorem \ref{thm:hp_pnp_flow_convergence}), without requiring the a priori boundedness assumptions on the iterates  imposed in the original PnP-Flow framework in \cite{pnp_flow}. Also, the new hybrid time schedule ensures the resulting time sequence satisfies the uniform boundedness assumption $\sum_{k=0}^{\infty}(1-t_k)<\infty$.
Notably, numerical experiments demonstrate the efficiency and high quality of the new algorithm.
Ablation studies further validate the individual effects of the initial time offset $t_0$ and the hybrid time schedule, with their combination yielding the best overall performance.
In summary, our main contributions are as follows:
\begin{itemize}
    \item We propose an improved PnP-Flow algorithm, called  HP-PnP-Flow, by incorporating a novel hybrid time schedule and an additional pixel-space projection operator. 

    \item We prove the convergence of the newly proposed HP-PnP-Flow algorithm under mild and reasonable assumptions, establishing HP-PnP-Flow as a theoretically grounded algorithm. 

    \item Extensive experiments on two datasets and five tasks demonstrate  that our algorithm  significantly outperforms PnP-Flow in terms of reconstruction quality and convergence efficiency.
\end{itemize}

The rest of this article is organized as follows. 
The next section collects some notations, introduces background on Plug-and-Play (PnP), flow matching and PnP-Flow throughout this paper. 
Section \ref{sec:method} presents the proposed method and main results of this paper.
Numerical experiments are conducted in Section \ref{sec:experiments} to demonstrate the effectiveness of the proposed algorithm.
Finally, some concluding remarks are given in the last section.

\section{Preliminaries}
\label{sec:preliminaries}

\subsection{Notation}
Throughout this paper, we use $\mathbb{N}$ to denote the set of all 
positive integers and let $d\in \mathbb{N}$. 
For any positive integer $m$, we denote $[m]:= \{1,...,m\}$ and $[m]_0:= \{0, 1,...,m\}$.
We let $\mathrm{I}_d \in \mathbb{R}^{d\times d}$ denote the identity matrix
and let $\|\cdot\|$ denote the Euclidean norm of vectors in $\mathbb{R}^d$. 
For a function $f: \mathbb{R}^d \to \mathbb{R}$, we write $\partial_i f$ to denote the $i$-th partial derivative of $f$. The gradient $\nabla f$ is the vector of partial derivative $(\partial_1 f,\partial_2 f,\cdots,\partial_d f)$. 
For a vector field $F=(F_1,F_2,\cdots,F_d): \mathbb{R}^d \to \mathbb{R}^d$, we denote its divergence by $\nabla \cdot F=\sum_{i=1}^{d}\partial_i F_i.$
Let $\mathcal{P}\left(\mathbb{R}^d\right)$ be the space of all probability distributions on $\left(\mathbb{R}^d, \mathcal{B}(\mathbb{R}^d)\right)$. 
By $\cdot_\#$ we denote the push-forward operation.

\subsection{Plug-and-Play (PnP)}

In this work we focus on solving the optimization problem \eqref{eq:min-optimization}. 
When $g$ is nonsmooth, standard gradient descent is generally not directly applicable. 
Proximal splitting methods instead handle the smooth
term $f$ by a gradient step and the regularization term $g$ through its proximal operator. 
%
We next recall the definition of the proximal operator.

\begin{defn}[Proximal operator]
    \label{def:prox}
    Let $g: \R^d \to \R \cup \{+\infty\}$ be a proper lower semicontinuous convex function. The proximal operator of $g$ with step
    size $\gamma > 0$ is the mapping $\prox_{\gamma g}: \R^d \to \R^d$
    defined by
    \begin{equation}
        \prox_{\gamma g}(z)
        := 
        \argmin_{x \in \R^d} 
        \left\{
        g(x) 
        +
        \tfrac{1}{2\gamma}
        \|x - z\|^2 
        \right\}.
    \end{equation}
\end{defn}

The quadratic term makes the proximal subproblem strongly convex, so the minimizer is unique. Intuitively, $\prox_{\gamma g}(z)$ finds a point that
remains close to $z$ while reducing the regularization term $g$.


The Forward-Backward Splitting (FBS) algorithm
\cite{combettes2005signal,prox_sample_1,prox_sample_2} alternates between a
forward gradient step on $f$ and a backward proximal step on $g$:
\begin{equation}
    \begin{cases}
        u_k = x_k - \gamma_k \nabla f(x_k), \\[2pt]
        x_{k+1} = \prox_{\gamma_k g}(u_k).
    \end{cases}
\end{equation}
Here, $\gamma_k \in (0,2/L)$ denotes the step size, where $L$ is the Lipschitz constant of $\nabla f$. Under standard assumptions for FBS, this step size condition ensures convergence of the iterates to a minimizer (see \cite{prox_sample_1} for details).
However, designing an explicit regularizer $g$ that accurately captures complex image priors is often challenging in practice.
The Plug-and-Play (PnP) framework \cite{venkatakrishnan2013plug} replaces the proximal step in this iteration with a pretrained denoiser $\denoiser_\sigma: \R^d \to \R^d$, where $\sigma$ controls the denoising strength, thereby incorporating an implicit image prior without requiring an explicit regularizer $g$.
The resulting PnP-FBS iteration is given by

\begin{equation}
    \begin{cases}
        u_k = x_k - \gamma_k \nabla f(x_k), 
        \\[2pt]
        x_{k+1} = \denoiser_{\sigma}(u_k).
    \end{cases}
    \label{eq:pnp_fbs}
\end{equation}
PnP combines the flexibility of optimization-based methods with the expressive power of neural networks, leading to strong performance across a wide range of image restoration tasks.

\subsection{Flow Matching}
Flow matching \cite{lipman2023flow,Albergo2023Building} is a generative modeling framework that learns a time-dependent vector field to transport samples from a base 
distribution $\nu \in \mathcal{P}(\mathbb{R}^d ) $ (e.g., standard  Gaussian) to the data distribution $\mu \in \mathcal{P}(\mathbb{R}^d)$.
The associated flow $\phi:[0,1] \times  \R^d\to\R^d$ is governed by the ordinary differential equation (ODE):
\begin{equation}
\label{eq:flow_ode}
    \frac{\dd}{\dd t}\,\phi_t(x) 
    = 
    v_t(\phi_t(x)),
    \qquad \phi_0(x) = x_0.
\end{equation}
Once the velocity field $v$ is learned,  sampling from the data distribution only requires solving ODE \eqref{eq:flow_ode} to attain $x_1 \sim \mu$ (see \cite{chen2018Neural} for details).
In addition, the flow $\phi_t$ transports the initial distribution
$p_0$ into a probability density path $p_t=(\phi_t)_\#p_0$, which satisfies the continuity (Liouville) equation
\begin{equation}
    \partial_t p_t(x)
    +
    \nabla\cdot\bigl(p_t(x)v_t(x)\bigr)
    =
    0.
\end{equation}
This equation characterizes the evolution of the density induced by the time-dependent vector field $v_t$ (see \cite{ambrosio2005gradient}).
%
%
The goal of flow matching is to learn the velocity field of the flow ODE \eqref{eq:flow_ode}. 
This learning process consists in minimizing the loss
function
\begin{equation}
    \mathcal{L}_{\mathrm{FM}}(\theta)
    :=
    \mathbb{E}_{t\sim\mathcal{U}[0,1],\,x\sim p_t}
    \left[
        \left\|
        v_t^\theta(x)-v_t(x)
        \right\|^2
    \right],
\end{equation}
where $v_t^\theta$ is parameterized by a neural network with weights $\theta$. In practice, however, the marginal velocity field $v_t(x)$ is
generally unavailable, making the flow matching objective intractable to evaluate directly. To circumvent this difficulty, flow matching instead
relies on a tractable conditional formulation.

Given a data sample $x_1 \sim p_1$, one can construct a conditional probability path $p_t(x \mid x_1)$ that interpolates between the base distribution $p_0$ at $t = 0$ and the target distribution $p_1$ at $t = 1$, together with an associated conditional vector field
$v_t(\cdot\mid x_1)$ \cite{pnp_flow,liu2023Flow}.
For the linear interpolation
\begin{equation}
    x_t
    = 
    (1 - t)\,x_0 + t\,x_1, 
    \qquad x_0 \sim p_0,\; x_1 \sim p_1,
\end{equation}
the corresponding conditional vector field along the interpolation path is 
$v_t(x_t \mid x_1)
= x_1 - x_0$. 
A key result of \cite{lipman2023flow} is that the conditional flow matching (CFM) objective has the same gradient with respect to the model parameters as the corresponding marginal flow matching objective.
Consequently, minimizing the CFM objective allows the neural vector field $v^\theta$ to learn the associated marginal vector field, while requiring only samples of the endpoint pair $(x_0,x_1)$ and the corresponding conditional vector field. For the linear interpolation introduced above, the CFM loss takes the form
\begin{equation}
    \mathcal{L}_{\text{CFM}}(\theta)
    := 
    \mathbb{E}_{t \sim \mathcal{U}[0,1],\,(x_0, x_1) \sim \pi}
    \Bigl[ 
    \bigl\| 
    v^\theta_t(x_t) - (x_1 - x_0) 
    \bigr\|^2 
    \Bigr],
    \label{eq:cfm_loss}
\end{equation}
where $\pi$ denotes a coupling between the base distribution $p_0$ and the target distribution $p_1$, i.e., a joint distribution on $(x_0,x_1)$ whose marginals are $p_0$ and $p_1$, respectively.
The independent coupling $\pi=p_0\otimes p_1$ corresponds to sampling $x_0$ and $x_1$ independently from $p_0$ and $p_1$, respectively. 
Such independent pairing may lead to intersecting transport paths and thereby increase the complexity of the velocity field to be learned. 

For the linear interpolation above, the conditional velocity is
$x_1-x_0$, and the corresponding marginal vector field is given by
\[
v_t(x)
=
\mathbb{E}[x_1-x_0\mid x_t=x].
\]
A fundamental result of flow matching \cite{lipman2023flow} shows that this marginal vector field can be learned through the tractable conditional flow matching objective.
To illustrate the transport process induced by the learned flow matching model, Figure~\ref{fig:flow_matching} visualizes the evolution of samples along the learned ODE trajectories on CelebA.

\begin{figure}[htp]
    \centering
    \includegraphics[width=0.95\textwidth]
    {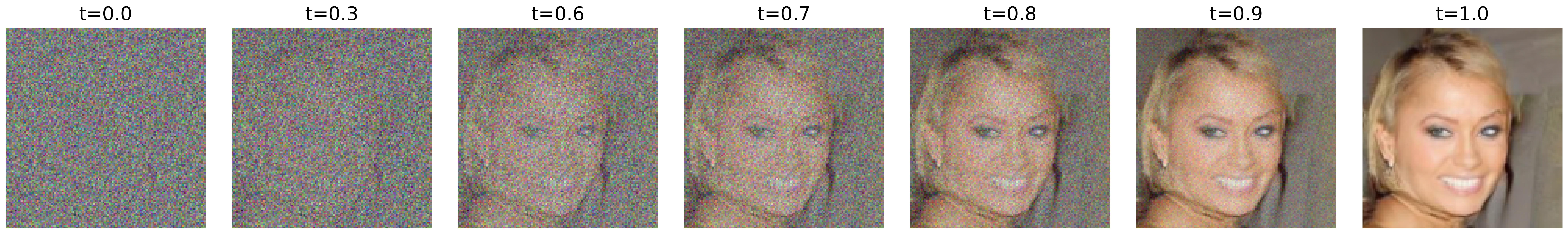}
    \caption{Evolution of samples along the learned flow matching trajectories on CelebA at different time points. The ODE \eqref{eq:flow_ode} is integrated using the forward Euler method with a step size of $0.01$.}
    \label{fig:flow_matching}
\end{figure}

\subsection{PnP-Flow: Plug-and-Play with flow matching}
Given a pretrained flow matching vector field $v^\theta$, an estimate of the clean endpoint $x_1$ can be constructed from an intermediate state $x_t$ along the prescribed interpolation path.
We first give the following definition.

\begin{defn}[Flow matching denoiser {\cite{pnp_flow}}]
\label{def:fm_denoiser}
Given a pretrained flow matching vector field $v^\theta$, the
time-dependent denoiser $D_t:\R^d\to\R^d$ is defined by
\begin{equation}
\label{eq:def-denoiser}
    D_t(x)
    :=
    x+(1-t)v^\theta_t(x),
    \qquad t\in[0,1].
\end{equation}
\end{defn}

Next, we explain the motivation behind the choice of the denoiser in \eqref{eq:def-denoiser}.
For the linear interpolation
\[
    x_t=(1-t)x_0+t \, x_1,
    \qquad (x_0,x_1)\sim\pi.
\]
For this interpolation, the minimizer $v_t^*$ of the CFM loss \eqref{eq:cfm_loss} over all
admissible vector fields is given by
\[
    v_t^*(x)
    =
    \mathbb{E}_{(x_0,x_1)\sim\pi}
    \left[
        x_1-x_0
        \,\middle|\,
        x_t=x
    \right].
\]
In practice, $v_t^*$ is replaced by the learned vector field $v^\theta$, yielding the practical denoiser $D_t$. Therefore, the corresponding ideal denoiser satisfies, for any $x \in \R^d$ and $t \in [0,1]$,
\begin{equation}
\begin{aligned}
\label{eq:fm_mmse}
D_t(x)
&=
x+(1-t)v_t^*(x)\\
&=
\mathbb E_{(x_0, x_1) \sim \pi}
\left[
x_t+(1-t)(x_1-x_0)
\,\middle|\,
x_t=x
\right] \\
&=
\mathbb E_{(x_0, x_1) \sim \pi}
\left[
x_1
\,\middle|\,
x_t=x
\right].
\end{aligned}
\end{equation}
Hence, $D_t(x)=\mathbb{E}_{(x_0, x_1) \sim \pi}[x_1\mid x_t=x]$ is the minimum mean-square error (MMSE) estimator of the clean endpoint $x_1$ given $x_t=x$, which is also used in    \cite{pokle2024trainingfree,zhang2024flow}.\\



\noindent
\textbf{PnP-Flow algorithm.} 
Recently in \cite{pnp_flow}, the authors incorporate the time-dependent denoiser $D_t$ defined in \eqref{eq:def-denoiser} into the FBS framework, leading to
\begin{equation}
\label{eq:pnp-flow}
\begin{cases}
u_k = x_k - \gamma_k \nabla f(x_k), 
\\[2pt]
y_k = (1-t_k)\xi_k + t_k u_k,
\\[2pt]
x_{k+1}=D_{t_k}(y_k),
\end{cases}
\end{equation}
where $\{\xi_k\}_{k\in[N-1]_0}$ is a sequence of i.i.d. latent samples drawn from the base distribution $p_0$.
This iteration differs from the classical PnP-FBS scheme
\eqref{eq:pnp_fbs} in that the denoiser $D_t$ explicitly depends on the time, and the gradient update is first reprojected along the interpolation path before denoising.
For clarity, we sketch the above PnP-Flow algorithm  in {\bf Algorithm \ref{alg:pnpflow}}.

\FloatBarrier
\begin{algorithm}[htbp]
  \centering
  \caption{PnP-Flow iteration.}
  \label{alg:pnpflow}
  \begin{algorithmic}[1]
    \Require 
    Initialize $x_0$,
             the total iterations $N$,
             step sizes $\{\gamma_k\}_{k \in[N-1]_0}$,
             time sequence $\{t_k\}_{k \in[N-1]_0}$ with $t_0 = 0$.
    \Ensure Reconstructed image $x_N$
    \For{$k = 0$ \textbf{to} $N - 1$}
      \State $u_k = x_k - \gamma_k \nabla f(x_k)$
        \Comment{Gradient step}
      \State $y_k = (1 - t_k)\,{\xi}_k + t_k\,u_k$,
        \quad ${\xi}_k \sim p_0$
        \Comment{Interpolation}
      \State $x_{k+1} = D_{t_k}(y_k)$
        \Comment{Flow matching denoising}
    \EndFor
    \State \Return $x_N$
  \end{algorithmic}
\end{algorithm}

\emph{Choice of the step size.}
To balance the contributions of the data-fidelity update and the flow matching denoiser, the authors of \cite{pnp_flow} adopt the following time-dependent step size:
\begin{equation}
    \gamma_k=(1-t_k)^\alpha,
    \qquad \alpha\in(0,1].
\end{equation}
Thus, the step size gradually decreases with $t$, reducing the influence of the data-fidelity update in the later stages of the iteration; see
\cite{pnp_flow} for further motivation of this choice. This choice yields favorable numerical performance in practice.

Recalling the definition of the denoiser
$D_t(x)
:=
x+(1-t)v^\theta_t(x)$, 
we observe that the factor $1-t$ progressively attenuates the correction induced by the flow-matching prior as $t \to 1$. Consequently, $D_t$ approaches the identity mapping, with $D_1=\mathrm{I}_d$.
This provides a flow-time-dependent mechanism for controlling the strength of the learned prior, in contrast to conventional PnP methods based on predefined denoising operators \cite{venkatakrishnan2013plug,chan2016plug}.
Substituting the definition of $D_{t_k}$ into the last equation of \eqref{eq:pnp-flow} yields the following compact update rule:
\begin{equation}
    x_{k+1} 
    = 
    (1 - t_k)\,{\xi}_k
    + t_k\bigl[x_k - \gamma_k \nabla f(x_k)\bigr]
    + (1 - t_k)\,v^\theta_{t_k}(y_k),
\end{equation}
where the first term accounts for the contribution
of the base noise in the interpolation step, the second term incorporates the gradient-based data-fidelity update, and the third term provides the correction induced by the learned flow-matching prior.


Under some regularity assumptions, PnP-Flow algorithm \eqref{eq:pnp-flow} admits the following convergence result (see \cite[Proposition 4]{pnp_flow}).

\begin{prop}[Convergence of PnP-Flow]
    \label{prop:pnpflow_convergence}
    Assume that $f: \R^d \to \R$ is continuously differentiable. The
    learned vector field $v^\theta: [0, 1] \times \R^d \to \R^d$ is continuous,
    and the base distribution $p_0$ has bounded support. Let the time
    sequence $\{t_k\}_{k \in \N}$ satisfy $\sum_{k=0}^\infty (1 - t_k) <
    +\infty$ and set $\gamma_k := 1 - t_k, \, k \in \N$. If the sequence $\{x_k\}_{k \in
    \N}$ produced by PnP-Flow algorithm \eqref{eq:pnp-flow} is bounded, then it is convergent.
\end{prop}

Despite its effectiveness, PnP-Flow algorithm \eqref{eq:pnp-flow} has two limitations.
\begin{enumerate}
\renewcommand{\labelenumi}{\arabic{enumi})}
\item First, the convergence result in Proposition \ref{prop:pnpflow_convergence}
requires the time sequence to satisfy
$
\sum_{k=0}^{\infty}(1-t_k)<+\infty.
$
However, the practical implementation of PnP-Flow algorithm \eqref{eq:pnp-flow} employs the 
linear schedule
$
t_k=\frac{k}{N},
\ k \in [N]_0.
$
For this time sequence,
$
\sum_{k=0}^{N}(1-t_k)
=
\sum_{k=0}^{N}\tfrac{N-k}{N}
=
\tfrac{N+1}{2},
$
which goes to infinity as $N\to\infty$. Hence, the family of linear time sequences is not uniformly bounded with respect to $N$, and therefore does not satisfy the convergence condition.

\item Second, PnP-Flow algorithm \eqref{eq:pnp-flow} sets
$t_0=0$, in which case the first interpolation step reduces to $y_0={\xi}$ and is therefore independent of the data-fidelity update $x_0-\gamma_0\nabla f(x_0)$. While this choice makes the first iteration
independent of the initialization (see \cite{pnp_flow}), it also prevents the measurement information encoded by the data-fidelity gradient from entering the first denoising step.
\end{enumerate}
Next we aim to propose an improved PnP-Flow algorithm to address the above issues.

\section{An efficient Hybrid-Projection PnP-Flow}
\label{sec:method}




Before presenting our proposed algorithm, we first introduce the following hybrid time schedule.\\

\emph{Hybrid time schedule.}
To address the first issue mentioned above, we introduce a switching point $S := \lfloor K N \rfloor$, where $N$ denotes the total number of iterations and $K \in (0, 1)$ determines the switching point between the two phases. 
Regarding the second issue, we introduce a tunable initial offset $t_0\in[0,1)$, which allows the data-fidelity update to contribute from the first iteration when $t_0>0$.
More precisely, for $\lambda\in(0,1)$, we define the hybrid time schedule by
\begin{equation}
\label{eq:hybrid_schedule}
    t_{k}^{(N)}
    :=
    \begin{cases}
        1 - (1 - t_0)\,\lambda^{k}, 
        & 0 \leq k \leq S ,  \\[6pt]
        t_{S}^{(N)}+
        \tfrac{k - S}
        {\,N  - S\,}\,(1 - t_{S}^{(N)}), 
        & S \leq k \leq N .
    \end{cases}
\end{equation}

We notice that the time schedule $\{t_{k}^{(N)}\}_{k \in [N]_0}$ defined by \eqref{eq:hybrid_schedule} consists of a geometric growth and linear growth phase, differing from the linear time schedule used in \cite{pnp_flow}. 
When $0 \leq k \leq S$, $t_{k}^{(N)}$ approaches $1$ geometrically, so that the noise weight $1-t_{k}^{(N)}$ decays exponentially. This rapidly reduces the influence of the base noise and promotes fast recovery of the dominant image structure. When $S \leq k \leq N$,  $t_{k}^{(N)}$ increases linearly from $t_S^{(N)}$ to $1$, providing a more gradual transition toward the final reconstruction and facilitating stable refinement.

The following result establishes a uniform bound and vanishing-tail property for the hybrid time schedule with respect to the iteration budget $N$.

\begin{prop}[Uniform boundedness and vanishing-tail property for the hybrid time schedule]
    \label{prop:hybrid_convergence}
    The time schedule $\{t_{k}^{(N)}\}_{k \in [N]_0}$ defined by
    \eqref{eq:hybrid_schedule} satisfies
    $\sup_{N \geq 1}
    \sum_{k=0}^{N} \big(1 - t_{k}^{(N)}\big) < +\infty$.
    Moreover, for any $q\in (0,1]$,
    $
    \lim_{N\to\infty}
    \sum_{k=S}^{N}\big(1-t_k^{(N)}\big)^q=0.
    $
\end{prop}

\begin{proof}
\textbf{Proof.}
For the geometric phase $0\leq k\leq S-1$, the definition
\eqref{eq:hybrid_schedule} gives
\[
\sum_{k=0}^{S-1}(1-t_{k}^{(N)})
=
(1-t_0)\sum_{k=0}^{S-1}\lambda^k
=
(1-t_0)\tfrac{1-\lambda^S}{1-\lambda}.
\]
For the linear phase, i.e., $S\leq k\leq N$, we similarly obtain from \eqref{eq:hybrid_schedule} that
\[
\sum_{k=S}^{N}(1-t_{k}^{(N)})
=
(1-t_S^{(N)})
\sum_{k=S}^{N}
\left(
1-\tfrac{k-S}{N-S}
\right)
=
\tfrac{N-S+1}{2}(1-t_S^{(N)}).
\]
Together with the definition of $t_S$, we have 
\[
\sum_{k=0}^{N}(1-t_{k}^{(N)})
=
(1-t_0)
\Big(
\tfrac{1-\lambda^S}{1-\lambda}
+
\tfrac{N-S+1}{2}\lambda^S
\Big).
\]
For a fixed $K\in(0,1)$, since $S=\lfloor KN\rfloor$, we have
$N-S+1\leq N+1<\frac{S+1}{K}+1$. Moreover,
\[
(S+1)\lambda^S
\leq
\sum_{j=0}^{\infty}(j+1)\lambda^j
=
\tfrac{1}{(1-\lambda)^2},
\]
and $\lambda^S < 1$. Therefore,
\[
(N-S+1)\lambda^S
\leq
\tfrac{1}{K(1-\lambda)^2}+1.
\]
Combining the above estimates yields
\[
\begin{aligned}
\sum_{k=0}^{N}(1-t_k^{(N)})
&\leq
(1-t_0)
\left[
\tfrac{1}{1-\lambda}
+
\tfrac{1}{2K(1-\lambda)^2}
+
\tfrac{1}{2}
\right].
\end{aligned}
\]
The right-hand side is independent of $N$. Taking the supremum
over $N\geq1$ gives
\[
\sup_{N\geq1}
\sum_{k=0}^{N}\big(1-t_k^{(N)}\big)
\leq
(1-t_0)
\left[
\tfrac{1}{1-\lambda}
+
\tfrac{1}{2K(1-\lambda)^2}
+
\tfrac{1}{2}
\right]
<+\infty.
\]
The first assertion is thus validated. Next, we aim to prove the second assertion. For the linear phase, we have
\begin{equation}
\label{eq:second-assertion}
\sum_{k=S}^{N}
\big(1-t_k^{(N)}\big)^q
=
(1-t_0)^q 
\lambda^{q S}
\sum_{k=S}^{N}
\big(
\tfrac{N-k}{N-S}
\big)^q
\leq 
(1-t_0)^q 
(N-S+1)
\lambda^{q S}.
\end{equation}
Since $S=\lfloor KN\rfloor \to\infty$
as $N\to\infty$, $K\in(0,1)$,
and $\lambda\in(0,1)$, it follows that
\[
\lim_{N\to\infty}(N-S+1)\lambda^{q S}=0.
\]
Combining this with \eqref{eq:second-assertion} yields
\[
\lim_{N\to\infty}
\sum_{k=S}^{N}
\big(1-t_k^{(N)}\big)^q=0.
\]
This completes the second assertion.
\end{proof}

Image intensities are commonly represented in $[0,1]$ and are typically normalized to $[-1,1]$ during training. However, PnP-Flow algorithm \eqref{eq:pnp-flow} does not guarantee that the intermediate iterates remain within this normalized range. Pixel values outside the valid range may lead to numerical instability and degrade reconstruction quality. \\

\emph{Pixel-space projection.}
To overcome this difficulty, we impose a box constraint by applying a component-wise
projection operator at the end of each iteration.

\begin{defn}[Projection operator]
    \label{def:projection}
    For any vector $x \in \R^d$, the projection operator
    $\mathcal{P}: \R^d \to [-1, 1]^d$ is defined component-wise as, for any $i \in [d]$,
    \begin{equation}
    \label{eq:def-Projection}
        \mathcal{P}(x)_i 
        := 
        \max\bigl(-1,\; \min(1,\; x_i)\bigr).
    \end{equation}
\end{defn}

The projection operator possesses two properties, which are crucial to the subsequent analysis:
\begin{enumerate}[label=(\roman*)]
    \item \textit{Non-expansiveness.} For any $a, b \in \R^d$,
        $\|\mathcal{P}(a) - \mathcal{P}(b)\| \leq \|a - b\|$, i.e.,
        the projection operator $\mathcal{P}$ is $1$-Lipschitz.
    \item 
    \textit{Identity on the constraint set.} $\mathcal{P}(x) = x$ if and
        only if $x \in [-1, 1]^d$; the projection modifies only those
        components that violate the bound.
\end{enumerate}


By combining the hybrid time schedule \eqref{eq:hybrid_schedule} with the projection operator $\mathcal{P}$ defined in \eqref{eq:def-Projection}, we are able to propose our new efficient Hybrid-Projection PnP-Flow (HP-PnP-Flow) algorithm.\\

\noindent\textbf{HP-PnP-Flow algorithm.}
In this paper, we introduce an improved variant of PnP-Flow,
termed HP-PnP-Flow:
\begin{equation}
\begin{aligned}
\label{eq:hp-pnp-flow}
\begin{cases}
u_k = x_k - \gamma_k \nabla f(x_k), 
\\[2pt]
y_{k} = (1-t_k){\xi}_k
+
t_k u_k,
\\[2pt]
x_{k+1}=\mathcal{P}\bigl(D_{t_k}(y_k)\bigr),
\end{cases}
\end{aligned}
\end{equation}
where $\xi_k \sim p_0$.
For a more intuitive presentation, the complete iteration is given in {\bf Algorithm \ref{alg:hp_pnp_flow}}.

\FloatBarrier
\begin{algorithm}[htbp]
  \centering
  \caption{HP-PnP-Flow iteration.}
  \label{alg:hp_pnp_flow}
  \begin{algorithmic}[1]
    \Require Initialize $x_0$,
             the total iterations $N$,
             step sizes $\{\gamma_k\}_{k \in[N-1]_0}$,
             time sequence $\{t_k\}_{k \in[N-1]_0}$ with $t_0 \in [0,1)$.
    \Ensure Reconstructed image $x_N$
    \For{$k = 0$ \textbf{to} $N - 1$}
      \State $u_k = x_k - \gamma_k \nabla f(x_k)$
        \Comment{Gradient step}
      \State $y_k = (1 - t_k)\,{\xi}_k + t_k\,u_k$,
        \quad ${\xi}_k \sim p_0$
        \Comment{Interpolation}
      \State $x_{k+1} = \mathcal{P}\bigl(D_{t_k}(y_k)\bigr)$
        \Comment{Denoising + projection}
    \EndFor
    \State \Return $x_N$
  \end{algorithmic}
\end{algorithm}

\emph{Choice of the step size.} By introducing a scaling parameter $\eta \in (0,+\infty)$, we adopt the following time-dependent step size in this paper:
\begin{equation}
    \label{eq:gamma_hp}
    \gamma_k=\eta(1-t_k)^\alpha,
    \qquad \alpha\in(0,1].
\end{equation}
In the special case $\eta=1$, the above step size with hyperparameter reduces to the usual step size in the literature \cite{pnp_flow}. 
The scaling hyperparameter $\eta$ provides additional control over the overall step magnitude without altering the decay profile determined by $\alpha$, thereby allowing the step size to be adapted to different degradation tasks. 
%
We next establish the convergence of the HP-PnP-Flow algorithm.

\begin{thm}[Convergence of HP-PnP-Flow]
\label{thm:hp_pnp_flow_convergence}
Assume that $f:\R^d\to\R$ is continuously differentiable, that the learned vector field $v^\theta:[0,1]\times\R^d\to\R^d$ is continuous, and that the base distribution $p_0$ has bounded support. For each iteration budget $N$, let $\{x_k^{(N)}\}_{k\in [N]_0}$ be generated by HP-PnP-Flow
\eqref{eq:hp-pnp-flow} with the hybrid time schedule
$\{t_k^{(N)}\}_{k\in [N]_0}$ defined by
\eqref{eq:hybrid_schedule} and
$
\gamma_k^{(N)}
$ defined by \eqref{eq:gamma_hp}.
Assume that the same initialization and the same latent sequence
$\{\xi_k\}_{k\in\N}$ are used for all $N$.
Then the sequence $\{x_N^{(N)}\}_{N \in \N}$ generated by HP-PnP-Flow \eqref{eq:hp-pnp-flow} is convergent.
\end{thm}

\begin{proof}
\textbf{Proof.}
Since $x_k^{(N)}=\mathcal{P}(D_{t^{(N)}_{k-1}}(y^{(N)}_{k-1})) \in [-1,1]^d$, using the triangle inequality, identity on the constraint set and the Lipschitz condition of $\mathcal{P}$ gives, for any $ k \geq 1$,
\begin{equation*}
        \big\|x^{(N)}_{k+1} - x^{(N)}_k\big\|
        =
        \big\|
        \mathcal{P}(D_{t^{(N)}_{k}}(y^{(N)}_k)) - 
        \mathcal{P}(x^{(N)}_k )
        \big\|
        \leq 
        \big\|D_{t^{(N)}_k}(y^{(N)}_k) - x^{(N)}_k\big\|.
\end{equation*}
From 
HP-PnP-Flow algorithm \eqref{eq:hp-pnp-flow}, one can see that
\begin{equation}
\begin{aligned}
& 
\big\|
D_{t^{(N)}_k}(y^{(N)}_k) - x^{(N)}_k
\big\| \\
&= 
\bigl\| 
(1 - t^{(N)}_k)\,{\xi}_k
+ t^{(N)}_k\,(x^{(N)}_k - \eta (1 - t^{(N)}_k)^\alpha \,\nabla f(x^{(N)}_k))
+ (1 - t^{(N)}_k)\,v_{t^{(N)}_k}^\theta(y^{(N)}_k) - x^{(N)}_k 
\bigr\| 
\\[4pt]
& \leq 
\big(1 - t^{(N)}_k\big)\,
\bigl\|
{\xi}_k - x_k^{(N)}
+ v_{t^{(N)}_k}^\theta(y^{(N)}_k) 
\bigr\|
+
\eta \big(1 - t^{(N)}_k\big)^\alpha\,
\big\|
t^{(N)}_k\,\nabla f(x^{(N)}_k)
\big\|.   
\end{aligned}
\end{equation}
Under the continuity assumptions on $\nabla f$ and $v^\theta$, the boundedness of $\{x^{(N)}_k\}$ in $[-1, 1]^d$ and the bounded support of ${\xi}_k$, there exists a constant $M >0$, independent of $k$ and $N$, such that 
\[
\big\|
{\xi}_k - x^{(N)}_k
+ v_{t^{(N)}_k}^\theta(y^{(N)}_k) 
\big\|
+
\eta \,
\big\|
t^{(N)}_k\,\nabla f(x^{(N)}_k)
\big\|
\leq 
M.
\]
As a result, we have, for any $k \geq 1$,
\begin{equation}
\label{eq:error}
\big\|x_{k+1}^{(N)}-x_k^{(N)}\big\|
\leq 
M\big(1-t_k^{(N)}\big)^\alpha.
\end{equation}
According to Proposition \ref{prop:hybrid_convergence},
for any $\epsilon>0$, there exists $N^*\geq 1$
such that, for any $N\geq N^*$,
\begin{equation}
\label{eq:remainder-bound}
\sum_{k=S}^{N}
\bigl(1-t_k^{(N)}\bigr)^\alpha
<
\tfrac{\epsilon}{3M},
\qquad S=\lfloor KN\rfloor.
\end{equation}
Let $N_2>N_1\geq N^*$ and set
$S_i=\lfloor KN_i\rfloor$ for  $i \in [1,2]$.
Since $S_1\leq S_2$ and the same initialization and latent sequence are used for all iteration budgets, we have
\begin{equation}
\label{eq:x_n1-x_n2}
x_{S_1}^{(N_1)}=x_{S_1}^{(N_2)}.
\end{equation}
Using the triangle inequality,
\eqref{eq:x_n1-x_n2}, and \eqref{eq:error},
we obtain
\begin{equation}
\begin{aligned}
\big\|
x_{N_1}^{(N_1)}-x_{N_2}^{(N_2)}
\big\|
& \leq
\big\|x_{N_1}^{(N_1)}-x_{S_1}^{(N_1)}\big\|
+
\big\|x_{N_2}^{(N_2)}-x_{S_1}^{(N_2)}\big\|\\
& \leq
M\sum_{k=S_1}^{N_1-1}
\bigl(1-t_k^{(N_1)}\bigr)^\alpha
+
M\sum_{k=S_1}^{N_2-1}
\bigl(1-t_k^{(N_2)}\bigr)^\alpha\\
& =
M\sum_{k=S_1}^{N_1-1}
\bigl(1-t_k^{(N_1)}\bigr)^\alpha
+
M\sum_{k=S_1}^{S_2-1}
\bigl(1-t_k^{(N_2)}\bigr)^\alpha+
M\sum_{k=S_2}^{N_2-1}
\bigl(1-t_k^{(N_2)}\bigr)^\alpha.
\end{aligned}
\end{equation}
By \eqref{eq:remainder-bound}, the first
and third terms are each bounded by
$\epsilon/3$. For the second term,
the geometric phase of the hybrid schedule
gives
\[
\begin{aligned}
M\sum_{k=S_1}^{S_2-1}
\bigl(1-t_k^{(N_2)}\bigr)^\alpha
&=
M(1-t_0)^\alpha
\sum_{k=S_1}^{S_2-1}\lambda^{\alpha k}
\leq 
\tfrac{M(1-t_0)^\alpha}{1-\lambda^\alpha}
\lambda^{S_1 ^\alpha}.
\end{aligned}
\]
Since $S_1=\lfloor KN_1\rfloor\to\infty$
as $N_1\to\infty$, and $\lambda \in (0,1)$, there exists
$N_0 \geq 1$ sufficiently large such that, for any $N_1 \geq N_0$,
\[
\tfrac{M(1-t_0)^\alpha}{1-\lambda^\alpha}
\lambda^{\alpha\lfloor KN_1\rfloor}
<\tfrac{\epsilon}{3}.
\]
Consequently, for any $N_2>N_1\geq \max\{N_0, N^\star\}$,
\[
\big\|
x_{N_1}^{(N_1)}-x_{N_2}^{(N_2)}
\big\|<\epsilon.
\]
Hence, $\{x_N^{(N)}\}_{N\in\mathbb N}$
is a Cauchy sequence in $\mathbb R^d$
and therefore converges.
\end{proof}

\textbf{Comparison to PnP-Flow.}
Our paper is clearly motivated by \cite{pnp_flow} and aims to achieve more promising performance on various inverse problems. In the following, we compare HP-PnP-Flow in this work with PnP-Flow and summarize the main improvements. 

\begin{enumerate}[label=(\roman*)]
    \item \textit{Time schedule design.}  The linear schedule used in \cite{pnp_flow} is not uniformly summable with respect to the iteration budget $N$, whereas the proposed hybrid schedule \eqref{eq:hybrid_schedule} satisfies the uniform summability properties established in Proposition \ref{prop:hybrid_convergence}.

    \item \textit{Pixel-space projection.} The convergence analysis of PnP-Flow in \cite[Proposition 4]{pnp_flow} assumes bounded iterates of $\{x_k\}$. In our algorithm, the projection $\mathcal{P}$ enforces $x_k\in[-1,1]^d$ at every iteration, thereby guaranteeing boundedness of the iterates directly from the algorithmic construction.

    \item \textit{Step size design.} The step size in \cite{pnp_flow} is chosen as $\gamma_k=(1-t_k)^\alpha$, whereas we adopt the scaled form
    $\gamma_k=\eta(1-t_k)^\alpha$ with $\eta\in (0,+\infty)$, allowing the overall step magnitude to be adjusted independently of the decay profile.

    \item \textit{Initial time offset.}
    The PnP-Flow algorithm \cite{pnp_flow} adopts $t_0=0$, rendering the first interpolation step independent of the data-fidelity update. We instead introduce a tunable initial offset $t_0 \in [0,1)$ in our algorithm, which enables measurement information to influence the first denoising step when $t_0>0$.
\end{enumerate}

\section{Numerical experiments}
\label{sec:experiments}
In this section, we evaluate the performance of HP-PnP-Flow algorithm \eqref{eq:hp-pnp-flow} on five canonical image restoration tasks, including denoising, deblurring, super-resolution, random-pixel inpainting, and box inpainting, using two publicly available datasets.
We compare our algorithm with several baseline methods using PSNR, SSIM, and LPIPS, and perform ablation studies to assess the effects of the hybrid time schedule and the initial time offset.

\subsection{Experimental Setup}

\noindent
\textbf{Datasets and preprocessing.}
We use two datasets with complementary characteristics.

\begin{itemize}
    \item \textbf{CelebA} contains approximately $200,000$ celebrity face images (see \cite{liu2015deep}). All images are center-cropped and resized to $128 \times 128$ pixels. 
    The dataset provides well-structured facial
     patterns and diverse texture details, making it suitable for evaluating image restoration methods.
    \item \textbf{AFHQ-Cat} is the cat subset of the Animal Faces HQ dataset (see \cite{choi2020stargan}), comprising roughly $5\,000$ images. Images are resized to $256 \times 256$ pixels. The higher resolution and fine fur detail provide a more challenging test of generalization.
\end{itemize}
Both datasets use standard training, validation, and test splits. All quantitative evaluations are performed on the test set. Pixel values are
normalized to $[-1, 1]$.\\

\noindent\textbf{Degradation tasks.}
We consider five degradation tasks. In all cases, additive Gaussian noise is included, with the task-dependent noise level specified below.
\begin{itemize}
    \item \textbf{Denoising:}
    $H=\mathrm{I}_d$, with $\sigma=0.2$.

    \item \textbf{Deblurring:}
    $H$ is a Gaussian blurring operator with
    $b=1.0$ on CelebA and $b=3.0$ on AFHQ-Cat, with $\sigma=0.05$.

    \item \textbf{Super-resolution:}
    $H$ is a bicubic downsampling operator with scale factor
    $\times 2$ on CelebA and $\times 4$ on AFHQ-Cat, with $\sigma=0.05$.

    \item \textbf{Random inpainting:}
    $H$ is a random binary masking operator retaining $30\%$ of the pixels,
    with $\sigma=0.01$.

    \item \textbf{Box inpainting:}
    $H$ masks a central $40\times40$ region on CelebA and an
    $80\times80$ region on AFHQ-Cat, with $\sigma=0.05$.
\end{itemize}


\noindent\textbf{Baseline methods.}
We compare our algorithm against several existing methods. 
PnP-Flow \cite{pnp_flow} is used as our primary baseline.
We additionally include three flow-matching-based restoration methods, namely OT-ODE \cite{pokle2024trainingfree}, D-Flow \cite{D-Flow}, and
Flow-Priors \cite{zhang2024flow}, together with PnP-GS \cite{pnp_gs}, a PnP-based baseline.
All baselines are evaluated using their optimal hyperparameters and officially released weights under identical degradation settings and test splits for
fair comparison.\\

\noindent\textbf{Implementation details.}
All experiments are conducted on an NVIDIA Tesla V100S GPU using PyTorch. Random seeds are fixed for reproducibility.
For the flow-matching-based methods, including our proposed method, we use the U-Net model weights and pretrained checkpoints from \cite{pnp_flow}.
The only exception is PnP-GS, which uses its own pretrained denoiser.

We note a difference in the evaluation protocol. Martin et al. \cite{pnp_flow} report results averaged over $25$ batches of $4$ images each, whereas we recompute all evaluation metrics using $100$ independent runs for each image. 
As a result, the values reported here may differ slightly from those in the original publication.
For each baseline method, we use the optimal hyperparameters reported in \cite{pnp_flow}, which were selected by grid search on the validation set
to maximize PSNR. 
Detailed hyperparameter settings for HP-PnP-Flow and all baseline methods are provided in Appendix \ref{sec:appendix_hyperparams}.

\subsection{Main Numerical Results}
Next we report quantitative comparisons of all methods across multiple image restoration tasks using PSNR, SSIM, and LPIPS metrics.
By ''N/A'' we mean that a method is not applicable to the task.
For example, PnP-GS relies on a discriminative denoiser that cannot perform generative completion and thus cannot handle the box inpainting forward model.
PnP-Flow(m) and HP-PnP-Flow(m) denote variants that average the denoising results over $m$ independent samples of $\xi$ at each iteration. 
Bold and underlining values indicate the best and the second-best performance, respectively.

\FloatBarrier
\begin{table}[!htbp]
    \centering
    \tiny
    \setlength{\tabcolsep}{0.6mm}
    \caption{Quantitative results on CelebA (100 test images).}
    \label{tab:celeba_results}
    \begin{tabularx}{\linewidth}{@{}l*{15}{>{\centering\arraybackslash}X}@{}}
        \toprule
        & \multicolumn{3}{c}{\textbf{Denoise}}
        & \multicolumn{3}{c}{\textbf{Deblur}}
        & \multicolumn{3}{c}{\textbf{Super-res.}}
        & \multicolumn{3}{c}{\textbf{Rand.\ inp.}}
        & \multicolumn{3}{c}{\textbf{Box inp.}} \\
        \cmidrule(lr){2-4} \cmidrule(lr){5-7} \cmidrule(lr){8-10} \cmidrule(lr){11-13} \cmidrule(lr){14-16}
        \textbf{Method}
        & PSNR$\uparrow$ & SSIM$\uparrow$ & LPIPS$\downarrow$
        & PSNR$\uparrow$ & SSIM$\uparrow$ & LPIPS$\downarrow$
        & PSNR$\uparrow$ & SSIM$\uparrow$ & LPIPS$\downarrow$
        & PSNR$\uparrow$ & SSIM$\uparrow$ & LPIPS$\downarrow$
        & PSNR$\uparrow$ & SSIM$\uparrow$ & LPIPS$\downarrow$ \\
        \midrule
        Degraded
        & 20.01 & 0.349 & 0.365
        & 27.79 & 0.739 & 0.126
        & 10.51 & 0.184 & 0.824
        & 12.26 & 0.197 & 1.024
        & 22.42 & 0.745 & 0.209 \\
        PnP-GS
        & \underline{32.59} & 0.910 & 0.035
        & 33.84 & 0.922 & 0.040
        & 30.35 & 0.902 & 0.056
        & 28.42 & 0.845 & 0.087
        & \multicolumn{3}{c}{N/A} \\
        OT-ODE
        & 30.49 & 0.858 & \underline{0.032}
        & 32.87 & 0.918 & \textbf{0.028}
        & 31.56 & 0.910 & \textbf{0.020}
        & 28.82 & 0.874 & 0.050
        & 29.92 & 0.923 & 0.033 \\
        D-Flow
        & 25.99 & 0.604 & 0.090
        & 31.16 & 0.852 & \underline{0.037}
        & 30.39 & 0.847 & 0.025
        & 33.65 & 0.944 & \underline{0.015}
        & 30.93 & 0.901 & \underline{0.024} \\
        Flow-Priors
        & 29.30 & 0.769 & 0.129
        & 31.47 & 0.858 & 0.055
        & 28.31 & 0.716 & 0.099
        & 32.95 & 0.943 & 0.019
        & 30.42 & 0.862 & 0.046 \\
        PnP-Flow(1)
        & 31.79 & 0.905 & 0.044
        & 34.33 & 0.935 & 0.038
        & 31.00 & 0.902 & 0.044
        & 33.06 & 0.943 & 0.018
        & 30.52 & 0.936 & 0.035 \\
        PnP-Flow(5)
        & 32.30 & 0.912 & 0.055
        & \textbf{34.65} & \textbf{0.940} & 0.046
        & 31.42 & 0.907 & 0.054
        & 33.94 & 0.953 & 0.022
        & 31.07 & 0.942 & 0.041 \\
        HP-PnP-Flow(1)
        & 32.39 & \underline{0.915} & \textbf{0.026}
        & 34.34 & 0.933 & 0.037
        & \underline{32.75} & \underline{0.929} & \underline{0.021}
        & \underline{33.98} & \underline{0.955} & \textbf{0.013}
        & \underline{31.25} & \underline{0.945} & \textbf{0.024} \\
        HP-PnP-Flow(5)
        & \textbf{32.91} & \textbf{0.923} & 0.035
        & \underline{34.65} & \underline{0.938} & 0.044
        & \textbf{33.27} & \textbf{0.936} & 0.025
        & \textbf{34.76} & \textbf{0.962} & 0.015
        & \textbf{32.02} & \textbf{0.954} & 0.027 \\
        \bottomrule
    \end{tabularx}
\end{table}

\FloatBarrier
\begin{table}[!htbp]
    \centering
    \tiny
    \setlength{\tabcolsep}{0.6mm}
    \caption{Quantitative results on AFHQ-Cat (100 test images).}
    \label{tab:afhqcat_results}
    \begin{tabularx}{\linewidth}{@{}l*{15}{>{\centering\arraybackslash}X}@{}}
        \toprule
        & \multicolumn{3}{c}{\textbf{Denoise}}
        & \multicolumn{3}{c}{\textbf{Deblur}}
        & \multicolumn{3}{c}{\textbf{Super-res.}}
        & \multicolumn{3}{c}{\textbf{Rand.\ inp.}}
        & \multicolumn{3}{c}{\textbf{Box inp.}} \\
        \cmidrule(lr){2-4} \cmidrule(lr){5-7} \cmidrule(lr){8-10} \cmidrule(lr){11-13} \cmidrule(lr){14-16}
        \textbf{Method}
        & PSNR$\uparrow$ & SSIM$\uparrow$ & LPIPS$\downarrow$
        & PSNR$\uparrow$ & SSIM$\uparrow$ & LPIPS$\downarrow$
        & PSNR$\uparrow$ & SSIM$\uparrow$ & LPIPS$\downarrow$
        & PSNR$\uparrow$ & SSIM$\uparrow$ & LPIPS$\downarrow$
        & PSNR$\uparrow$ & SSIM$\uparrow$ & LPIPS$\downarrow$ \\
        \midrule
        Degraded
        & 20.00 & 0.294 & 0.531
        & 24.39 & 0.530 & 0.436
        & 11.59 & 0.212 & 0.867
        & 13.24 & 0.214 & 1.082
        & 21.53 & 0.727 & 0.215 \\
        PnP-GS
        & \textbf{33.03} & 
        \underline{0.896} & 
        \underline{0.078}
        & 28.38 & 0.771 & 0.331
        & 24.26 & 0.642 & 0.350
        & 29.80 & 0.838 & 0.127
        & \multicolumn{3}{c}{N/A} \\
        OT-ODE
        & 30.51 & 0.819 & 0.080
        & 26.50 & 0.680 & \textbf{0.149}
        & 26.73 & 0.738 & \textbf{0.108}
        & 30.01 & 0.850 & 0.086
        & 24.51 & 0.874 & 0.095 \\
        D-Flow
        & 26.42 & 0.570 & 0.182
        & 28.58 & 0.745 & \underline{0.168}
        & 25.19 & 0.616 & 0.190
        & 32.83 & 0.900 & 0.044
        & 27.02 & 0.841 & \underline{0.081} \\
        Flow-Priors
        & 29.68 & 0.757 & 0.164
        & 27.28 & 0.727 & 0.188
        & 24.01 & 0.572 & 0.271
        & 32.98 & 0.910 & 0.051
        & 26.13 & 0.809 & 0.130 \\
        PnP-Flow(1)
        & 31.65 & 0.867 & 0.139
        & 28.74 & 0.781 & 0.285
        & 27.58 & 0.781 & 0.158
        & 33.75 & 0.924 & \underline{0.038}
        & 26.40 & 0.896 & 0.111 \\
        PnP-Flow(5)
        & 31.89 & 0.869 & 0.170
        & \underline{29.03} & \textbf{0.786} & 0.313
        & \underline{28.02} & \underline{0.792} & 0.168
        & \underline{34.47} & \underline{0.933} & 0.045
        & 27.15 & 0.901 & 0.127 \\
        HP-PnP-Flow(1)
        & 32.51 & 0.892 & \textbf{0.066}
        & 28.88 & 0.780 & 0.284
        & 27.93 & 0.783 & \underline{0.150}
        & 33.91 & 0.928 & \textbf{0.035}
        & \underline{27.29} & \underline{0.912} & \textbf{0.072} \\
        HP-PnP-Flow(5)
        & \underline{32.97} & \textbf{0.900} & 0.086
        & \textbf{29.10} & \underline{0.784} & 0.313
        & \textbf{28.42} & \textbf{0.796} & 0.161
        & \textbf{34.60} & \textbf{0.936} & 0.040
        & \textbf{28.05} & \textbf{0.921} & 0.082 \\
        \bottomrule
    \end{tabularx}
\end{table}

\FloatBarrier

Tables \ref{tab:celeba_results} and \ref{tab:afhqcat_results} present quantitative
comparisons of different methods, showing that our algorithm achieves the best or second-best performance in most settings across the two datasets and three evaluation metrics.
One can observe that, the largest PSNR improvement over PnP-Flow is observed in super-resolution on CelebA, where HP-PnP-Flow(1) improves upon PnP-Flow(1) by $1.75$ dB under the same single-sampling setting.
%
Compared with the single-sample variant HP-PnP-Flow(1),
HP-PnP-Flow(5) provides an additional PSNR improvement of
$0.2-0.8$\,dB.
Notably, HP-PnP-Flow(1) outperforms PnP-Flow(5) on most restoration tasks despite using only a single sample per iteration.
Compared with the other baselines, HP-PnP-Flow shows more balanced performance across different restoration tasks. 
While D-Flow and Flow-Priors are particularly effective on inpainting, and PnP-GS and OT-ODE perform well on denoising and deblurring, our algorithm remains competitive across all five degradation settings on both datasets.

In addition, our algorithm achieves these results with a substantially smaller iteration budget on most tasks. For instance, on AFHQ-Cat, HP-PnP-Flow uses $60$ iterations
for box inpainting versus $100$ for PnP-Flow, and $100$ iterations for deblurring versus $500$ for PnP-Flow (see Appendix \ref{sec:appendix_hyperparams} for details).


\subsection{Comparison to PnP-Flow}
\label{sec:hp_vs_pnp}
In the following, we compare HP-PnP-Flow \eqref{eq:hp-pnp-flow} with PnP-Flow \eqref{eq:pnp-flow} under the single-sampling setting, and analyze their iteration-wise reconstruction trajectories across all five restoration tasks and both datasets. Visual comparisons are provided in
Appendix \ref{sec:appendix_iteration_comparison}.

One can observe that HP-PnP-Flow exhibits faster convergence during the early iterations. 
For example, from Figures \ref{fig:app_denoise_celeba} and
\ref{fig:app_denoise_afhq}, we observe that, at iteration $20$, HP-PnP-Flow improves upon PnP-Flow by $7.58$ and $7.96$\,dB in PSNR, and by $0.1611$ and $0.219$ in SSIM, on CelebA and AFHQ-Cat, respectively. Similar early-stage advantages are also observed on the other four tasks.
%
In terms of equivalent PnP-Flow iterations, the reconstruction quality achieved by HP-PnP-Flow after $20$ iterations is comparable to that of PnP-Flow after approximately $40$-$120$ iterations, depending on the task. 
The largest efficiency gap is observed for random inpainting ($20$ vs.\ $80$-$120$ iterations), whereas the smallest gap occurs for deblurring ($20$ vs.\ $40$-$60$ iterations).

HP-PnP-Flow also achieves better final reconstruction quality. Across the five tasks and both datasets , its converged PSNR exceeds that of PnP-Flow by $0.38$-$1.74$\,dB, with the largest improvements observed for super-resolution and random inpainting (see Appendix \ref{sec:appendix_iteration_comparison} for details). 
On AFHQ-Cat, HP-PnP-Flow requires only $100$ iterations for deblurring and super-resolution, compared with $500$ iterations for PnP-Flow, and $60$ iterations for box inpainting, compared with $100$ for PnP-Flow. 
Taken together, the faster early-stage convergence and reduced iteration budget enable HP-PnP-Flow to reach a
given PSNR level in substantially fewer iterations than PnP-Flow.

\subsection{Qualitative Comparison}
In this subsection we provide visual comparisons of all methods on a single test image from CelebA and AFHQ-Cat across the five degradation tasks. 
Each row corresponds to one task, ordered from top to
bottom as denoising, deblurring, super-resolution, random inpainting, and box inpainting. 
The PSNR value of each reconstruction is reported below
the corresponding image.

\FloatBarrier
\begin{figure}[htbp]
    \centering
    \setlength{\tabcolsep}{0.3pt}
    \renewcommand{\arraystretch}{1.0}
    \begin{tabular}{*{10}{>{\centering\arraybackslash\footnotesize}m{0.098\textwidth}}}
        \makecell[c]{Clean\vphantom{\makecell[c]{Flow-\\Priors}}} &
        \makecell[c]{Degraded\vphantom{\makecell[c]{Flow-\\Priors}}} &
        \makecell[c]{PnP-GS\vphantom{\makecell[c]{Flow-\\Priors}}} &
        \makecell[c]{OT-ODE\vphantom{\makecell[c]{Flow-\\Priors}}} &
        \makecell[c]{D-Flow\vphantom{\makecell[c]{Flow-\\Priors}}} &
        \makecell[c]{Flow-\\Priors} &
        \makecell[c]{PnP-\\Flow(1)} &
        \makecell[c]{PnP-\\Flow(5)} &
        \makecell[c]{HP-PnP-\\Flow(1)} &
        \makecell[c]{HP-PnP-\\Flow(5)} \\[4pt]

        \makecell[c]{\includegraphics[width=\linewidth]{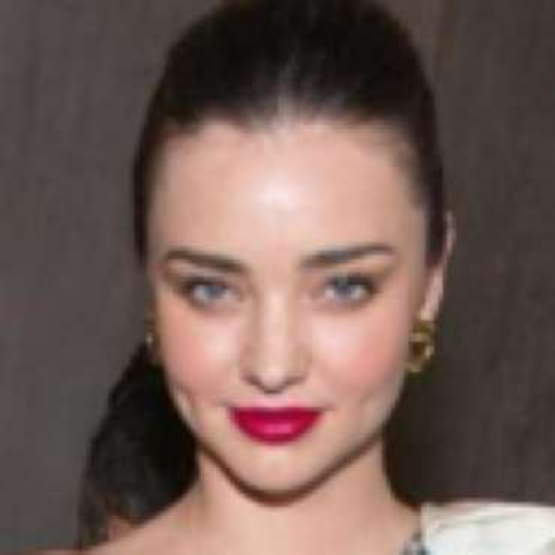}\\\vphantom{\tiny PSNR: 00.00}} &
        \makecell[c]{\includegraphics[width=\linewidth]{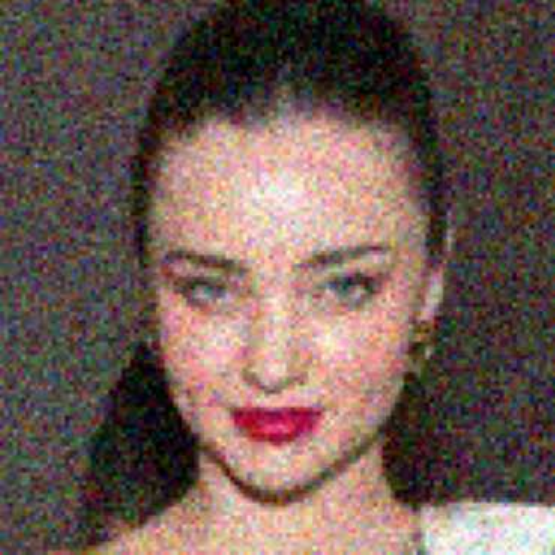}\\\tiny PSNR: 20.01} &
        \makecell[c]{\includegraphics[width=\linewidth]{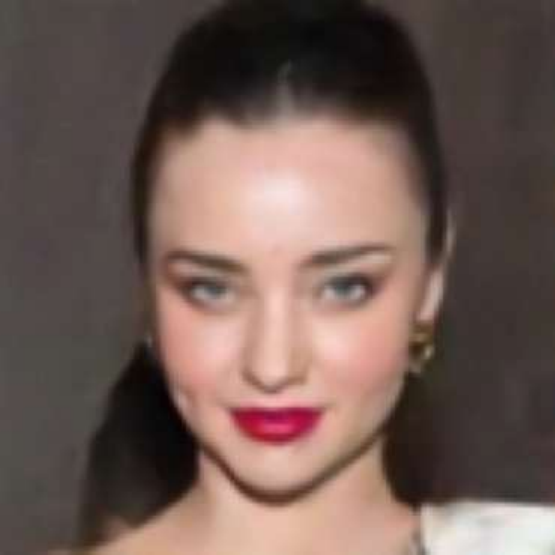}\\\tiny PSNR: 34.37} &
        \makecell[c]{\includegraphics[width=\linewidth]{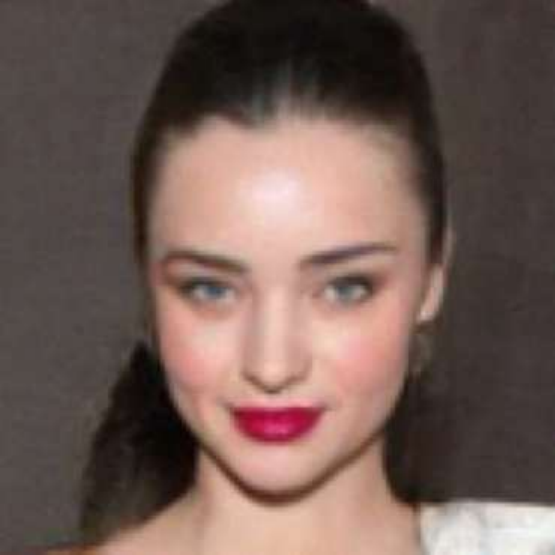}\\\tiny PSNR: 32.10} &
        \makecell[c]{\includegraphics[width=\linewidth]{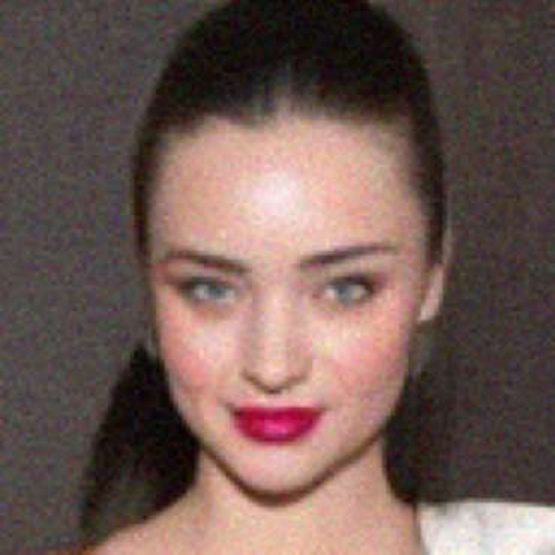}\\\tiny PSNR: 27.06} &
        \makecell[c]{\includegraphics[width=\linewidth]{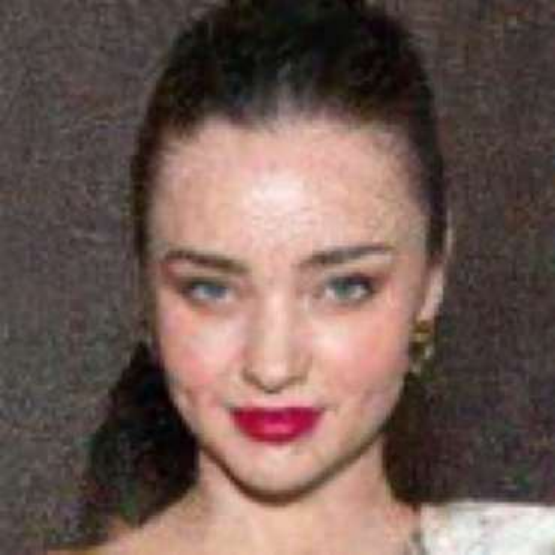}\\\tiny PSNR: 29.97} &
        \makecell[c]{\includegraphics[width=\linewidth]{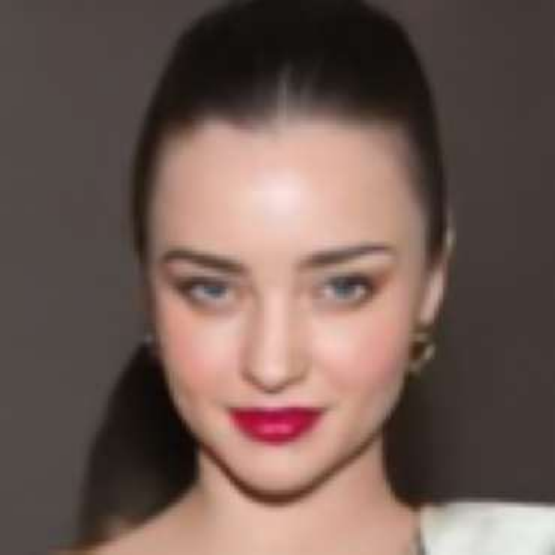}\\\tiny PSNR: 33.96} &
        \makecell[c]{\includegraphics[width=\linewidth]{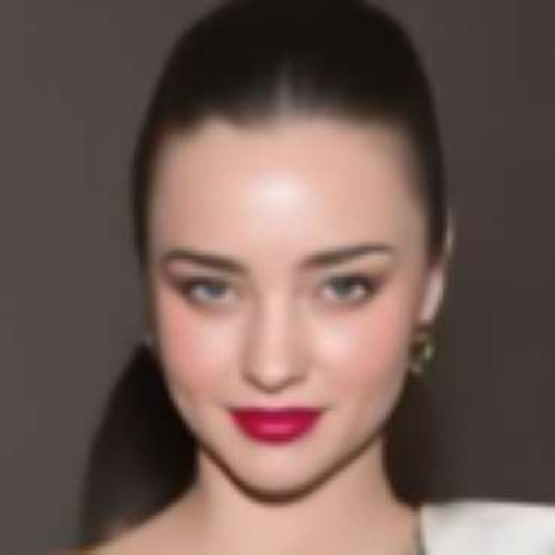}\\\tiny PSNR: 34.29} &
        \makecell[c]{\includegraphics[width=\linewidth]{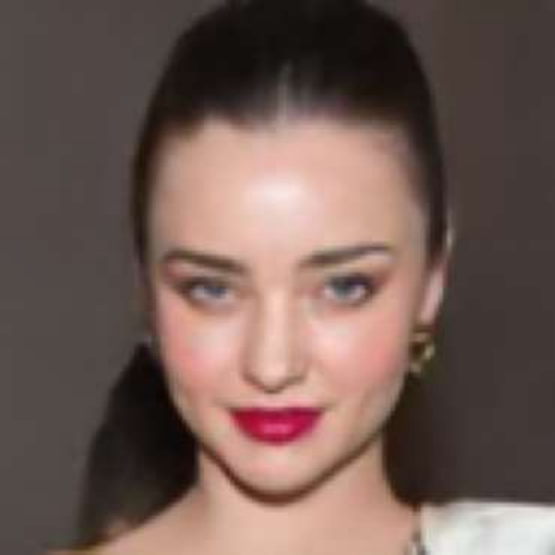}\\\tiny PSNR: 34.47} &
        \makecell[c]{\includegraphics[width=\linewidth]{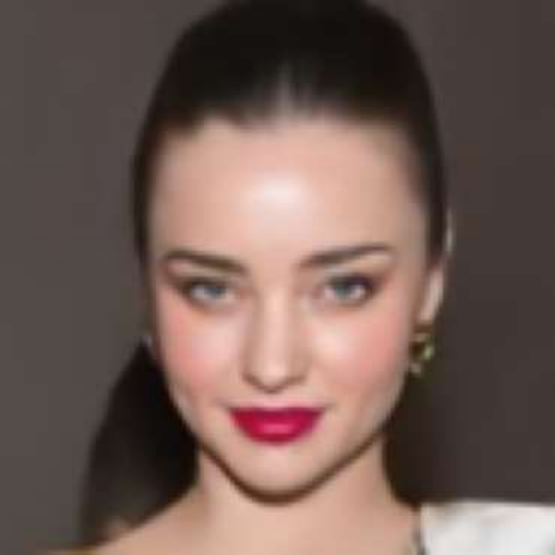}\\\tiny PSNR: 34.86}
        \\[5pt]
        
        \makecell[c]{\includegraphics[width=\linewidth]{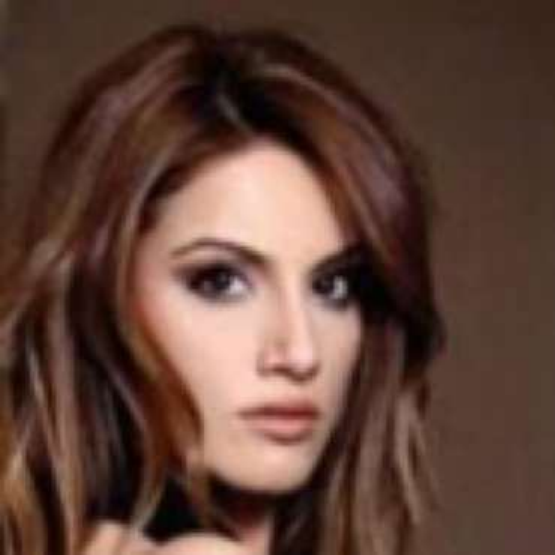}\\\vphantom{\tiny PSNR: 00.00}} &
        \makecell[c]{\includegraphics[width=\linewidth]{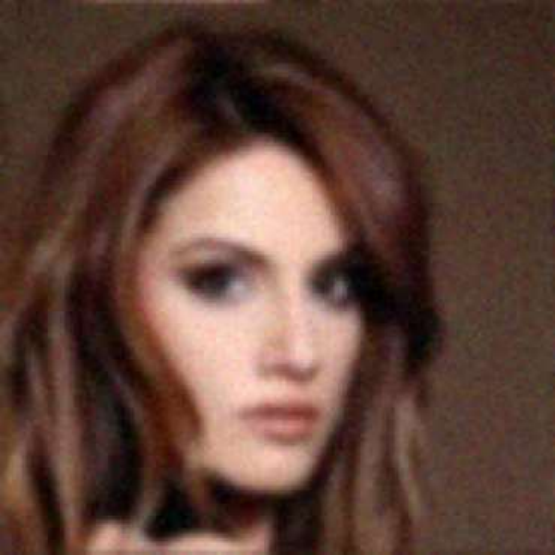}\\\tiny PSNR: 28.72} &
        \makecell[c]{\includegraphics[width=\linewidth]{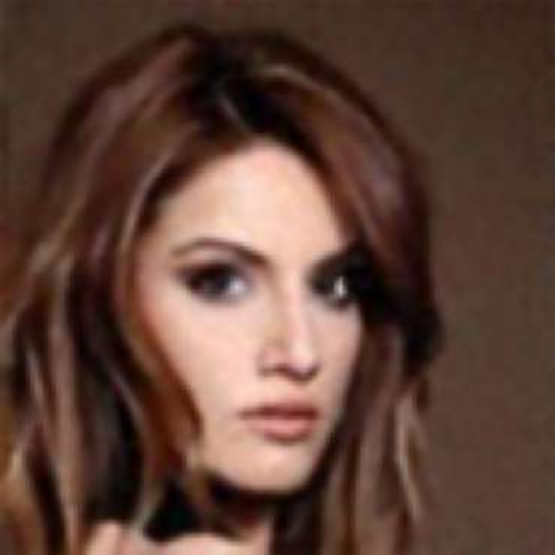}\\\tiny PSNR: 34.92} &
        \makecell[c]{\includegraphics[width=\linewidth]{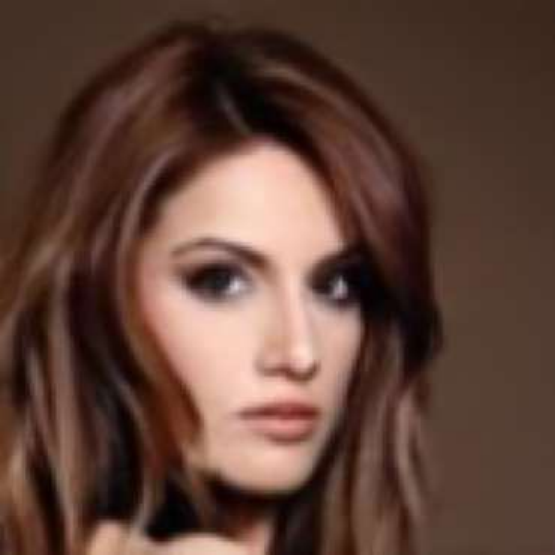}\\\tiny PSNR: 33.78} &
        \makecell[c]{\includegraphics[width=\linewidth]{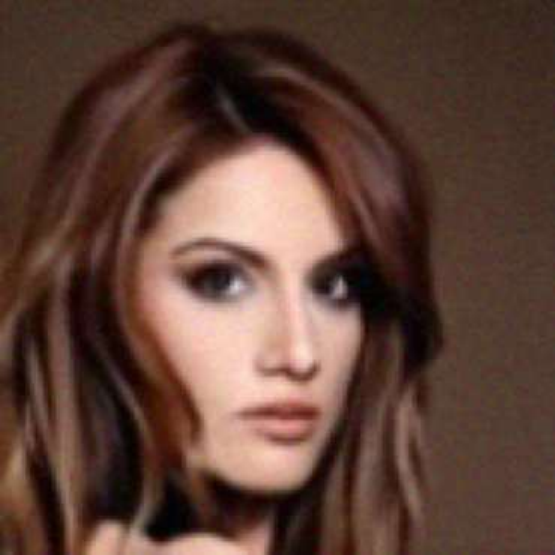}\\\tiny PSNR: 31.21} &
        \makecell[c]{\includegraphics[width=\linewidth]{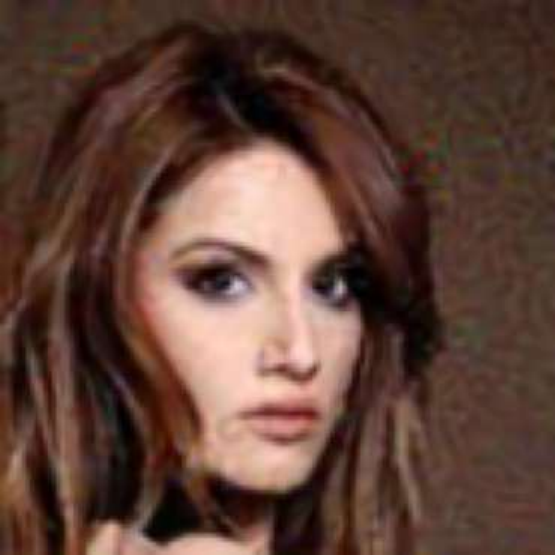}\\\tiny PSNR: 32.28} &
        \makecell[c]{\includegraphics[width=\linewidth]{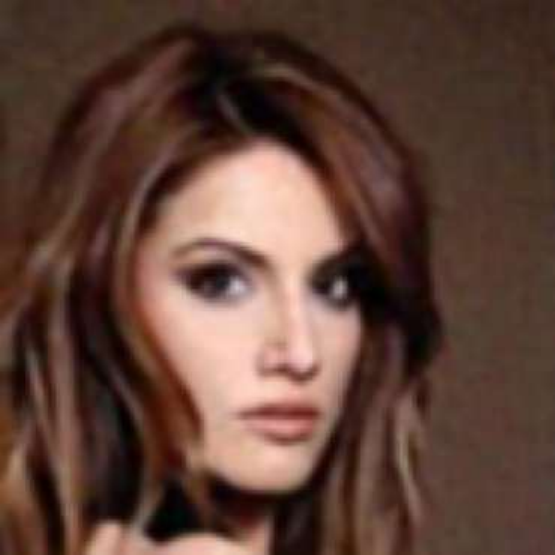}\\\tiny PSNR: 35.29} &
        \makecell[c]{\includegraphics[width=\linewidth]{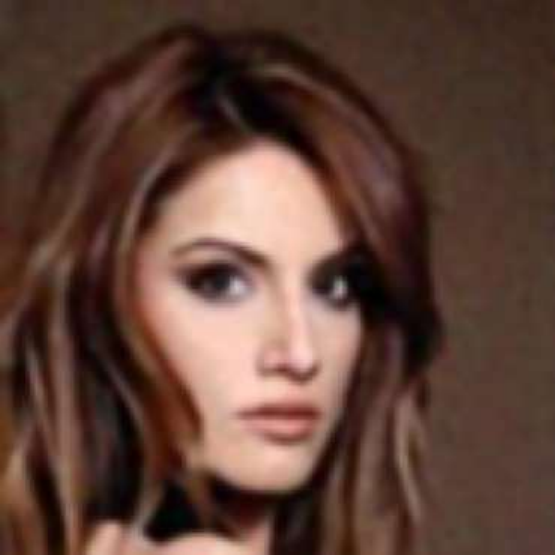}\\\tiny PSNR: 35.72} &
        \makecell[c]{\includegraphics[width=\linewidth]{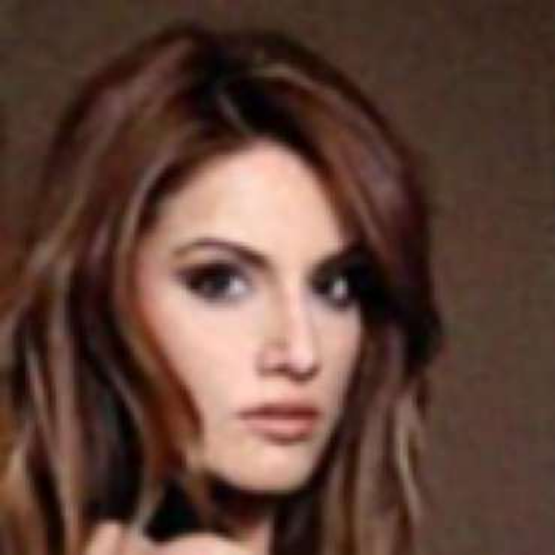}\\\tiny PSNR: 35.25} &
        \makecell[c]{\includegraphics[width=\linewidth]{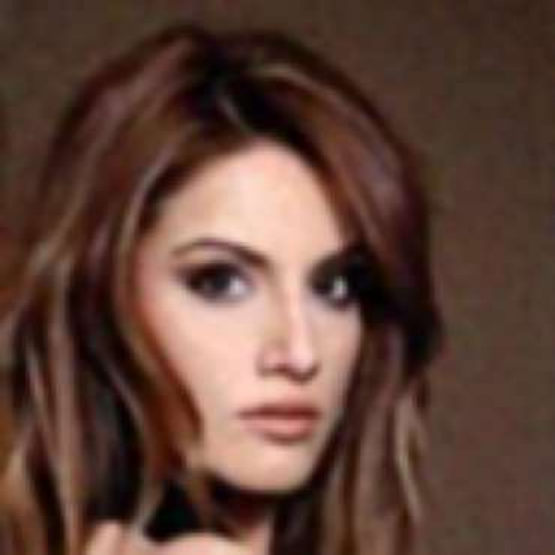}\\\tiny PSNR: 35.70}
        \\[5pt]
        
        \makecell[c]{\includegraphics[width=\linewidth]{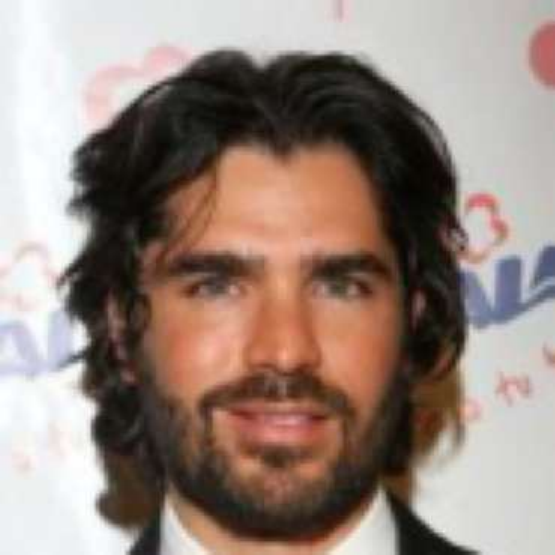}\\\vphantom{\tiny PSNR: 00.00}} &
        \makecell[c]{\includegraphics[width=\linewidth]{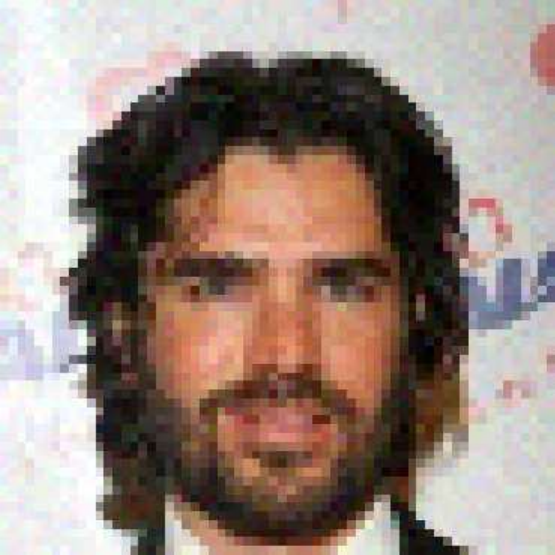}\\\tiny PSNR: 7.95} &
        \makecell[c]{\includegraphics[width=\linewidth]{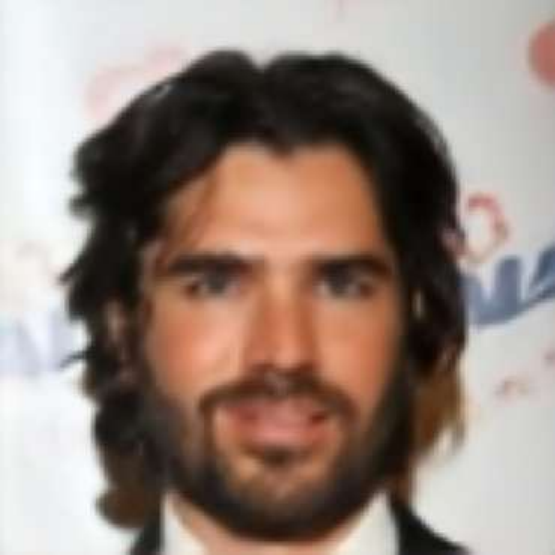}\\\tiny PSNR: 30.81} &
        \makecell[c]{\includegraphics[width=\linewidth]{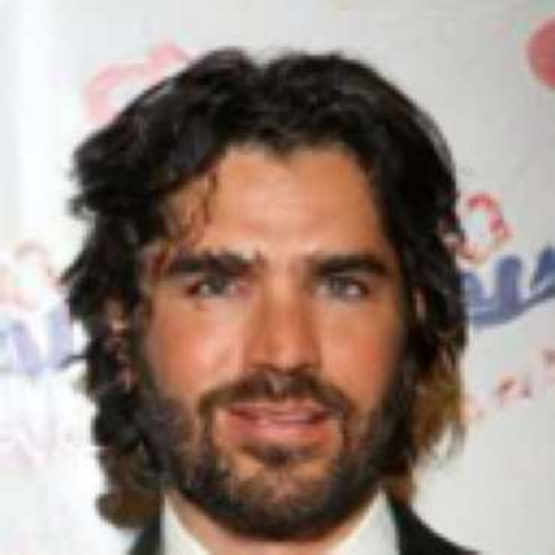}\\\tiny PSNR: 29.63} &
        \makecell[c]{\includegraphics[width=\linewidth]{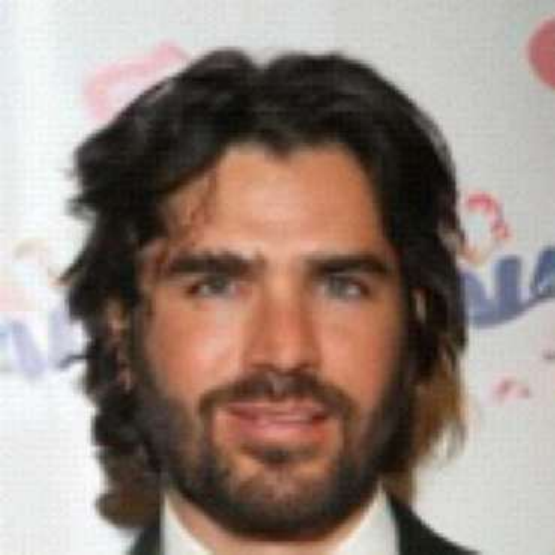}\\\tiny PSNR: 29.20} &
        \makecell[c]{\includegraphics[width=\linewidth]{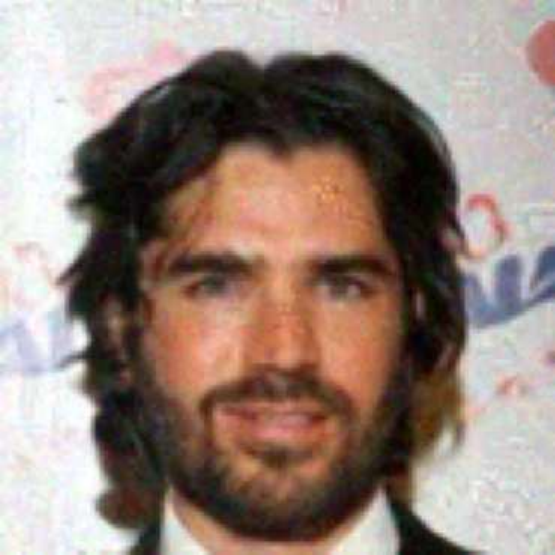}\\\tiny PSNR: 28.19} &
        \makecell[c]{\includegraphics[width=\linewidth]{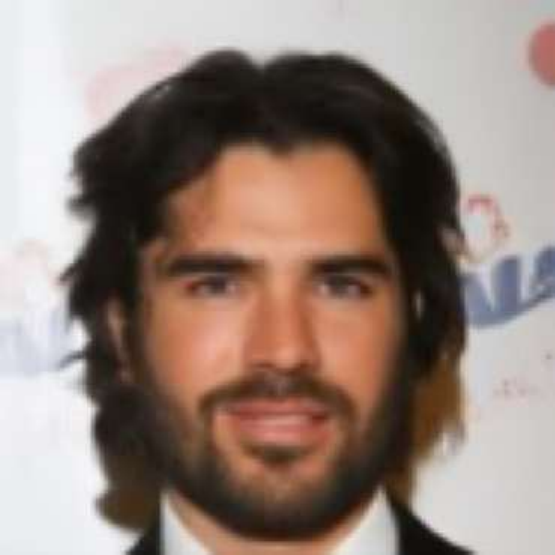}\\\tiny PSNR: 29.56} &
        \makecell[c]{\includegraphics[width=\linewidth]{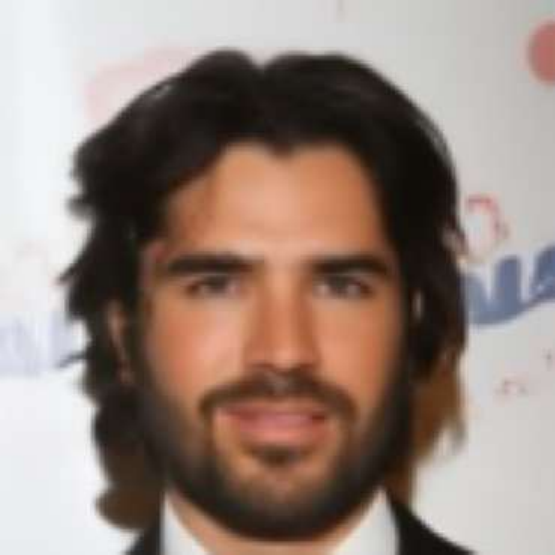}\\\tiny PSNR: 29.99} &
        \makecell[c]{\includegraphics[width=\linewidth]{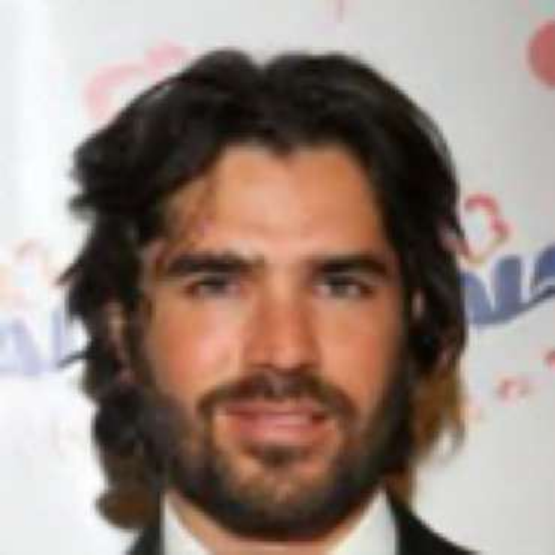}\\\tiny PSNR: 31.57} &
        \makecell[c]{\includegraphics[width=\linewidth]{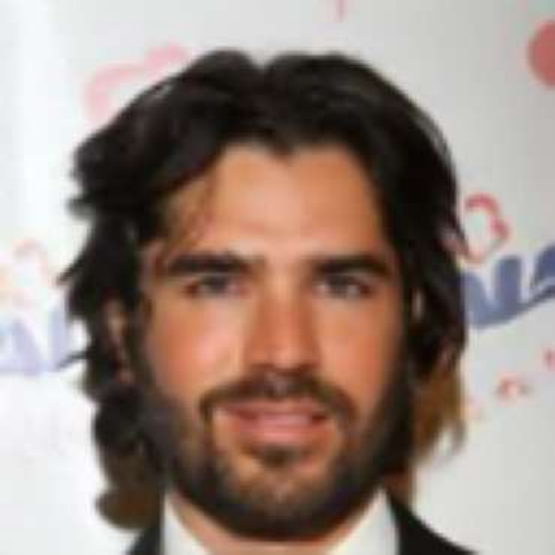}\\\tiny PSNR: 32.04}
        \\[5pt]
        
        \makecell[c]{\includegraphics[width=\linewidth]{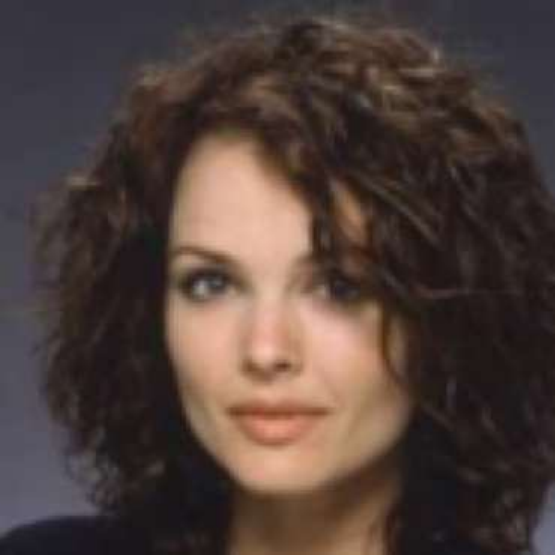}\\\vphantom{\tiny PSNR: 00.00}} &
        \makecell[c]{\includegraphics[width=\linewidth]{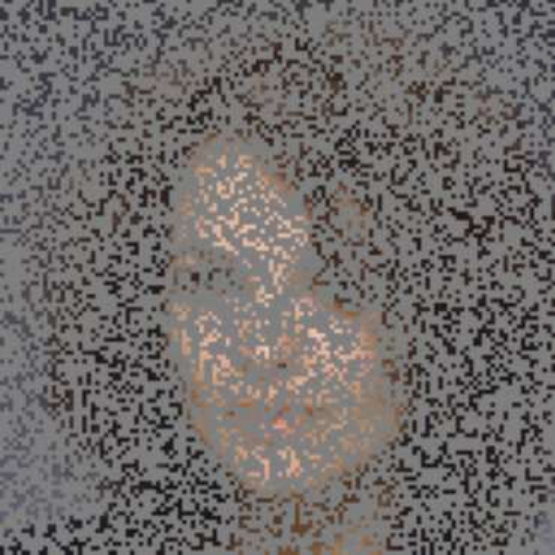}\\\tiny PSNR: 12.06} &
        \makecell[c]{\includegraphics[width=\linewidth]{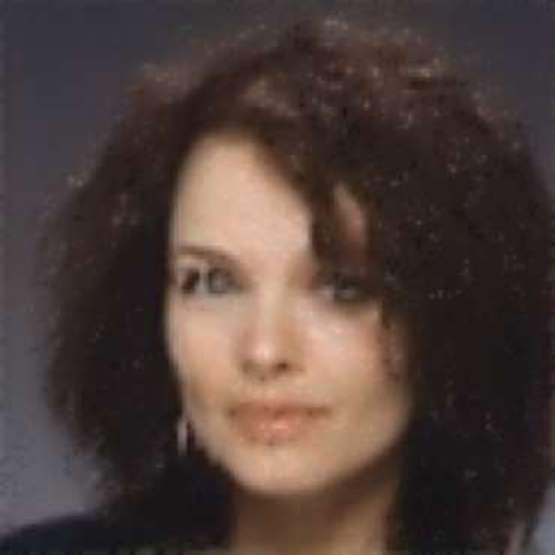}\\\tiny PSNR: 30.33} &
        \makecell[c]{\includegraphics[width=\linewidth]{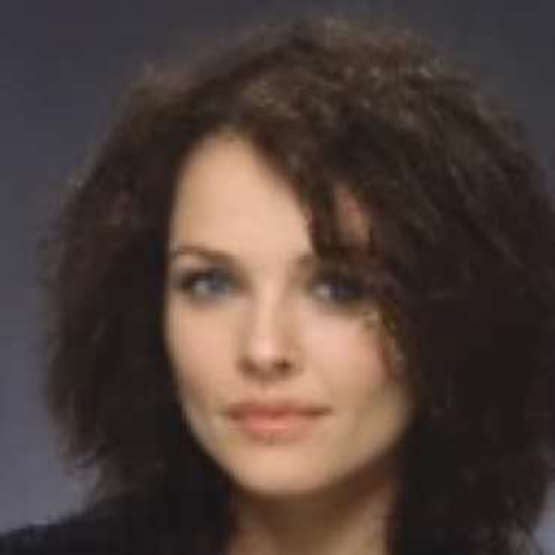}\\\tiny PSNR: 31.42} &
        \makecell[c]{\includegraphics[width=\linewidth]{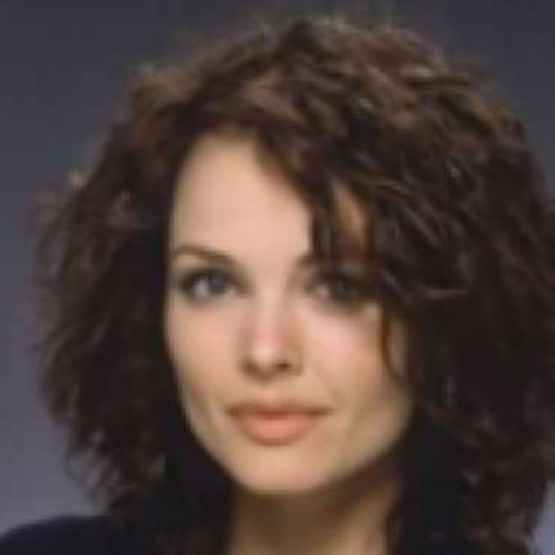}\\\tiny PSNR: 35.40} &
        \makecell[c]{\includegraphics[width=\linewidth]{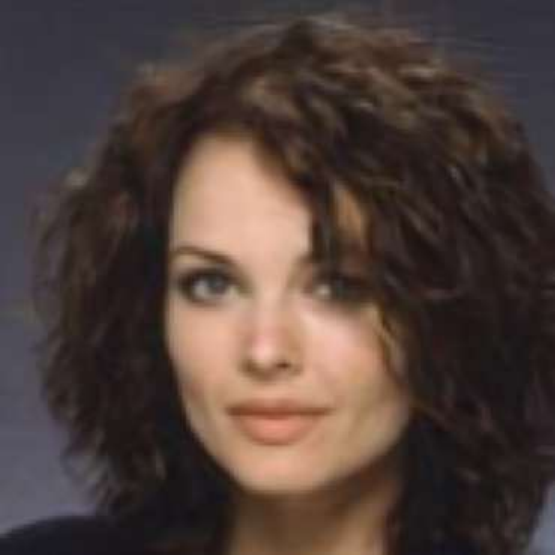}\\\tiny PSNR: 35.26} &
        \makecell[c]{\includegraphics[width=\linewidth]{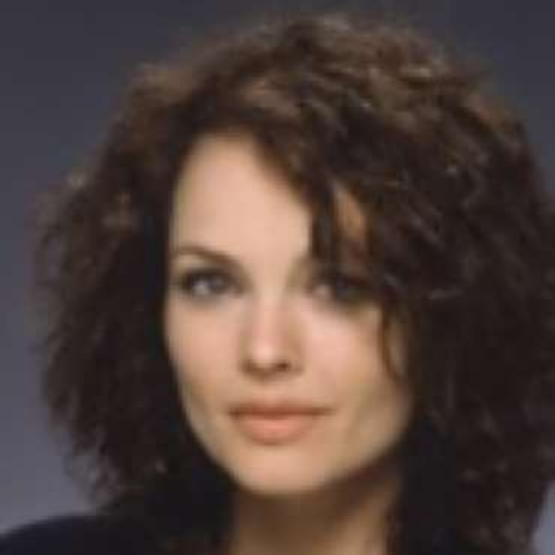}\\\tiny PSNR: 34.99} &
        \makecell[c]{\includegraphics[width=\linewidth]{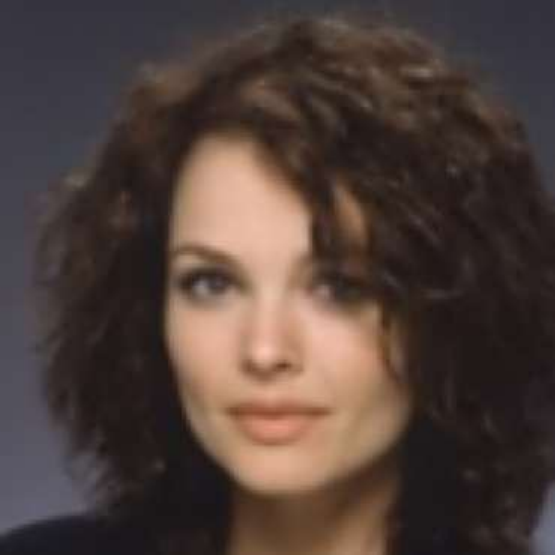}\\\tiny PSNR: 35.65} &
        \makecell[c]{\includegraphics[width=\linewidth]{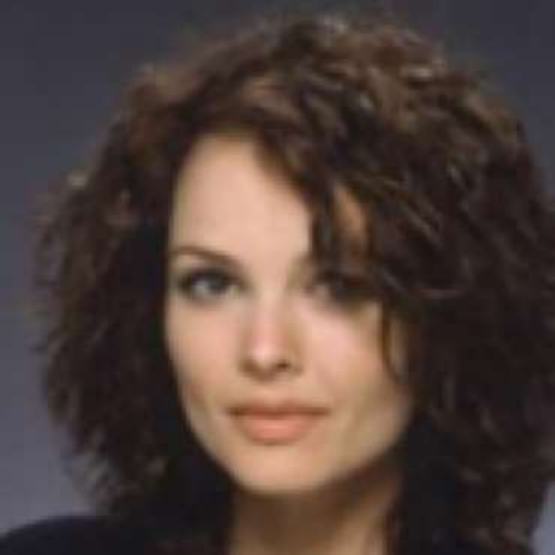}\\\tiny PSNR: 36.33} &
        \makecell[c]{\includegraphics[width=\linewidth]{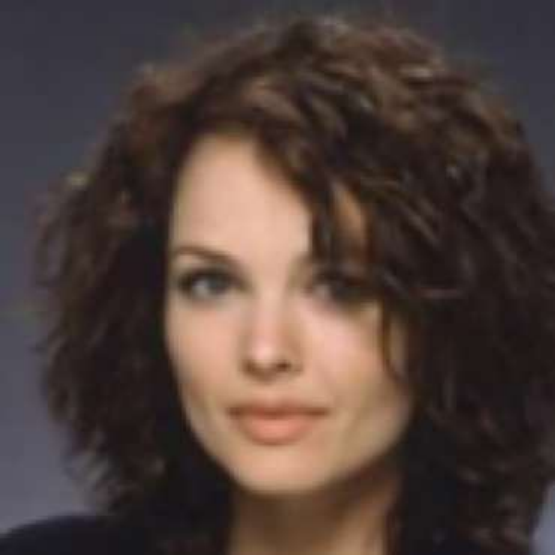}\\\tiny PSNR: 36.81}
        \\[5pt]
        
        \makecell[c]{\includegraphics[width=\linewidth]{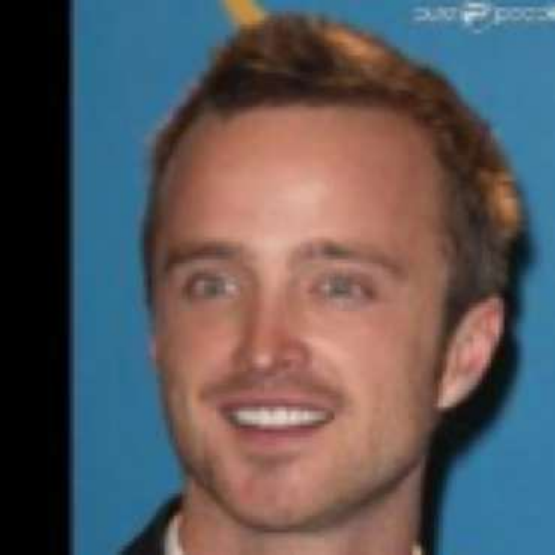}\\\vphantom{\tiny PSNR: 00.00}} &
        \makecell[c]{\includegraphics[width=\linewidth]{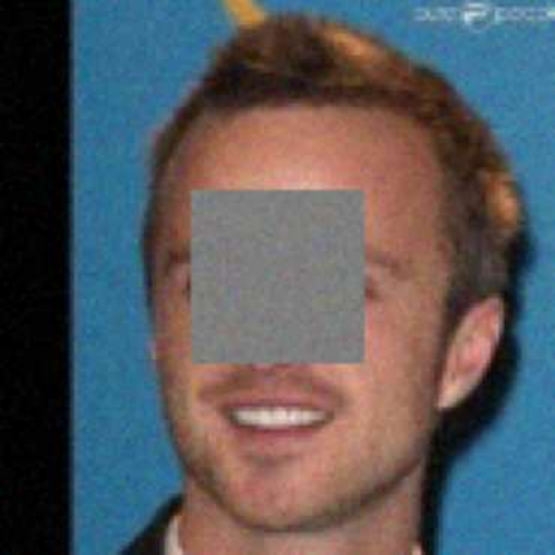}\\\tiny PSNR: 24.87} &
        \makecell[c]{\text{N/A}\\\vphantom{\tiny PSNR: 00.00}} &
        \makecell[c]{\includegraphics[width=\linewidth]{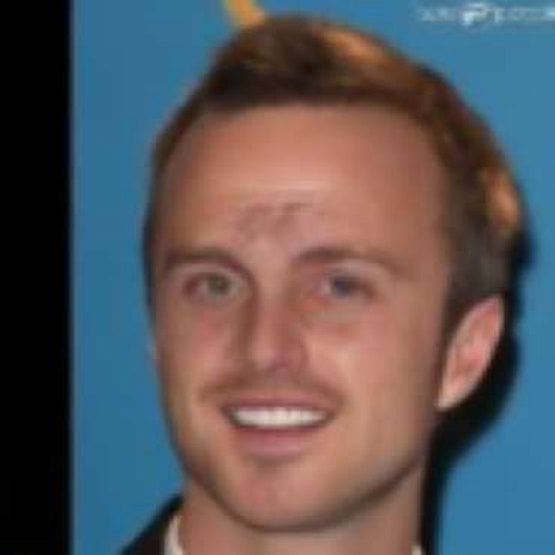}\\\tiny PSNR: 31.16} &
        \makecell[c]{\includegraphics[width=\linewidth]{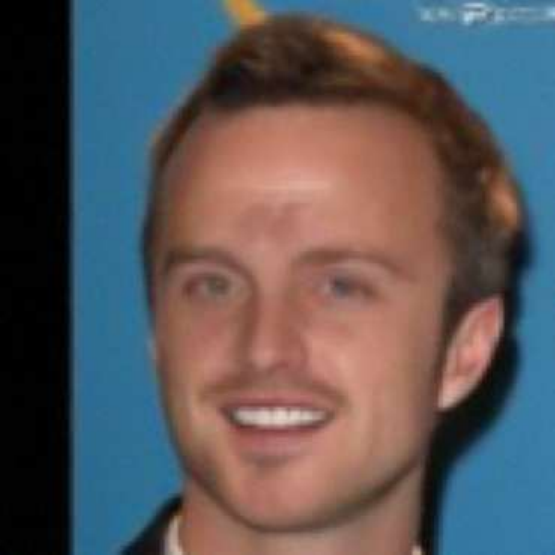}\\\tiny PSNR: 30.99} &
        \makecell[c]{\includegraphics[width=\linewidth]{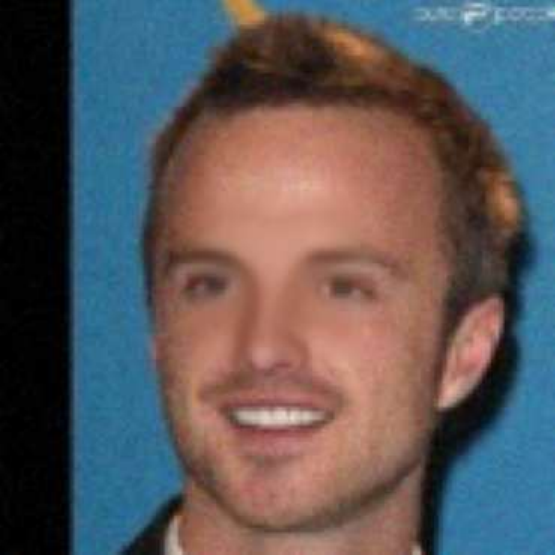}\\\tiny PSNR: 32.05} &
        \makecell[c]{\includegraphics[width=\linewidth]{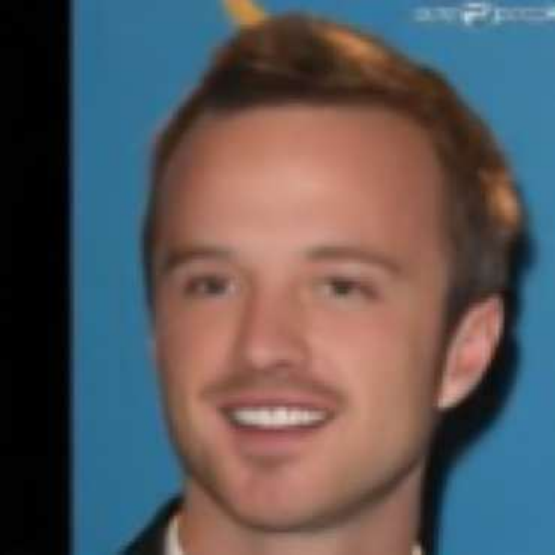}\\\tiny PSNR: 34.27} &
        \makecell[c]{\includegraphics[width=\linewidth]{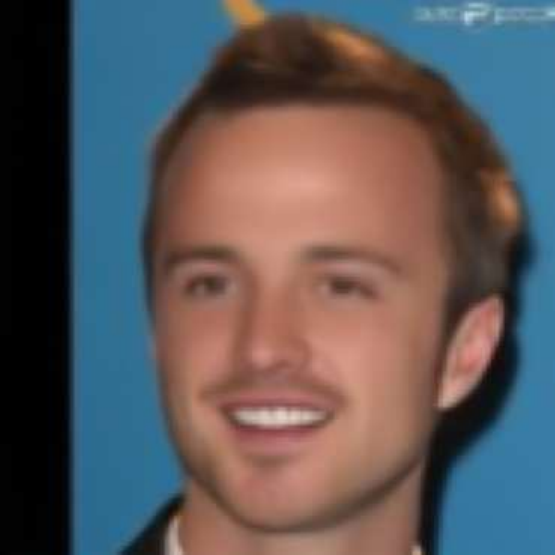}\\\tiny PSNR: 34.16} &
        \makecell[c]{\includegraphics[width=\linewidth]{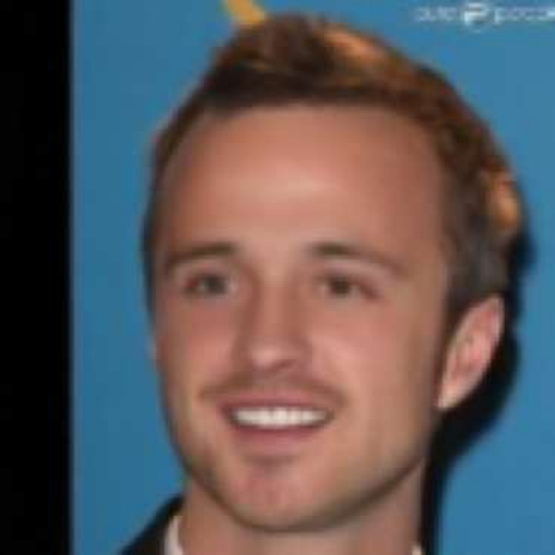}\\\tiny PSNR: 33.76} &
        \makecell[c]{\includegraphics[width=\linewidth]{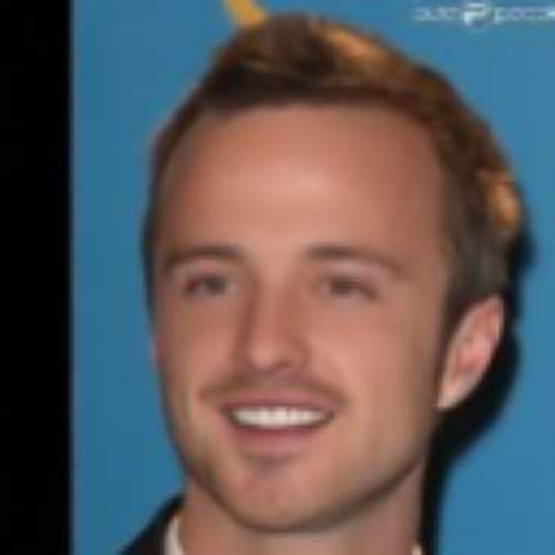}\\\tiny PSNR: 35.45}
        \\[3pt]
    \end{tabular}
    \caption{PSNR comparison of all methods on CelebA across five restoration tasks (from top to bottom: denoising, deblurring, super-resolution, random inpainting, and box inpainting).}
    \label{fig:restore_compare_celeba}
\end{figure}
\FloatBarrier

\FloatBarrier
\begin{figure}[htbp]
    \centering
    \setlength{\tabcolsep}{0.3pt}
    \renewcommand{\arraystretch}{1.0}
    \begin{tabular}{*{10}{>{\centering\arraybackslash\footnotesize}m{0.098\textwidth}}}
        \makecell[c]{Clean\vphantom{\makecell[c]{Flow-\\Priors}}} &
        \makecell[c]{Degraded\vphantom{\makecell[c]{Flow-\\Priors}}} &
        \makecell[c]{PnP-GS\vphantom{\makecell[c]{Flow-\\Priors}}} &
        \makecell[c]{OT-ODE\vphantom{\makecell[c]{Flow-\\Priors}}} &
        \makecell[c]{D-Flow\vphantom{\makecell[c]{Flow-\\Priors}}} &
        \makecell[c]{Flow-\\Priors} &
        \makecell[c]{PnP-\\Flow(1)} &
        \makecell[c]{PnP-\\Flow(5)} &
        \makecell[c]{HP-PnP-\\Flow(1)} &
        \makecell[c]{HP-PnP-\\Flow(5)} \\[4pt]
        
        \makecell[c]{\includegraphics[width=\linewidth]{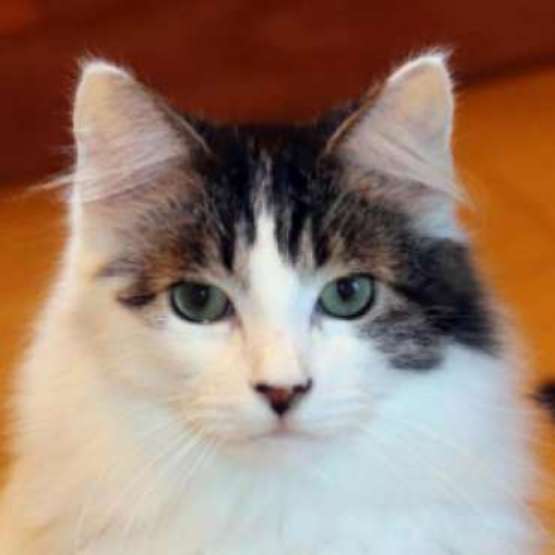}\\\vphantom{\tiny PSNR: 00.00}} &
        \makecell[c]{\includegraphics[width=\linewidth]{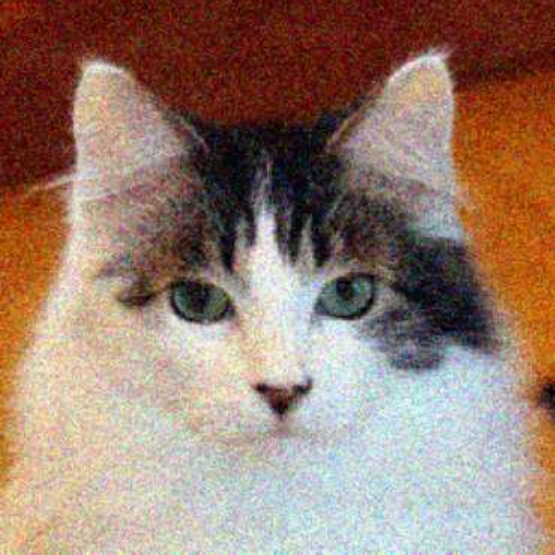}\\\tiny PSNR: 20.00} &
        \makecell[c]{\includegraphics[width=\linewidth]{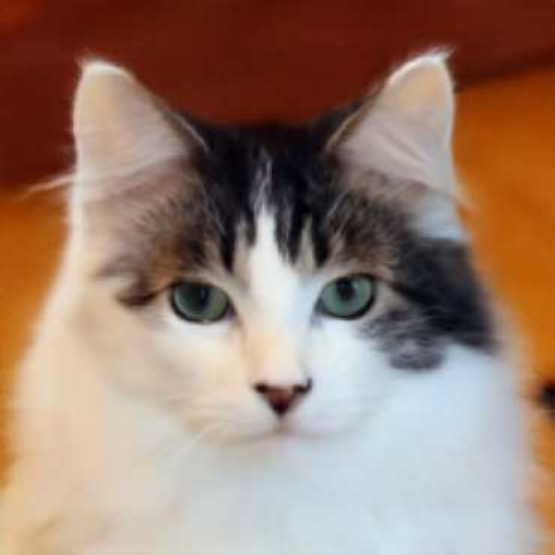}\\\tiny PSNR: 34.25} &
        \makecell[c]{\includegraphics[width=\linewidth]{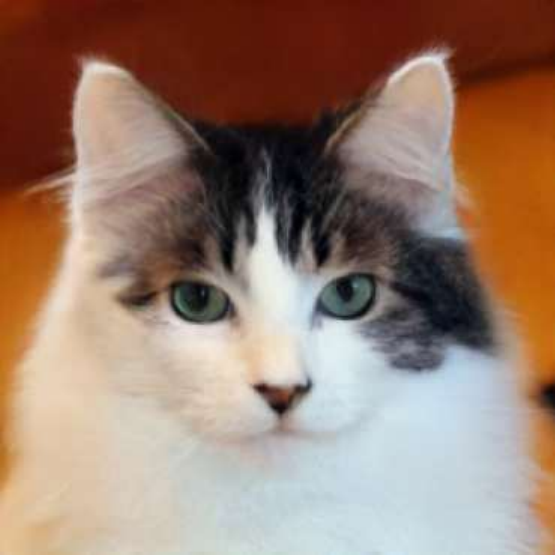}\\\tiny PSNR: 31.86} &
        \makecell[c]{\includegraphics[width=\linewidth]{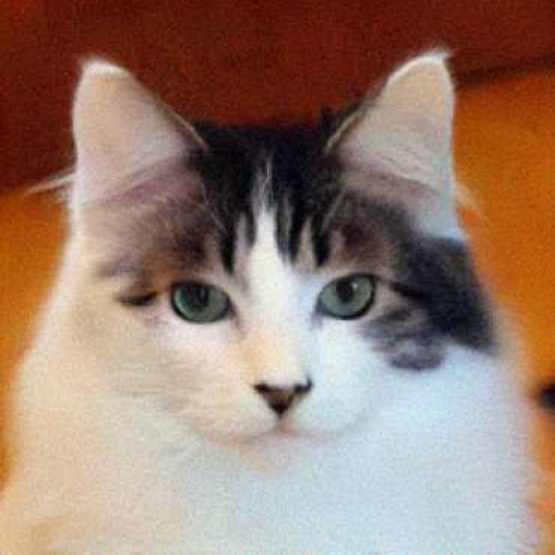}\\\tiny PSNR: 26.40} &
        \makecell[c]{\includegraphics[width=\linewidth]{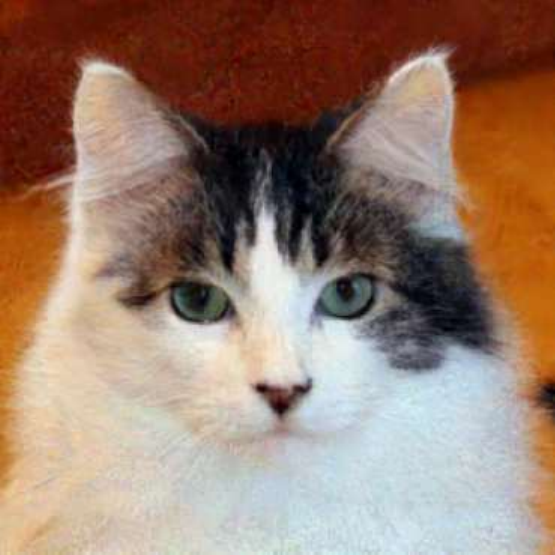}\\\tiny PSNR: 31.10} &
        \makecell[c]{\includegraphics[width=\linewidth]{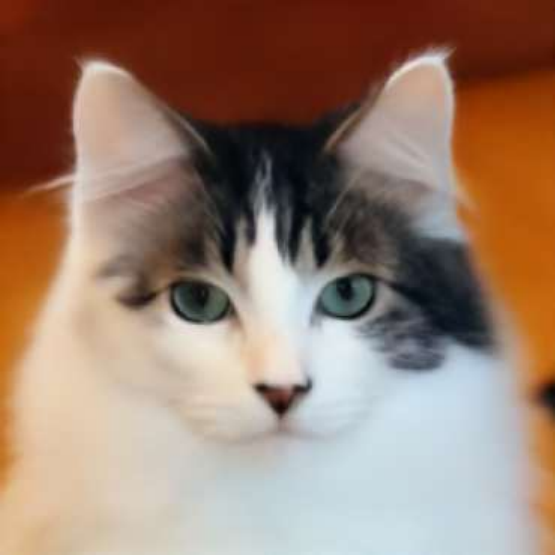}\\\tiny PSNR: 32.85} &
        \makecell[c]{\includegraphics[width=\linewidth]{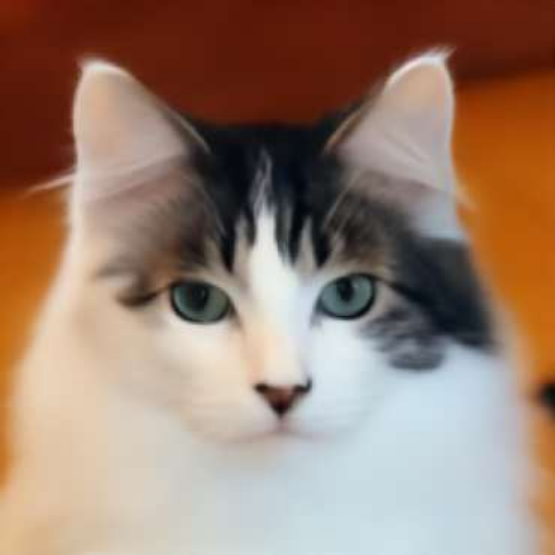}\\\tiny PSNR: 33.07} &
        \makecell[c]{\includegraphics[width=\linewidth]{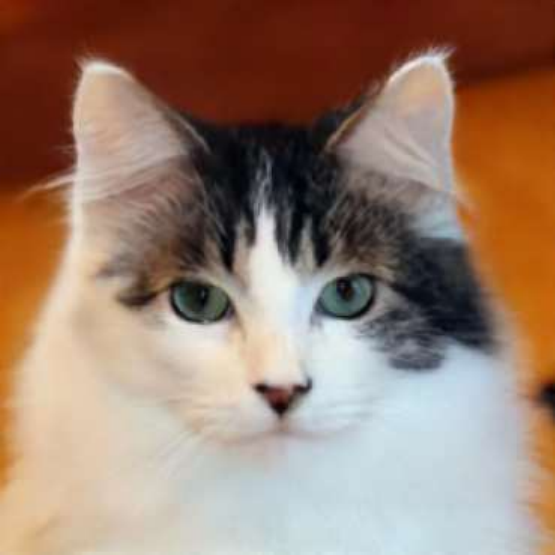}\\\tiny PSNR: 33.79} &
        \makecell[c]{\includegraphics[width=\linewidth]{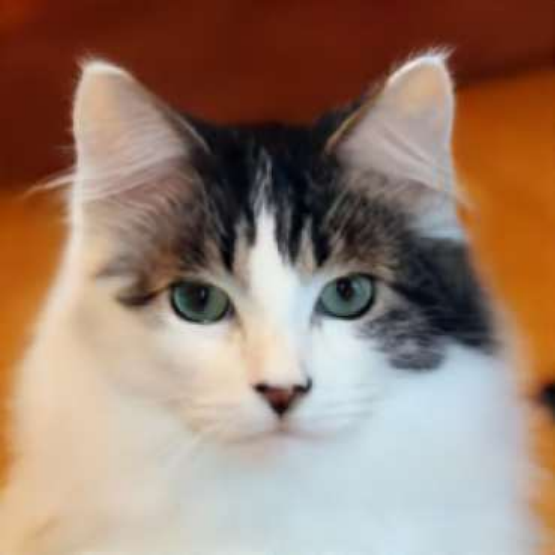}\\\tiny PSNR: 34.16}
        \\[5pt]
        
        \makecell[c]{\includegraphics[width=\linewidth]{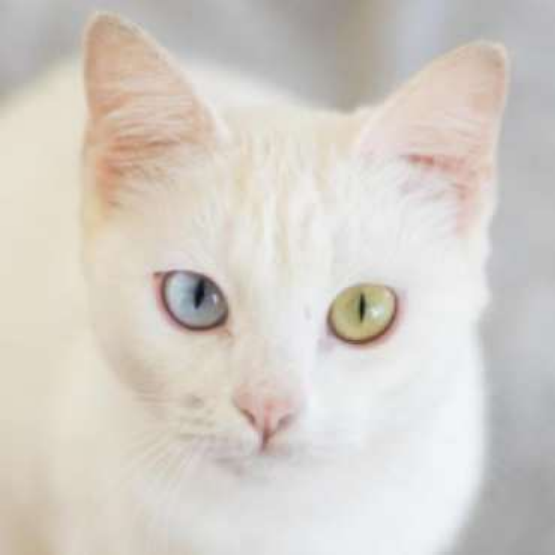}\\\vphantom{\tiny PSNR: 00.00}} &
        \makecell[c]{\includegraphics[width=\linewidth]{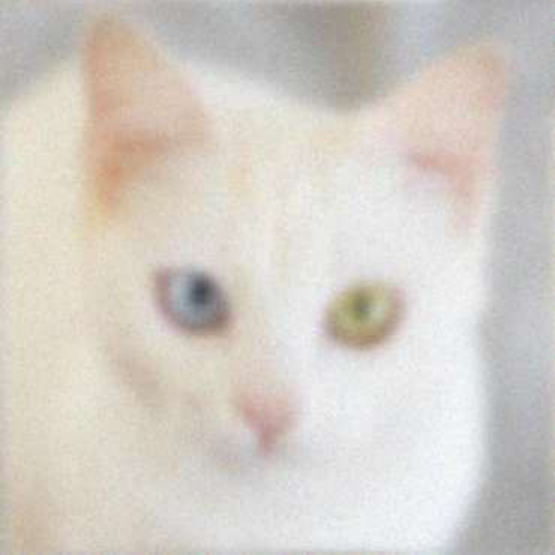}\\\tiny PSNR: 28.89} &
        \makecell[c]{\includegraphics[width=\linewidth]{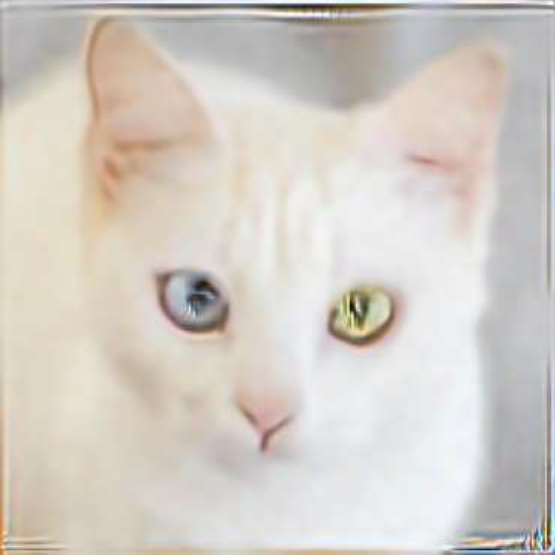}\\\tiny PSNR: 25.00} &
        \makecell[c]{\includegraphics[width=\linewidth]{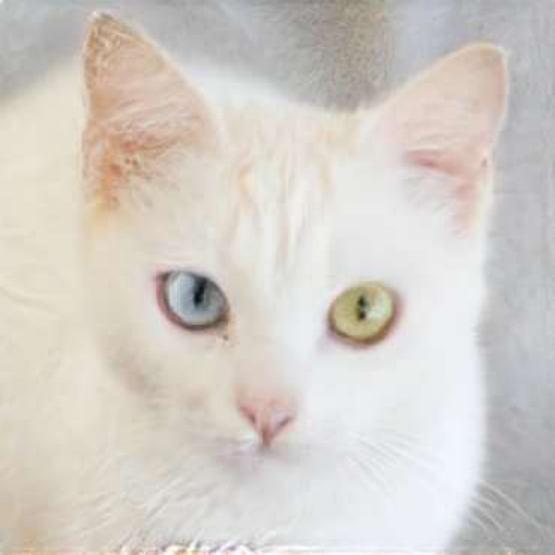}\\\tiny PSNR: 32.34} &
        \makecell[c]{\includegraphics[width=\linewidth]{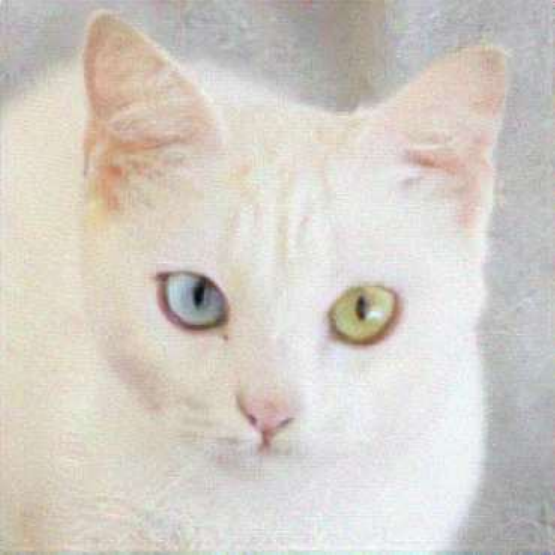}\\\tiny PSNR: 26.84} &
        \makecell[c]{\includegraphics[width=\linewidth]{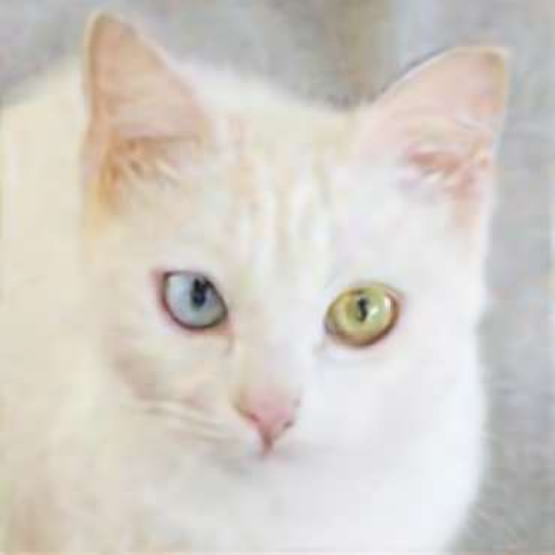}\\\tiny PSNR: 33.18} &
        \makecell[c]{\includegraphics[width=\linewidth]{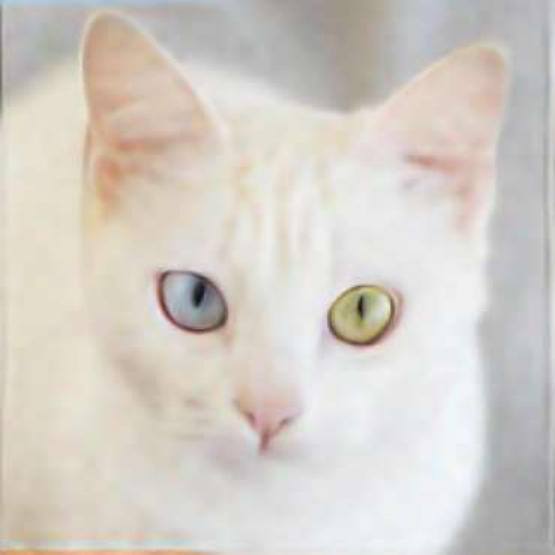}\\\tiny PSNR: 32.43} &
        \makecell[c]{\includegraphics[width=\linewidth]{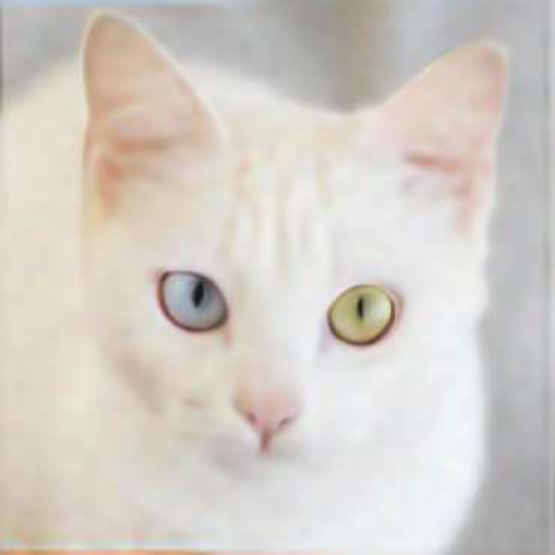}\\\tiny PSNR: 33.60} &
        \makecell[c]{\includegraphics[width=\linewidth]{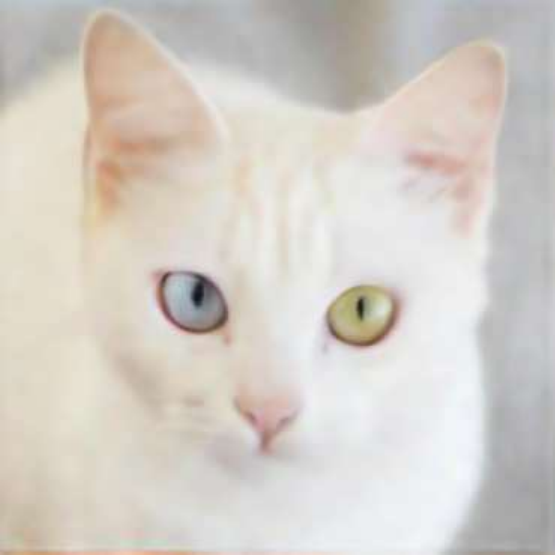}\\\tiny PSNR: 34.67} &
        \makecell[c]{\includegraphics[width=\linewidth]{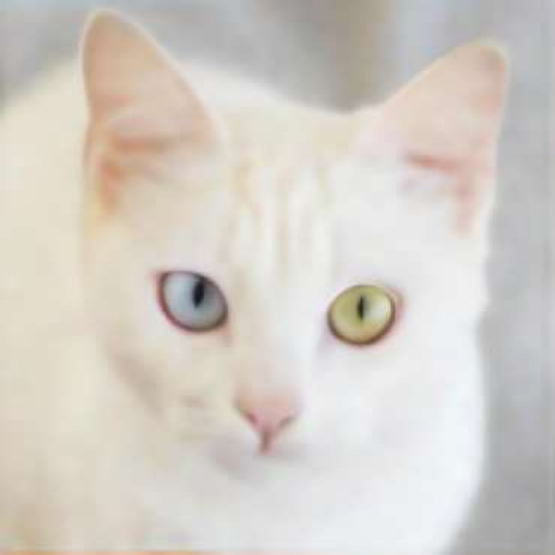}\\\tiny PSNR: 35.81}
        \\[5pt]
        
        \makecell[c]{\includegraphics[width=\linewidth]{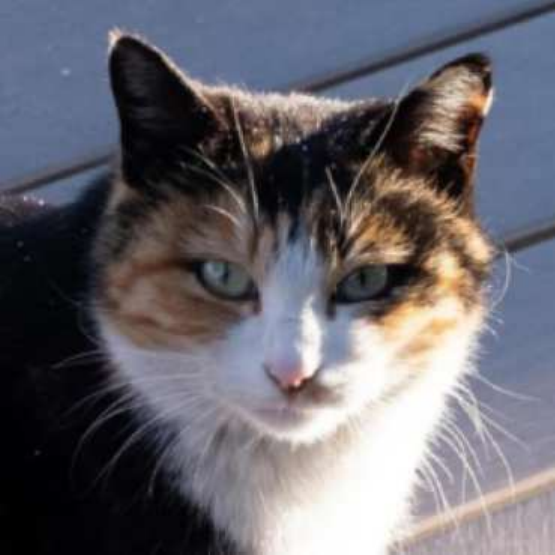}\\\vphantom{\tiny PSNR: 00.00}} &
        \makecell[c]{\includegraphics[width=\linewidth]{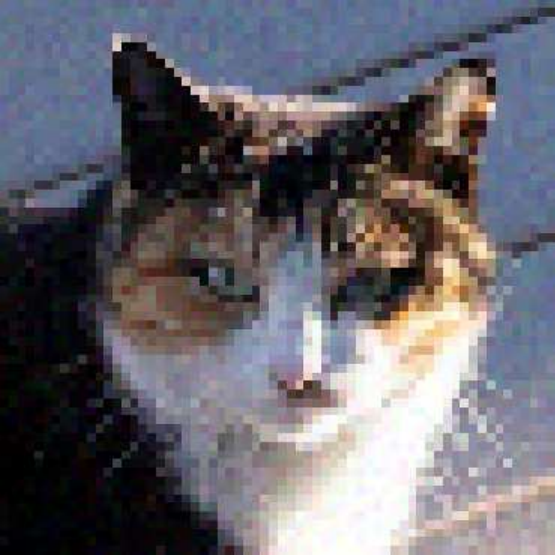}\\\tiny PSNR: 9.70} &
        \makecell[c]{\includegraphics[width=\linewidth]{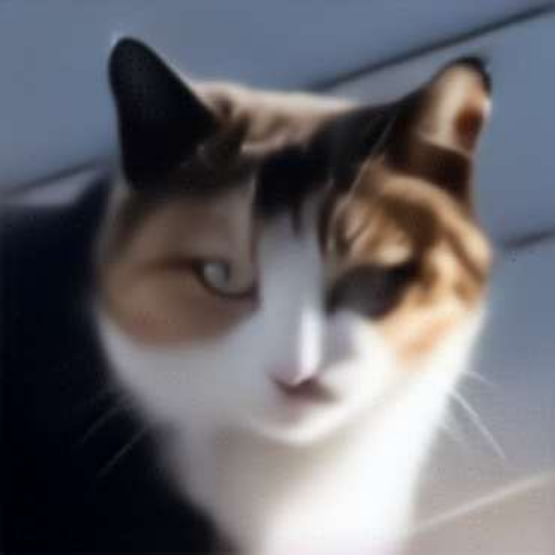}\\\tiny PSNR: 25.48} &
        \makecell[c]{\includegraphics[width=\linewidth]{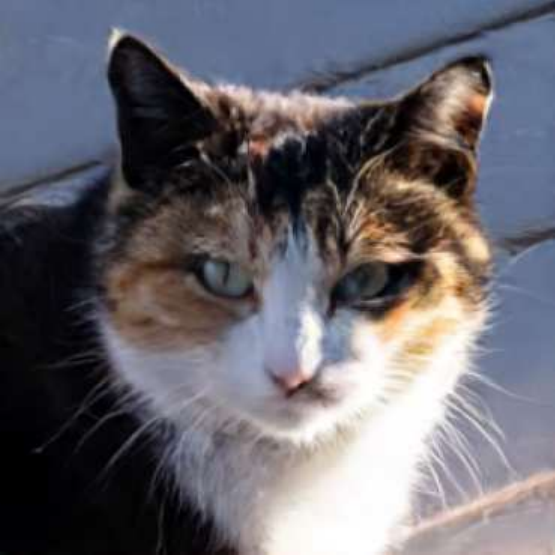}\\\tiny PSNR: 25.07} &
        \makecell[c]{\includegraphics[width=\linewidth]{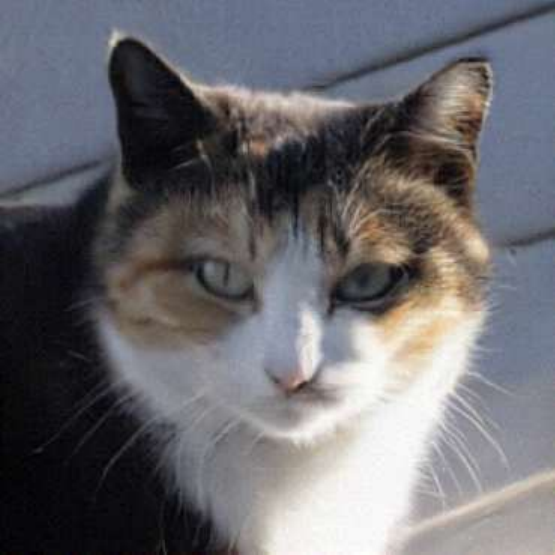}\\\tiny PSNR: 23.99} &
        \makecell[c]{\includegraphics[width=\linewidth]{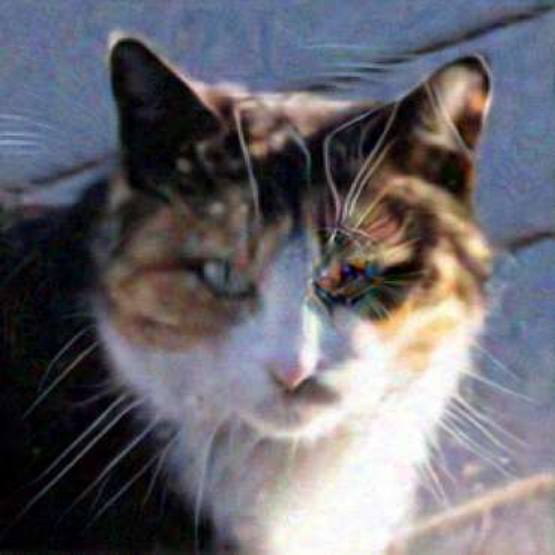}\\\tiny PSNR: 23.73} &
        \makecell[c]{\includegraphics[width=\linewidth]{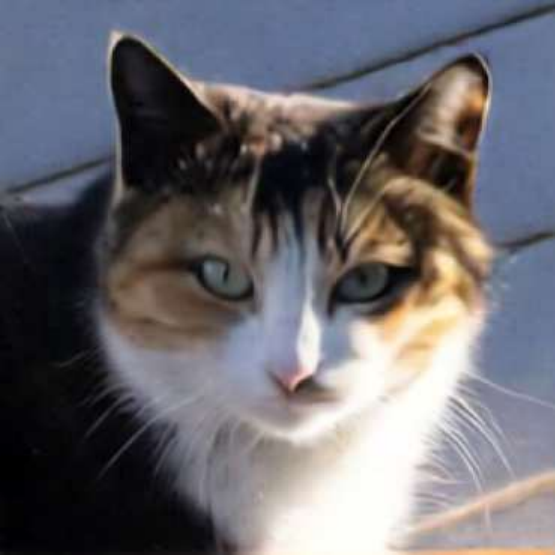}\\\tiny PSNR: 26.02} &
        \makecell[c]{\includegraphics[width=\linewidth]{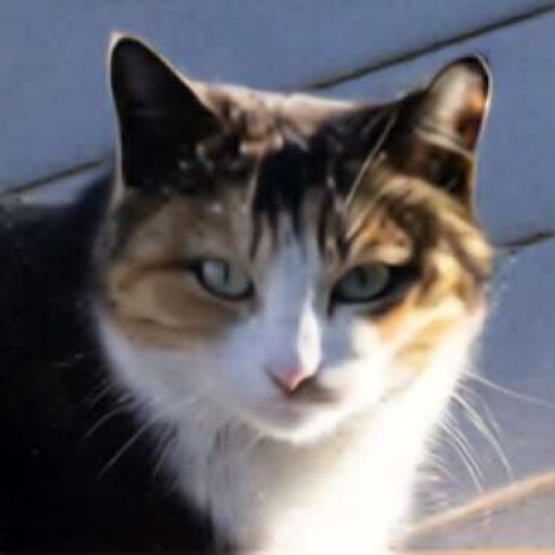}\\\tiny PSNR: 26.73} &
        \makecell[c]{\includegraphics[width=\linewidth]{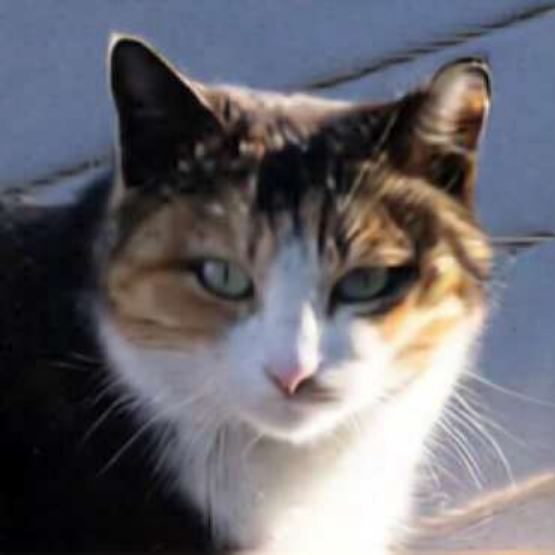}\\\tiny PSNR: 25.98} &
        \makecell[c]{\includegraphics[width=\linewidth]{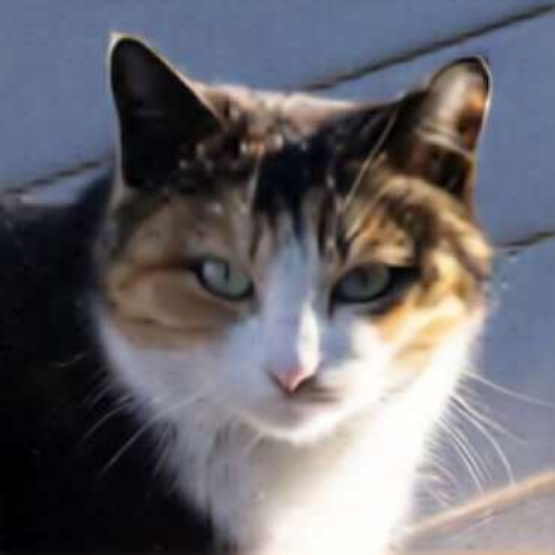}\\\tiny PSNR: 26.75}
        \\[5pt]
        
        \makecell[c]{\includegraphics[width=\linewidth]{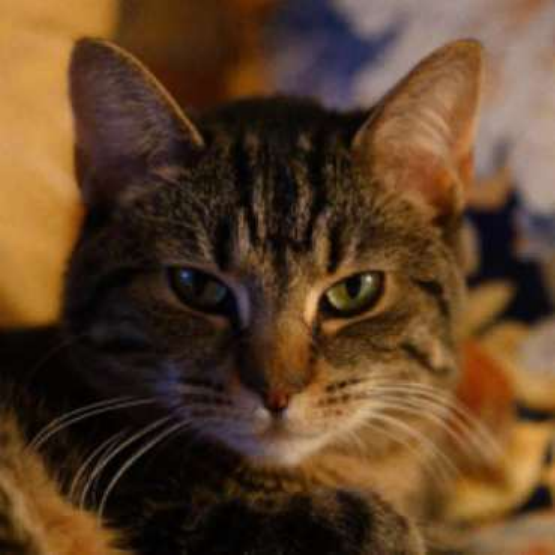}\\\vphantom{\tiny PSNR: 00.00}} &
        \makecell[c]{\includegraphics[width=\linewidth]{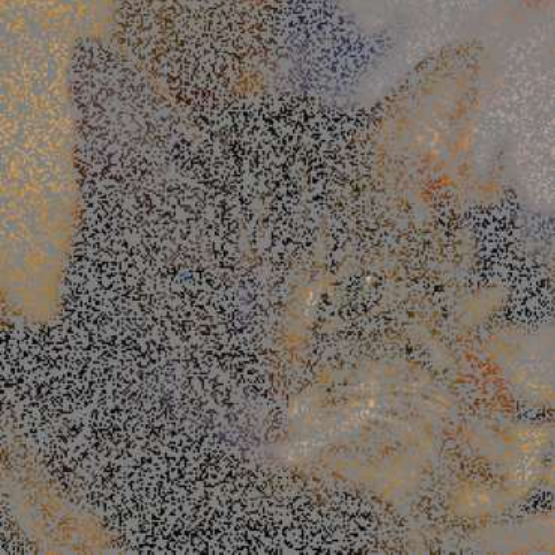}\\\tiny PSNR: 11.49} &
        \makecell[c]{\includegraphics[width=\linewidth]{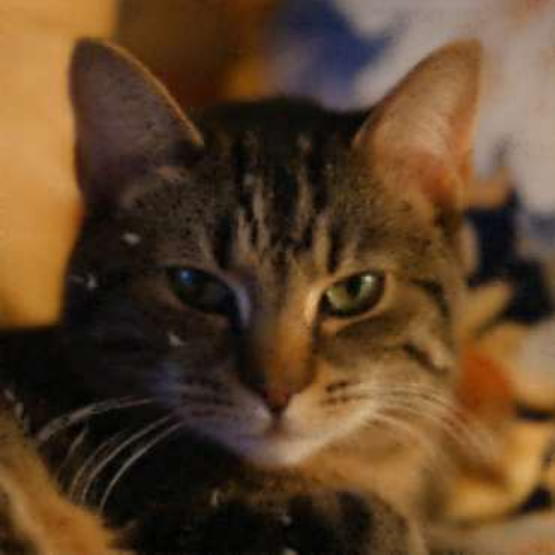}\\\tiny PSNR: 30.54} &
        \makecell[c]{\includegraphics[width=\linewidth]{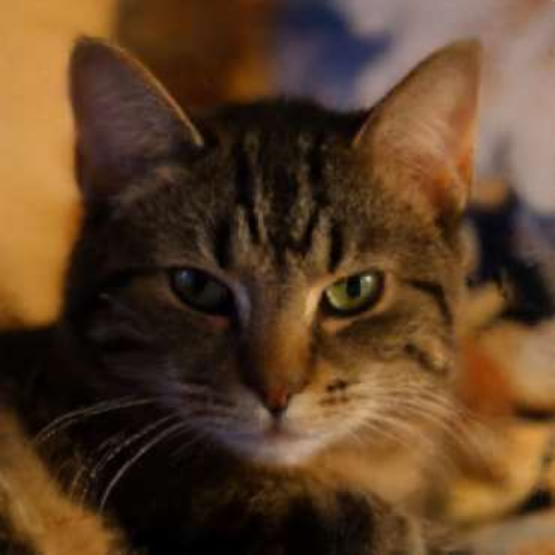}\\\tiny PSNR: 30.15} &
        \makecell[c]{\includegraphics[width=\linewidth]{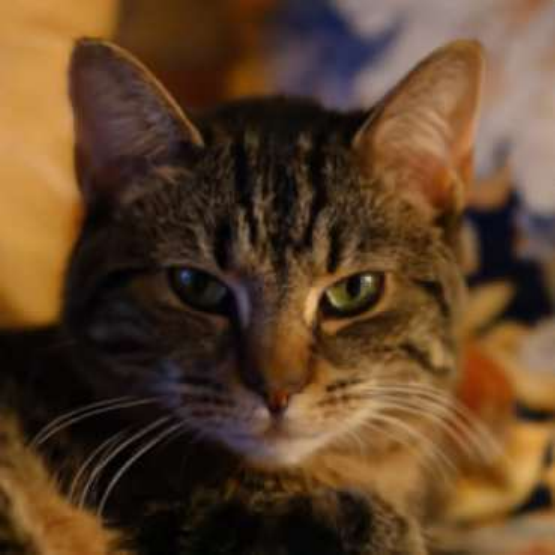}\\\tiny PSNR: 33.52} &
        \makecell[c]{\includegraphics[width=\linewidth]{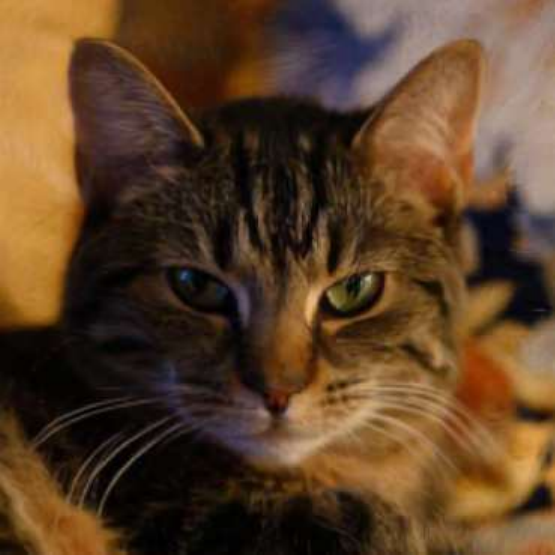}\\\tiny PSNR: 33.35} &
        \makecell[c]{\includegraphics[width=\linewidth]{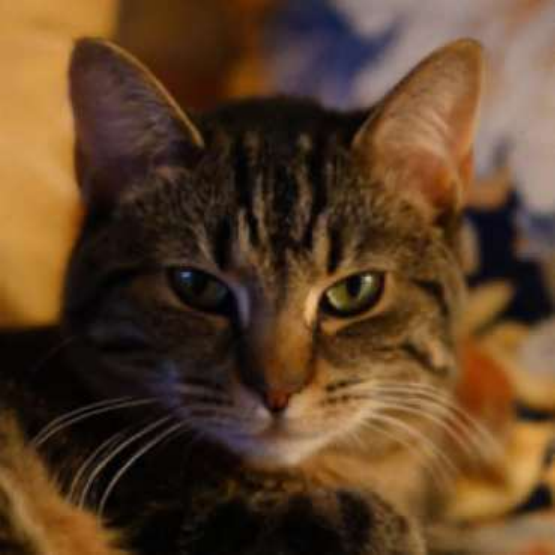}\\\tiny PSNR: 34.69} &
        \makecell[c]{\includegraphics[width=\linewidth]{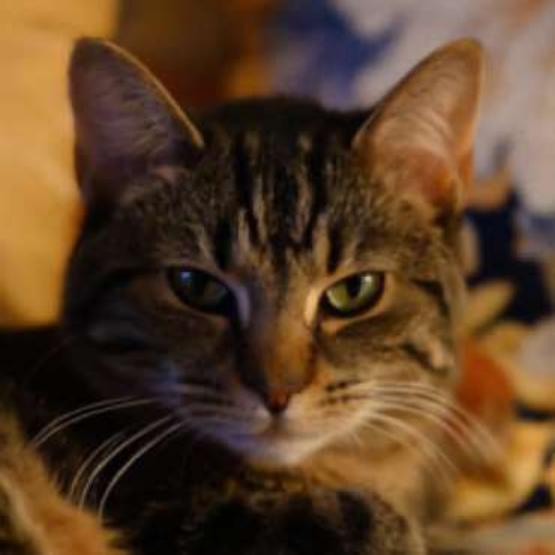}\\\tiny PSNR: 35.39} &
        \makecell[c]{\includegraphics[width=\linewidth]{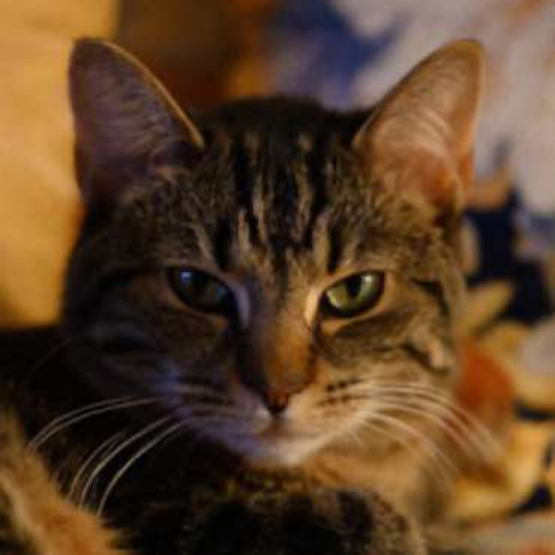}\\\tiny PSNR: 34.87} &
        \makecell[c]{\includegraphics[width=\linewidth]{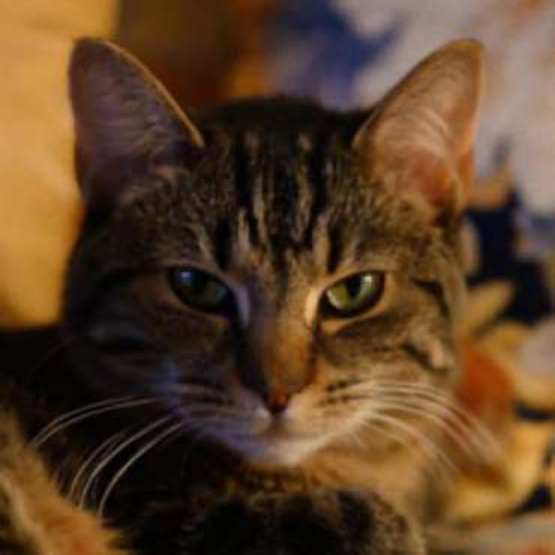}\\\tiny PSNR: 35.45}
        \\[5pt]
        
        \makecell[c]{\includegraphics[width=\linewidth]{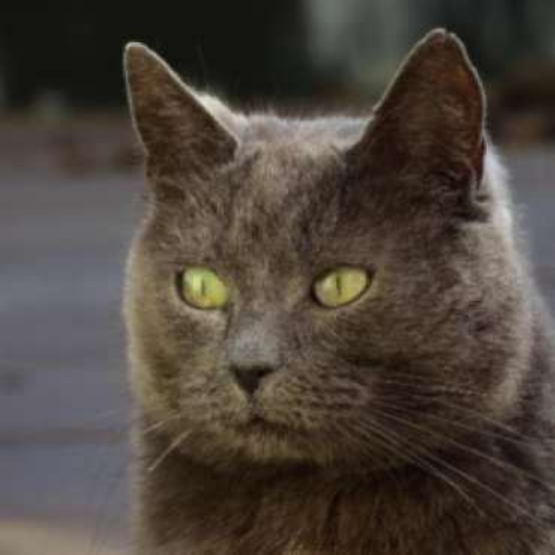}\\\vphantom{\tiny PSNR: 00.00}} &
        \makecell[c]{\includegraphics[width=\linewidth]{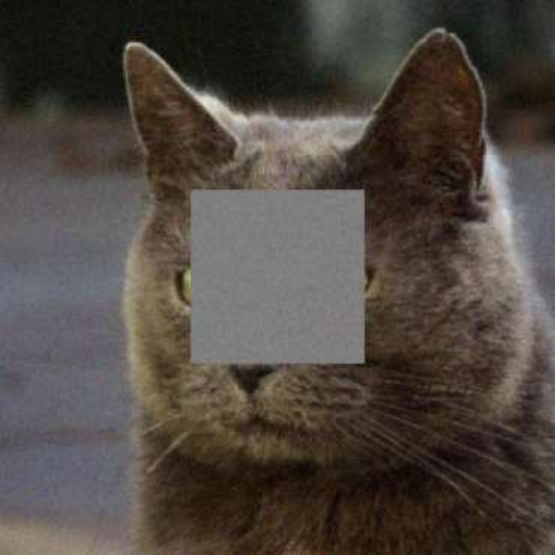}\\\tiny PSNR: 22.89} &
        \makecell[c]{\text{N/A}\\\vphantom{\tiny PSNR: 00.00}} & 
        \makecell[c]{\includegraphics[width=\linewidth]{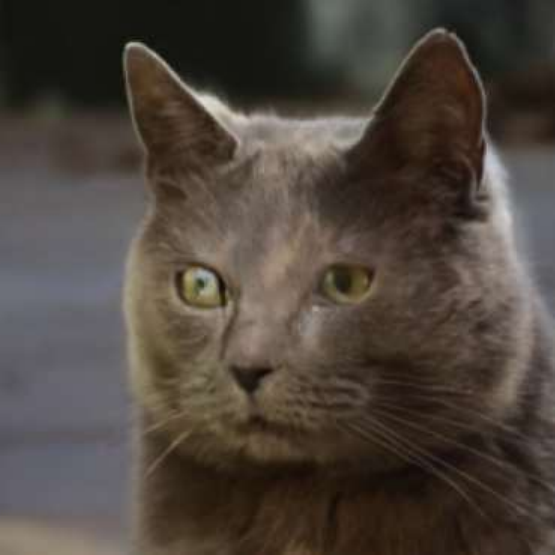}\\\tiny PSNR: 26.94} &
        \makecell[c]{\includegraphics[width=\linewidth]{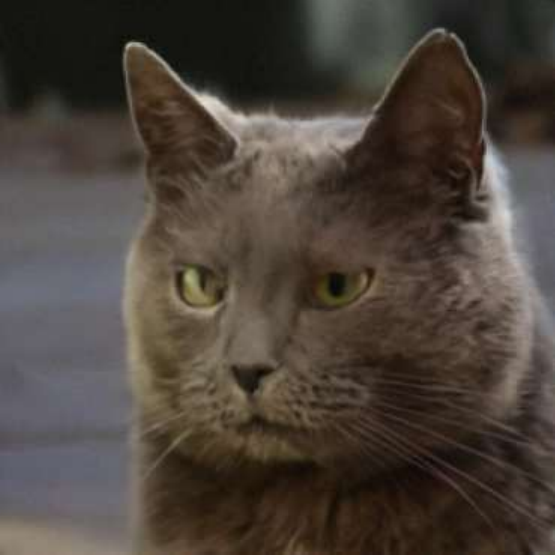}\\\tiny PSNR: 27.95} &
        \makecell[c]{\includegraphics[width=\linewidth]{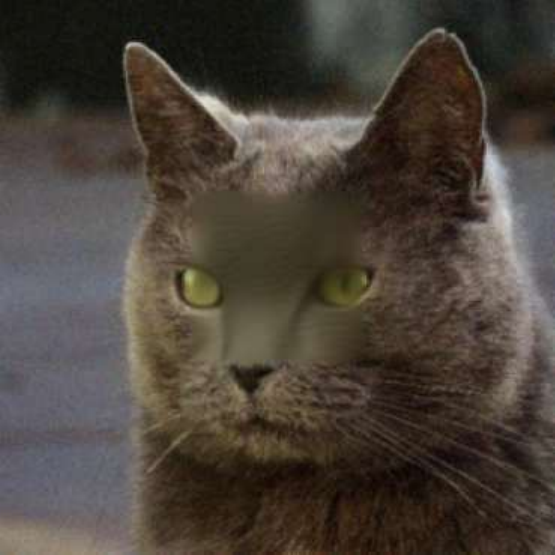}\\\tiny PSNR: 28.17} &
        \makecell[c]{\includegraphics[width=\linewidth]{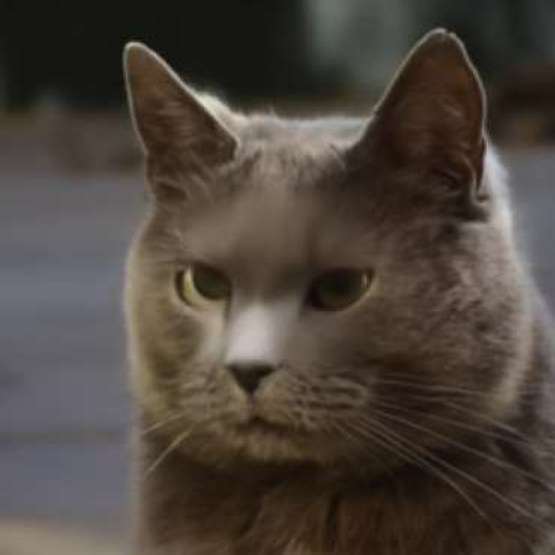}\\\tiny PSNR: 26.28} &
        \makecell[c]{\includegraphics[width=\linewidth]{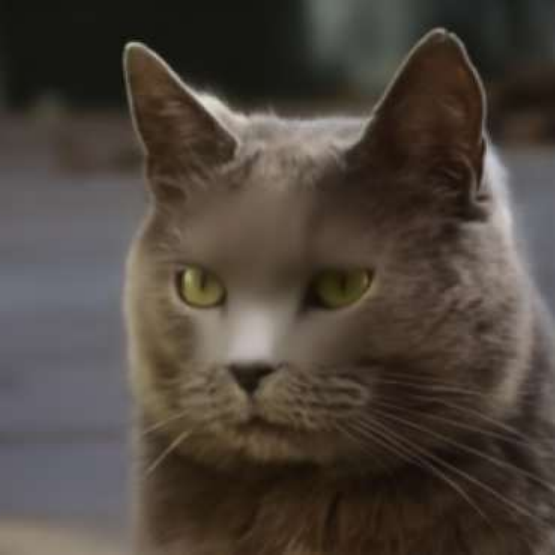}\\\tiny PSNR: 27.64} &
        \makecell[c]{\includegraphics[width=\linewidth]{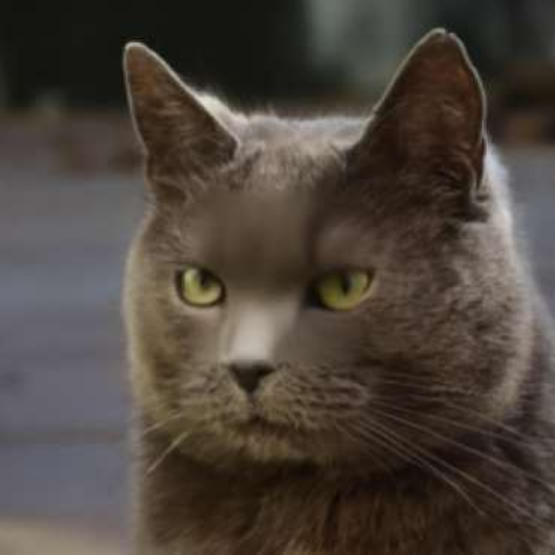}\\\tiny PSNR: 28.02} &
        \makecell[c]{\includegraphics[width=\linewidth]{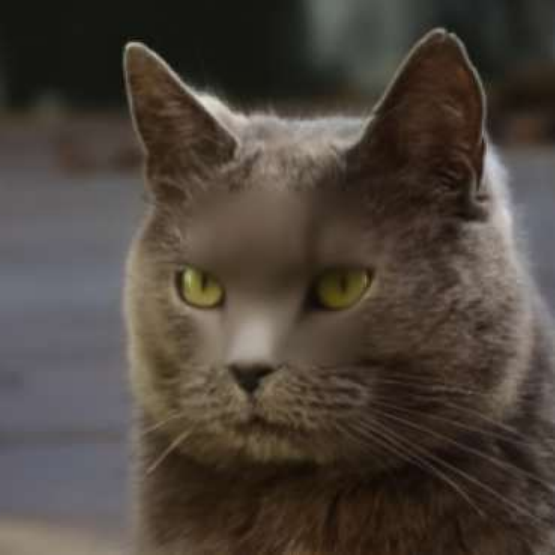}\\\tiny PSNR: 28.73}
        \\[3pt]
    \end{tabular}
    \caption{PSNR comparison of all methods on AFHQ-Cat across five restoration tasks (from top to bottom: denoising, deblurring, super-resolution, random inpainting, and box inpainting).}
    \label{fig:restore_compare_afhq}
\end{figure}
\FloatBarrier

Figures \ref{fig:restore_compare_celeba} and
\ref{fig:restore_compare_afhq} further illustrate the strong and consistent restoration performance of HP-PnP-Flow across all five degradation tasks on both datasets. 
On CelebA, our algorithm achieves clear PSNR improvements over PnP-Flow in super-resolution, random inpainting,
and box inpainting. 
Even under the single-sampling setting, HP-PnP-Flow(1) matches or surpasses most competing methods at their
best-performing settings, while typically requiring fewer iterations.

In terms of visual quality, HP-PnP-Flow produces consistently clean reconstructions across the five tasks, with few noticeable artifacts, although mild over-smoothing may occur in some cases. 
By comparison, PnP-GS is less competitive on generative restoration tasks, while D-Flow performs less favorably on denoising and deblurring. 
Flow-Priors tends to introduce more visible noise and artifacts, whereas OT-ODE preserves textures reasonably well but is less effective on the more ill-posed
restoration tasks. 
Overall, HP-PnP-Flow provides a favorable balance of
reconstruction quality and robustness across different degradation settings.


\subsection{Ablation Study}
\label{sec:ablation}
We conduct the ablation study on random inpainting with $70\%$ missing pixels and Gaussian noise $\sigma=0.01$, using $N=100$ iterations and set $\alpha=0.01$. 
Results are averaged over $10$ test images and reported as
the mean~$\pm$~standard deviation of PSNR (dB).\\


\noindent
\textbf{Initial offset $t_0$.} 
To isolate the effect of the initial time offset, we fix a linear schedule by setting $K=0$ and vary $t_0\in\{0, 0.2,0.5,0.7,0.9\}$, with $t_0=0$ corresponding to the
original PnP-Flow.

\FloatBarrier
\begin{figure}[htp]
    \centering
    \includegraphics[width=0.95\textwidth]{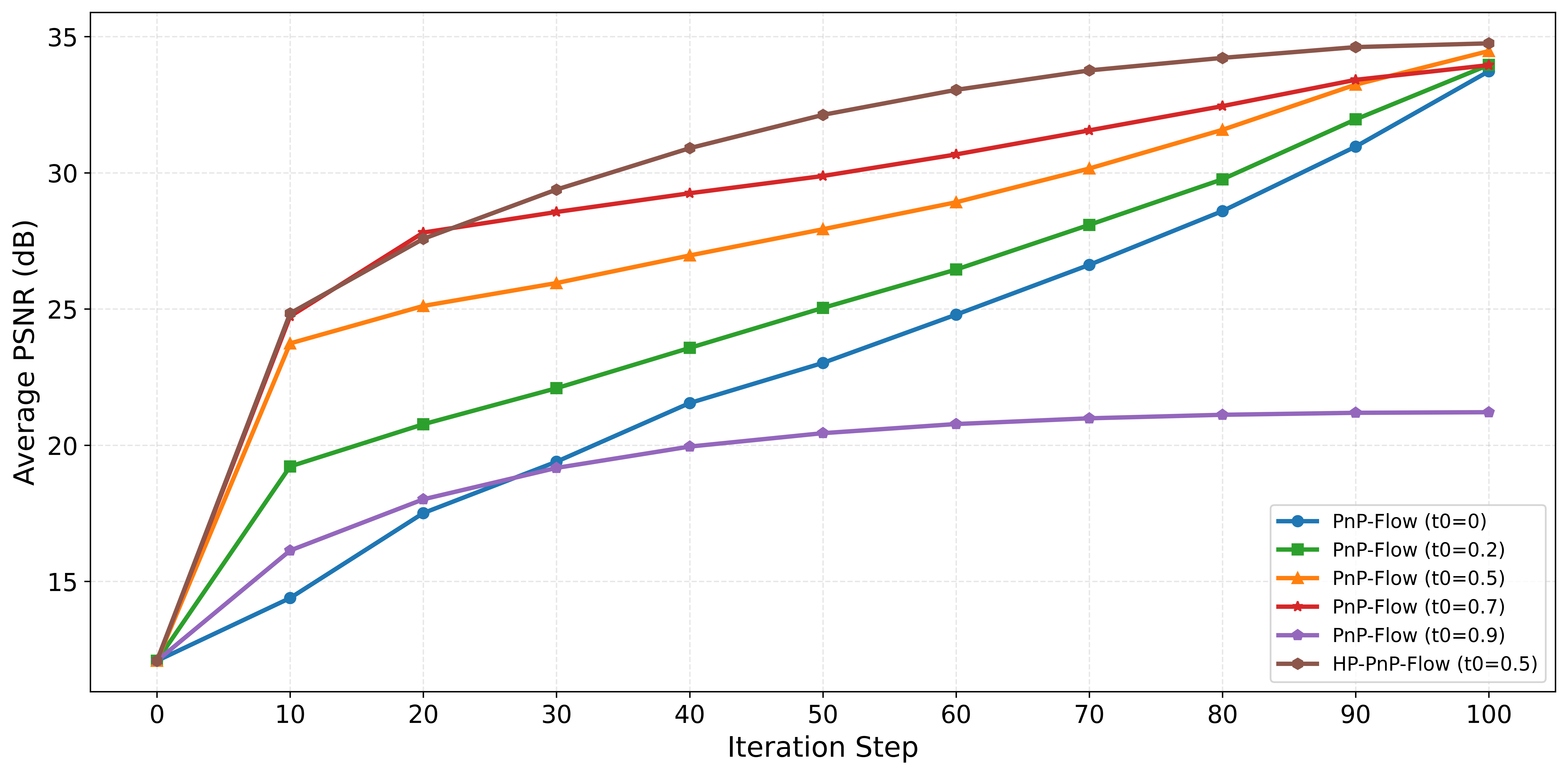}
    \caption{PSNR convergence for different initial offsets $t_0$.}
    \label{fig:ablation_t0}
\end{figure}
\FloatBarrier

\FloatBarrier
\begin{table}[htbp]
  \centering
  \caption{PSNR (dB) for different initial offsets $t_0$.}
  \label{tab:l0_ablation}
  \begin{tabular}{@{}l c c c c c c@{}}
    \toprule
    & PnP-Flow ($t_0=0$) & $t_0=0.2$ & $t_0=0.5$ & $t_0=0.7$ & $t_0=0.9$ & HP-PnP-Flow \\
    \midrule
    Mean PSNR & 33.73 & 33.96 & 34.47 & 33.95 & 21.21 & \textbf{34.75} \\
    Std.\ dev. & 1.07  & 1.12  & 1.19  & 2.61  & 1.56  & 1.34  \\
    \bottomrule
  \end{tabular}
\end{table}
\FloatBarrier

Figure \ref{fig:ablation_t0} shows the PSNR trajectories for different choices of $t_0$. 
Compared with the baseline $t_0=0$, a moderate positive
initial offset accelerates early-stage convergence. 
In particular, setting $t_0=0.5$ yields a mean PSNR of $34.47\,\mathrm{dB}$, improving upon PnP-Flow by $0.74\,\mathrm{dB}$ while maintaining low variability.
For $t_0\in[0.2,0.7]$, the final PSNR consistently exceeds that of the baseline, whereas performance deteriorates when $t_0$ becomes too large (e.g., $t_0=0.9$). 
In addition, as revealed by Table \ref{tab:l0_ablation}, our algorithm achieves a further $0.28\,\mathrm{dB}$ improvement over the best $t_0$-only variant.\\


\noindent\textbf{Transition coefficient $K$.}
For the ablation study, we additionally consider the endpoint cases. We define \(K=0\) as the purely linear schedule and \(K=1\) as the purely geometric schedule \(t_k=1-(1-t_0)\lambda^k\).
To evaluate the effect of the transition coefficient $K$, we fix $\lambda=0.965$ and vary $K\in\{0,0.2,0.5,0.7,1.0\}$.
The choice $K=0$ recovers a purely linear schedule used in \cite{pnp_flow}, while $K=1$ corresponds to a purely geometric schedule.


\FloatBarrier
\begin{figure}[htp]
    \centering
    \includegraphics[width=0.95\textwidth]{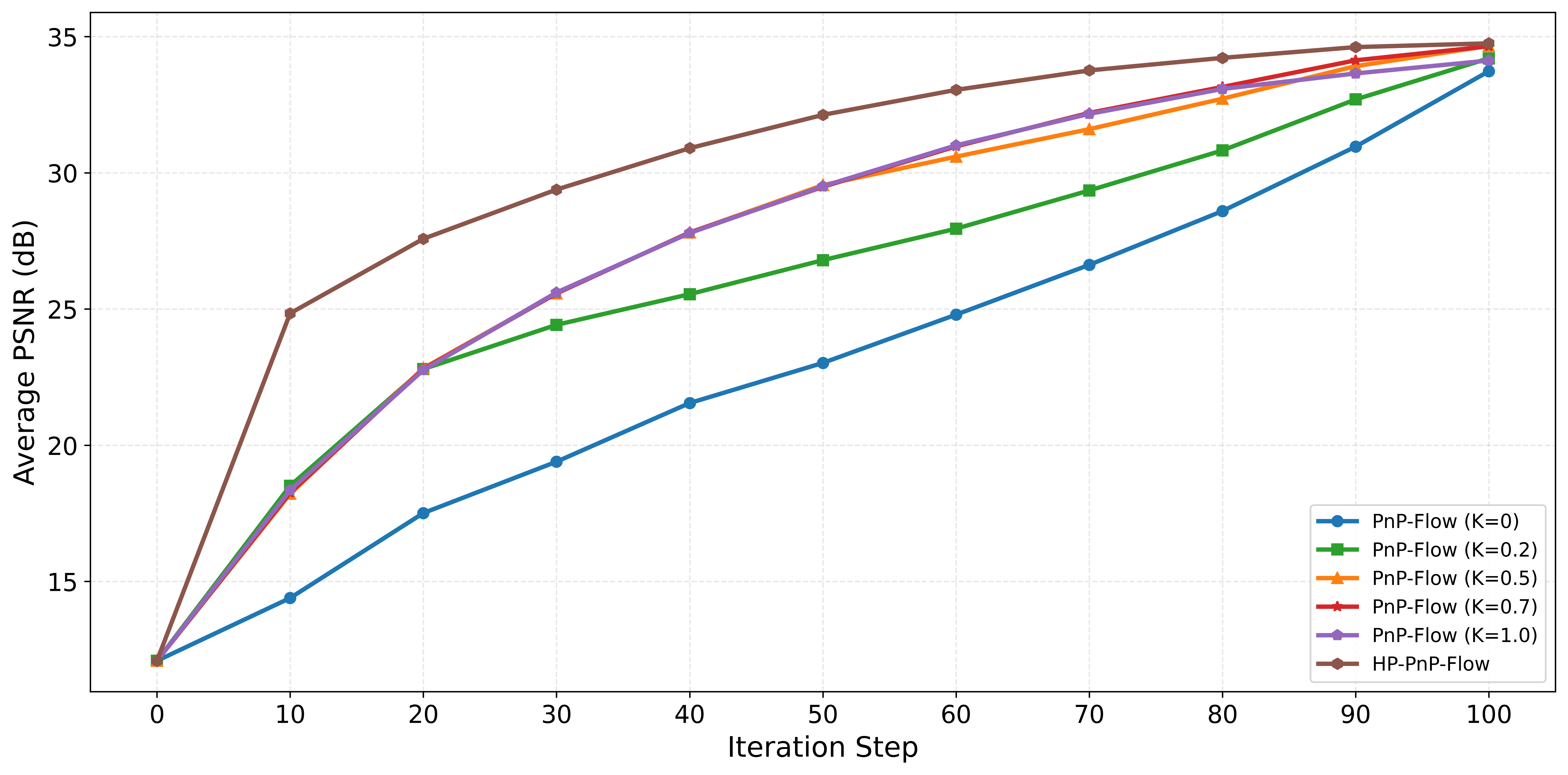}
    \caption{PSNR convergence for different transition coefficients $K$.}
    \label{fig:ablation_K}
\end{figure}
\FloatBarrier

\FloatBarrier
\begin{table}[htbp]
  \centering
  \caption{PSNR (dB) for different transition coefficients $K$.}
  \label{tab:K_ablation}
  \begin{tabular}{@{}l c c c c c c@{}}
    \toprule
    & PnP-Flow ($K=0$) & $K=0.2$ & $K=0.5$ & $K=0.7$ & $K=1.0$ & HP-PnP-Flow \\
    \midrule
    Mean PSNR & 33.73 & 34.20 & 34.66 & 34.65 & 34.12 & \textbf{34.75} \\
    Std.\ dev. & 1.07  & 1.17  & 1.24  & 1.18  & 1.07  & 1.34  \\
    \bottomrule
  \end{tabular}
\end{table}
\FloatBarrier

From Figure \ref{fig:ablation_K}, we observe that a larger $K$ accelerates early-stage convergence by increasing the contribution of the geometric phase.
The purely linear schedule ($K=0$) exhibits the slowest early-stage convergence, whereas larger values of $K$ lead to more rapid initial PSNR improvement.
When $K=0.5$, HP-PnP-Flow achieves the highest final PSNR among the single-factor variants ($34.66\,\mathrm{dB}$), outperforming both the purely linear ($K=0$) and purely geometric ($K=1$) schedules.
This suggests that the two-phase design effectively combines the complementary advantages of geometric growth for rapid early recovery and linear refinement for stable convergence.

As shown in Table \ref{tab:K_ablation}, the proposed HP-PnP-Flow, which combines the optimal $t_0$, the hybrid time schedule, and the projection operator, achieves both the fastest early-stage convergence and the highest final PSNR ($34.75\,\mathrm{dB}$) among all configurations,
suggesting that these three components provide complementary benefits.

\section{Conclusion}
\label{sec:conclusion}
In this work, we propose an efficient hybrid-projection Plug-and-Play-Flow algorithm, an improved Plug-and-Play algorithm for image inverse problems that integrates flow matching priors with a hybrid time schedule and pixel-space projection operator.
Furthermore, we establish the convergence of HP-PnP-Flow without imposing the boundedness assumption on the iterates required in original PnP-Flow (see Proposition 4 in \cite{pnp_flow}), since the pixel-space projection keeps the iterates within the prescribed pixel range.
The hybrid schedule combines geometric and linear growth phases, satisfies the convergence condition, and achieves faster early-stage reconstruction and improved numerical performance over the linear schedule used in PnP-Flow.
Experiments on CelebA and AFHQ-Cat across five degradation tasks show that HP-PnP-Flow achieves the best or second-best PSNR, SSIM, and LPIPS in most settings, while requiring fewer iterations than PnP-Flow (e.g., $100$ vs.\
$500$ iterations for deblurring on AFHQ-Cat). 
Ablation studies further demonstrate that the initial time offset $t_0$ and the hybrid schedule contribute independently to the performance gains, with their combination yielding the best overall performance.

Several directions remain for future investigation. 
First, we plan to establish quantitative error bounds for HP-PnP-Flow under mild assumptions. 
Second, it would be interesting to explore whether the proposed hybrid time schedule can be extended to other generative-model-based PnP frameworks beyond flow matching.




\bibliography{references}

\appendix
\renewcommand{\thesection}{\Alph{section}}

\section{Hyperparameter Configurations}
\label{sec:appendix_hyperparams}

We report the optimal hyperparameter settings for HP-PnP-Flow and all baseline methods on CelebA (Table \ref{tab:hyperparams_celeba}) and AFHQ-Cat
(Table \ref{tab:hyperparams_afhqcat}). These settings are selected by grid search on the validation set to maximize PSNR. A dash  $"-"$ denotes that the corresponding method is not applicable to the given degradation task.

\FloatBarrier
\begin{table}[htbp]
  \centering
  \caption{Optimal hyperparameter settings on CelebA.}
  \label{tab:hyperparams_celeba}
  \setlength{\tabcolsep}{4.5pt}
  \footnotesize
  \begin{tabular}{lccccc}
    \toprule
    Method & Denoising & Deblurring & Super-res.\ & Rand.\ inp.\ & Box inp. \\
    \midrule
    \multicolumn{6}{l}{\textbf{PnP-GS}} \\
    $\gamma$ (step size)            & ---  & 2.0  & 2.0  & ---  & ---  \\
    $\alpha$ (momentum)             & 1.0  & 0.3  & 1.0  & ---  & ---  \\
    $\sigma_f$ (noise factor)       & 1.0  & 1.8  & 3.0  & ---  & ---  \\
    $n_{\text{iter}}$ (iterations)  & 1    & 35   & 20   & 23   & ---  \\
    \addlinespace[0.5em]
    \multicolumn{6}{l}{\textbf{OT-ODE}} \\
    $t_0$ (initial time)            & 0.3  & 0.4  & 0.1  & 0.1  & 0.1  \\
    $\gamma$                        & $\sqrt{t}$ & $\sqrt{t}$ & constant & constant & $\sqrt{t}$ \\
    \addlinespace[0.5em]
    \multicolumn{6}{l}{\textbf{Flow-Priors}} \\
    $\lambda$ (regularization)      & 100  & 1000 & 10000 & 10000 & 10000 \\
    $\eta$ (learning rate)          & 0.01 & 0.01 & 0.1   & 0.01  & 0.01  \\
    \addlinespace[0.5em]
    \multicolumn{6}{l}{\textbf{D-Flow}} \\
    $\lambda$ (regularization)      & 0.001 & 0.001 & 0.001 & 0.01  & 0.001 \\
    $\alpha$ (blending)             & 0.1   & 0.1   & 0.1   & 0.1   & 0.1   \\
    $n_{\text{iter}}$ (iterations)  & 3     & 7     & 10    & 20    & 9     \\
    \addlinespace[0.5em]
    \multicolumn{6}{l}{\textbf{PnP-Flow}} \\
    $\alpha$ (step size exponent)   & 0.8  & 0.01 & 0.3  & 0.01 & 0.5  \\
    $N$ (iterations)                & 100  & 100  & 100  & 100  & 100  \\
    \addlinespace[0.5em]
    \multicolumn{6}{l}{\textbf{HP-PnP-Flow (ours)}} \\
    $\eta$ (step size scale)        & 1.0  & 1.0  & 1.0  & 1.0  & 2.0  \\
    $\alpha$ (step size exponent)   & 0.8  & 0.01 & 0.3  & 0.01 & 0.5  \\
    $N$ (iterations)                & 100  & 100  & 100  & 100  & 60   \\
    $t_0$ (initial offset)          & 0.4  & 0.2  & 0.4  & 0.5  & 0.0  \\
    $\lambda$ (decay factor)        & 0.965 & 0.9 & 0.965 & 0.965 & 0.99 \\
    $K$ (transition coefficient)    & 0.3  & 0.3  & 0.7  & 0.7  & 0.2  \\
    \bottomrule
  \end{tabular}

  \vspace{4pt}
  {\footnotesize Note 1: PnP-Flow(1) and PnP-Flow(5) share the same hyperparameters across all tasks. \\
  Note 2: HP-PnP-Flow(1) and HP-PnP-Flow(5) share the same hyperparameters across all tasks.}
\end{table}
\FloatBarrier

\FloatBarrier
\begin{table}[htbp]
  \centering
  \caption{Optimal hyperparameter settings on AFHQ-Cat.}
  \label{tab:hyperparams_afhqcat}
  \setlength{\tabcolsep}{4.5pt}
  \footnotesize
  \begin{tabular}{lccccc}
    \toprule
    Method & Denoising & Deblurring & Super-res.\ & Rand.\ inp.\ & Box inp. \\
    \midrule
    \multicolumn{6}{l}{\textbf{PnP-GS}} \\
    $\gamma$ (step size)            & ---  & 2.0  & 2.0  & ---  & ---  \\
    $\alpha$ (momentum)             & 1.0  & 0.3  & 1.0  & ---  & ---  \\
    $\sigma_f$ (noise factor)       & 1.0  & 1.8  & 5.0  & ---  & ---  \\
    $n_{\text{iter}}$ (iterations)  & 1    & 60   & 50   & 23   & ---  \\
    \addlinespace[0.5em]
    \multicolumn{6}{l}{\textbf{OT-ODE}} \\
    $t_0$ (initial time)            & 0.3  & 0.3  & 0.1  & 0.1  & 0.1  \\
    $\gamma$                        & $\sqrt{t}$ & $\sqrt{t}$ & constant & constant & $\sqrt{t}$ \\
    \addlinespace[0.5em]
    \multicolumn{6}{l}{\textbf{Flow-Priors}} \\
    $\lambda$ (regularization)      & 100  & 1000 & 10000 & 10000 & 10000 \\
    $\eta$ (learning rate)          & 0.01 & 0.01 & 0.1   & 0.01  & 0.01  \\
    \addlinespace[0.5em]
    \multicolumn{6}{l}{\textbf{D-Flow}} \\
    $\lambda$ (regularization)      & 0.001 & 0.001 & 0.001 & 0.001 & 0.01  \\
    $\alpha$ (blending)             & 0.1   & 0.5   & 0.1   & 0.1   & 0.1   \\
    $n_{\text{iter}}$ (iterations)  & 3     & 20    & 20    & 20    & 9     \\
    \addlinespace[0.5em]
    \multicolumn{6}{l}{\textbf{PnP-Flow}} \\
    $\alpha$ (step size exponent)   & 0.8  & 0.01 & 0.01 & 0.01 & 0.5  \\
    $N$ (iterations)                & 100  & 500  & 500  & 200  & 100  \\
    \addlinespace[0.5em]
    \multicolumn{6}{l}{\textbf{HP-PnP-Flow (ours)}} \\
    $\eta$ (step size scale)        & 1.2  & 1.0  & 1.2  & 1.0  & 2.0  \\
    $\alpha$ (step size exponent)   & 0.8  & 0.01 & 0.01 & 0.01 & 0.5  \\
    $N$ (iterations)                & 100  & 100  & 100  & 100  & 60   \\
    $t_0$ (initial offset)          & 0.4  & 0.2  & 0.4  & 0.3  & 0.0  \\
    $\lambda$ (decay factor)        & 0.965 & 0.965 & 0.965 & 0.965 & 0.99 \\
    $K$ (transition coefficient)    & 0.4  & 0.3  & 0.7  & 0.7  & 0.2  \\
    \bottomrule
  \end{tabular}

  \vspace{4pt}
  {\footnotesize Note 1: PnP-Flow(1) and PnP-Flow(5) share the same hyperparameters across all tasks. \\
  Note 2: HP-PnP-Flow(1) and HP-PnP-Flow(5) share the same hyperparameters across all tasks.}
\end{table}

\FloatBarrier
\section{Iteration-wise Comparison with PnP-Flow}
\label{sec:appendix_iteration_comparison}
Figures \ref{fig:app_denoise_celeba}-\ref{fig:app_boxinp_afhq} present
iteration snapshots of HP-PnP-Flow and PnP-Flow across all five restoration tasks and both datasets. All experiments are conducted under the
single-sampling setting.

\FloatBarrier
\begin{figure}[htp]
    \centering
    \begin{subfigure}[b]{0.85\textwidth}
        \centering
        \includegraphics[width=\textwidth]{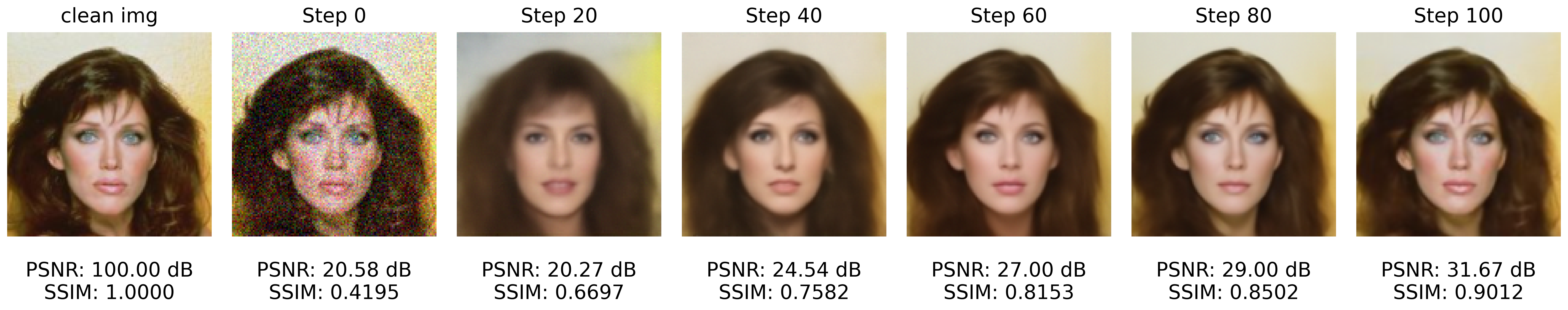}
        \caption{PnP-Flow}
    \end{subfigure}
    \\[6pt]
    \begin{subfigure}[b]{0.85\textwidth}
        \centering
        \includegraphics[width=\textwidth]{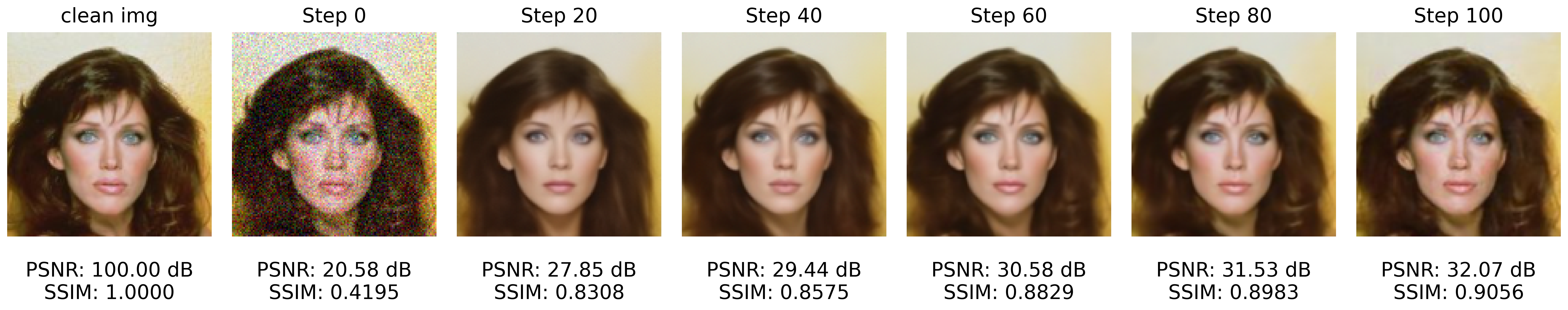}
        \caption{HP-PnP-Flow}
    \end{subfigure}
    \caption{Denoising on CelebA ($\sigma=0.2$).}
    \label{fig:app_denoise_celeba}
\end{figure}
\FloatBarrier

\begin{figure}[htp]
    \centering
    \begin{subfigure}[b]{0.85\textwidth}
        \centering
        \includegraphics[width=\textwidth]{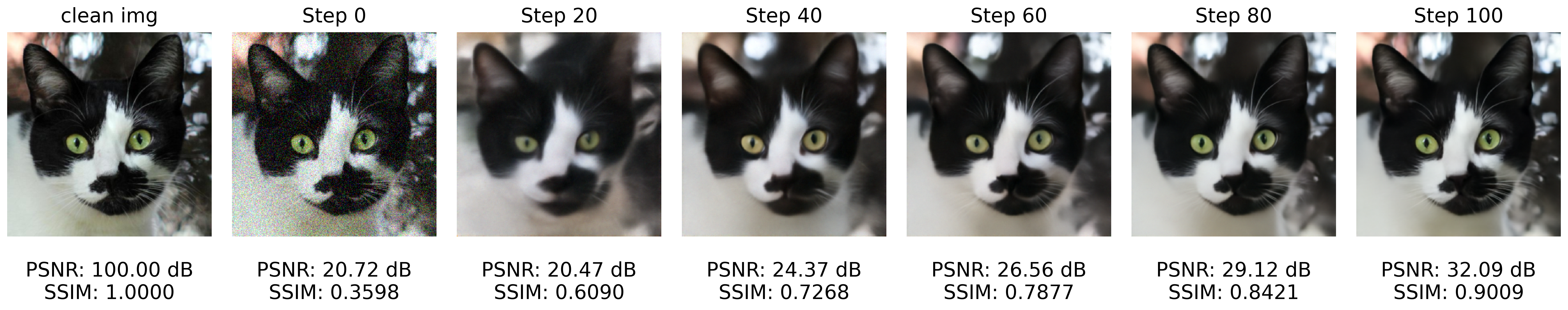}
        \caption{PnP-Flow}
    \end{subfigure}
    \\[6pt]
    \begin{subfigure}[b]{0.85\textwidth}
        \centering
        \includegraphics[width=\textwidth]{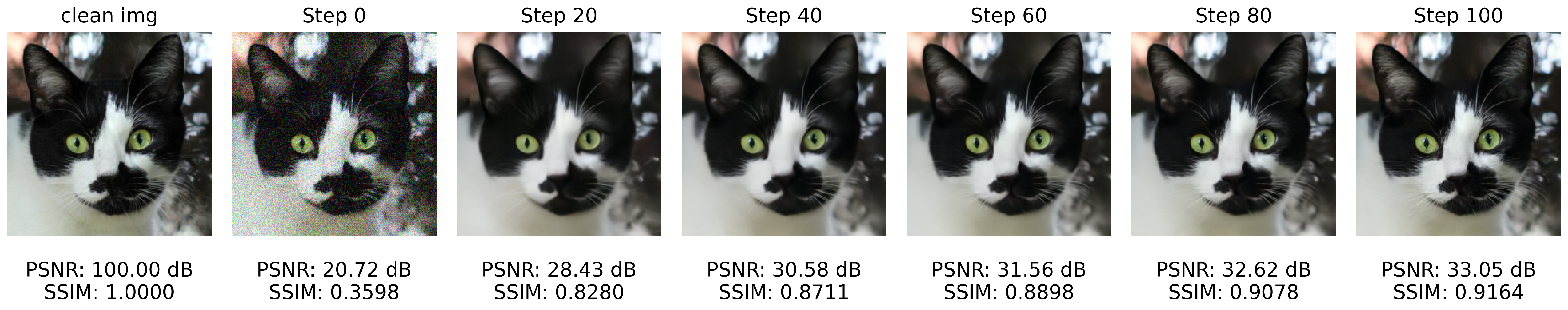}
        \caption{HP-PnP-Flow}
    \end{subfigure}
    \caption{Denoising on AFHQ-Cat ($\sigma=0.2$).}
    \label{fig:app_denoise_afhq}
\end{figure}
\FloatBarrier

\begin{figure}[htp]
    \centering
    \begin{subfigure}[b]{0.85\textwidth}
        \centering
        \includegraphics[width=\textwidth]{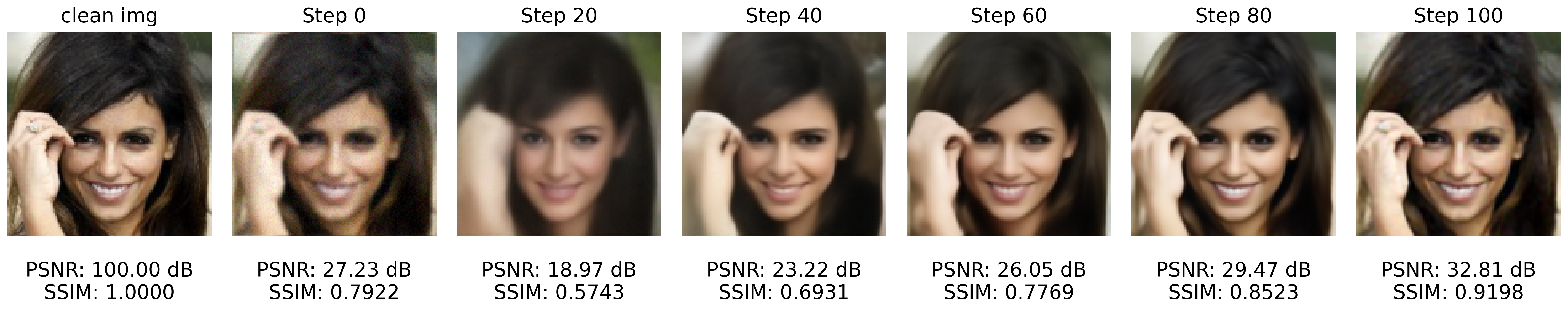}
        \caption{PnP-Flow}
    \end{subfigure}
    \\[6pt]
    \begin{subfigure}[b]{0.85\textwidth}
        \centering
        \includegraphics[width=\textwidth]{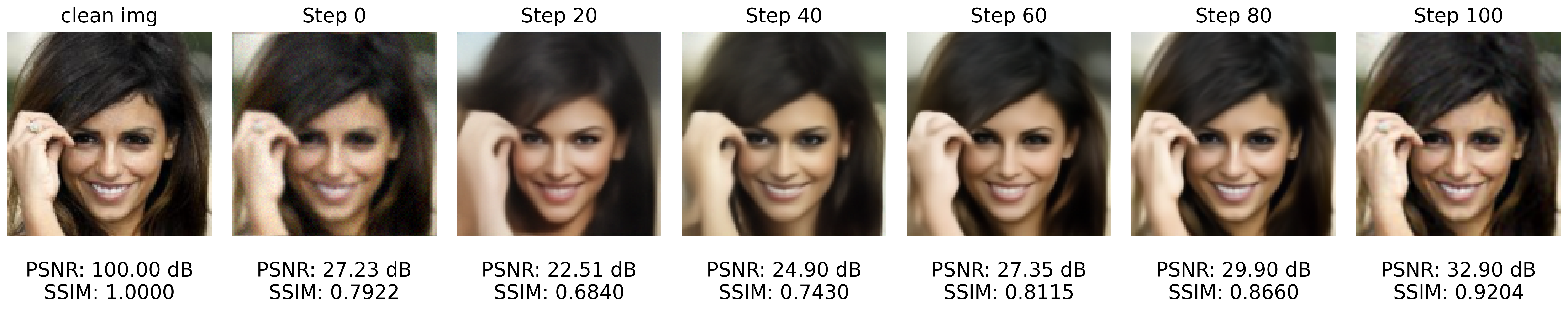}
        \caption{HP-PnP-Flow}
    \end{subfigure}
    \caption{Deblurring on CelebA ($\sigma=0.05$, $b=1.0$).}
    \label{fig:app_deblur_celeba}
\end{figure}
\FloatBarrier

\begin{figure}[htp]
    \centering
    \begin{subfigure}[b]{0.85\textwidth}
        \centering
        \includegraphics[width=\textwidth]{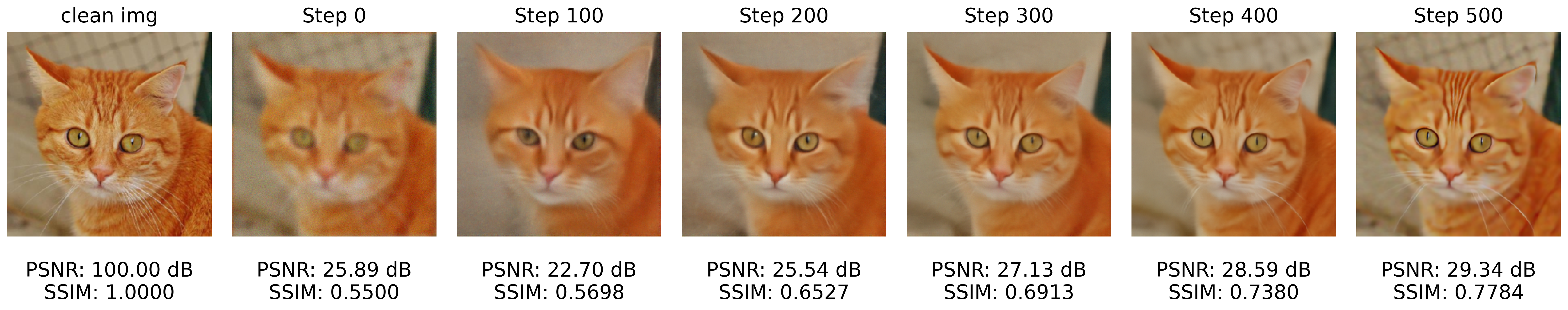}
        \caption{PnP-Flow}
    \end{subfigure}
    \\[6pt]
    \begin{subfigure}[b]{0.85\textwidth}
        \centering
        \includegraphics[width=\textwidth]{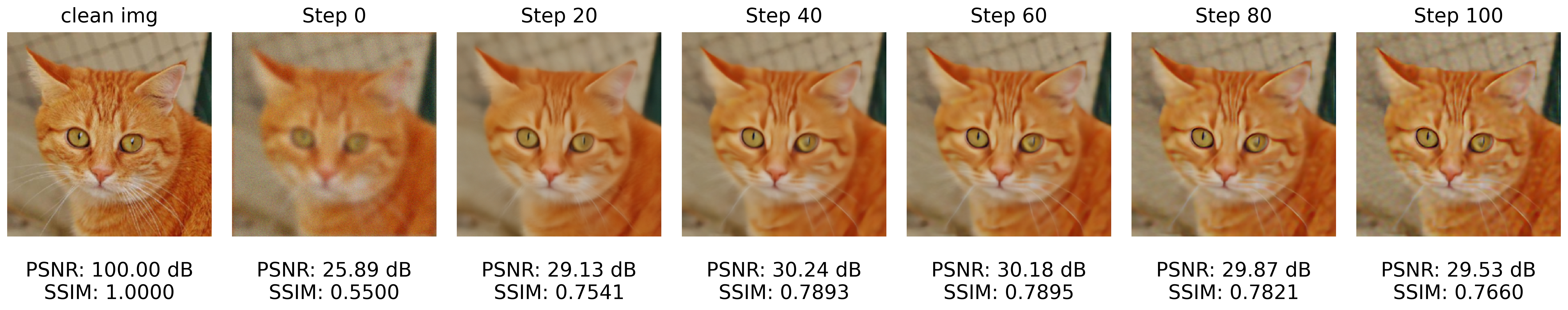}
        \caption{HP-PnP-Flow}
    \end{subfigure}
    \caption{Deblurring on AFHQ-Cat ($\sigma=0.05$, $b=3.0$).}
    \label{fig:app_deblur_afhq}
\end{figure}
\FloatBarrier

\begin{figure}[htp]
    \centering
    \begin{subfigure}[b]{0.85\textwidth}
        \centering
        \includegraphics[width=\textwidth]{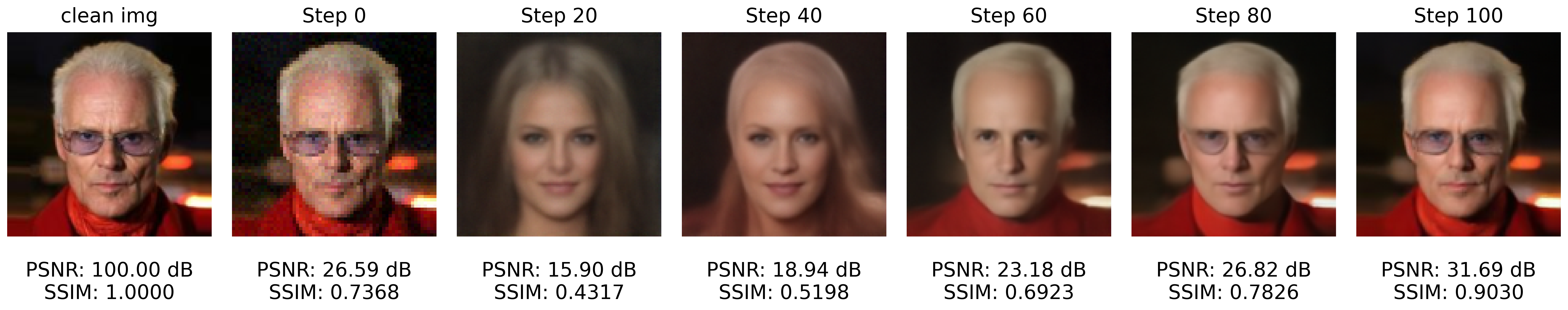}
        \caption{PnP-Flow}
    \end{subfigure}
    \\[6pt]
    \begin{subfigure}[b]{0.85\textwidth}
        \centering
        \includegraphics[width=\textwidth]{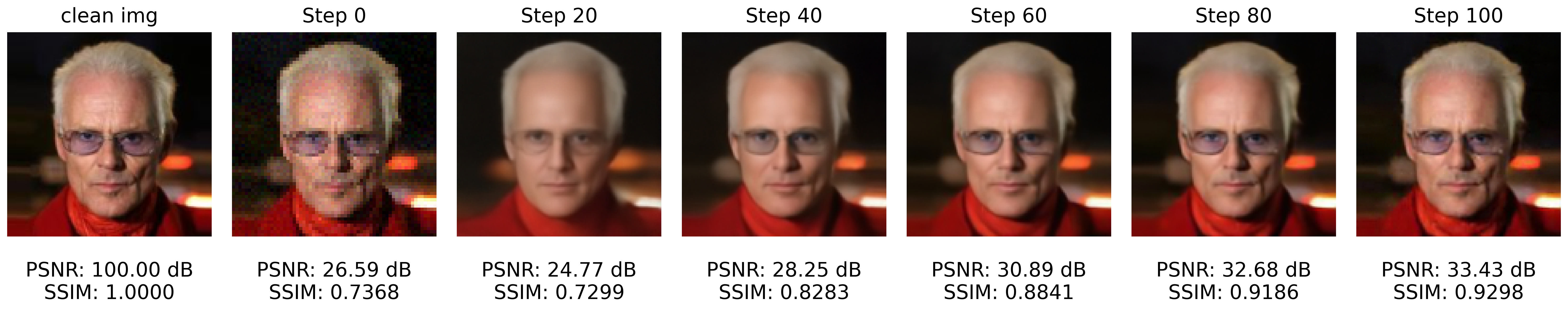}
        \caption{HP-PnP-Flow}
    \end{subfigure}
    \caption{Super-resolution on CelebA ($\times2$, $\sigma=0.05$).}
    \label{fig:app_sr_celeba}
\end{figure}
\FloatBarrier

\begin{figure}[htp]
    \centering
    \begin{subfigure}[b]{0.85\textwidth}
        \centering
        \includegraphics[width=\textwidth]{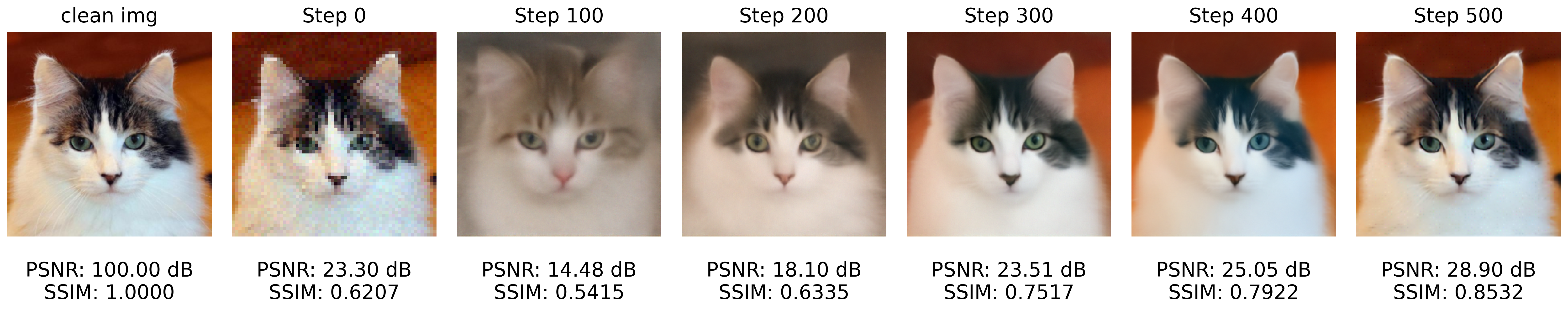}
        \caption{PnP-Flow}
    \end{subfigure}
    \\[6pt]
    \begin{subfigure}[b]{0.85\textwidth}
        \centering
        \includegraphics[width=\textwidth]{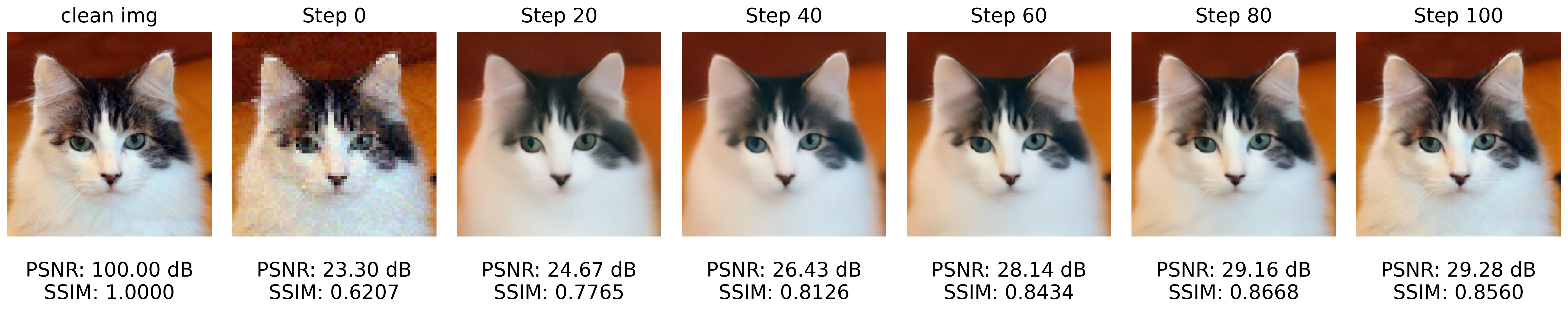}
        \caption{HP-PnP-Flow}
    \end{subfigure}
    \caption{Super-resolution on AFHQ-Cat ($\times4$, $\sigma=0.05$).}
    \label{fig:app_sr_afhq}
\end{figure}
\FloatBarrier

\begin{figure}[htp]
    \centering
    \begin{subfigure}[b]{0.85\textwidth}
        \centering
        \includegraphics[width=\textwidth]{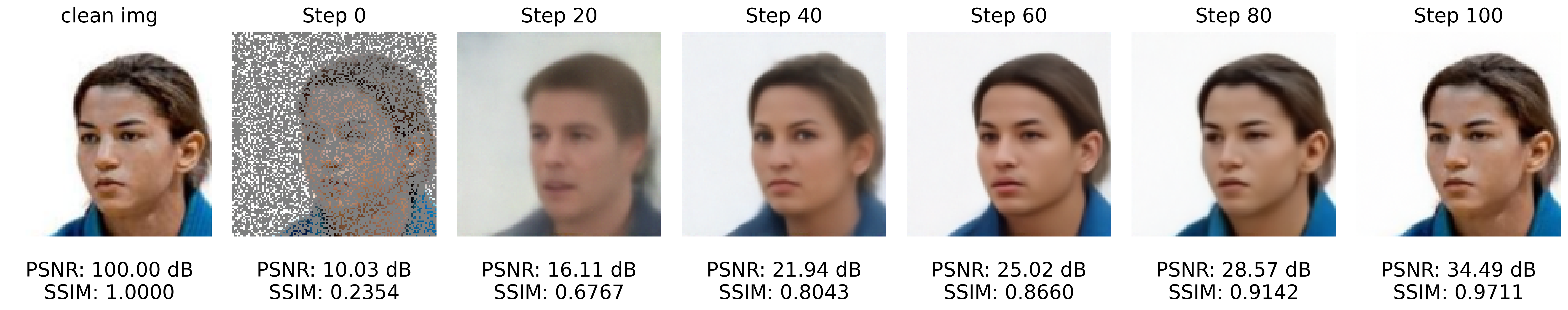}
        \caption{PnP-Flow}
    \end{subfigure}
    \\[6pt]
    \begin{subfigure}[b]{0.85\textwidth}
        \centering
        \includegraphics[width=\textwidth]{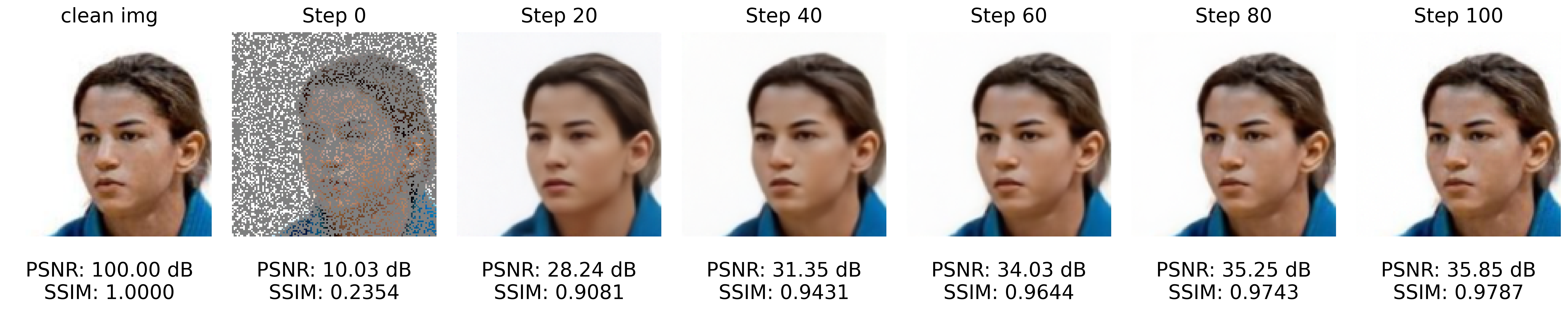}
        \caption{HP-PnP-Flow}
    \end{subfigure}
    \caption{Random inpainting on CelebA (70\% masked, $\sigma=0.01$).}
    \label{fig:app_randinp_celeba}
\end{figure}
\FloatBarrier

\begin{figure}[htp]
    \centering
    \begin{subfigure}[b]{0.85\textwidth}
        \centering
        \includegraphics[width=\textwidth]{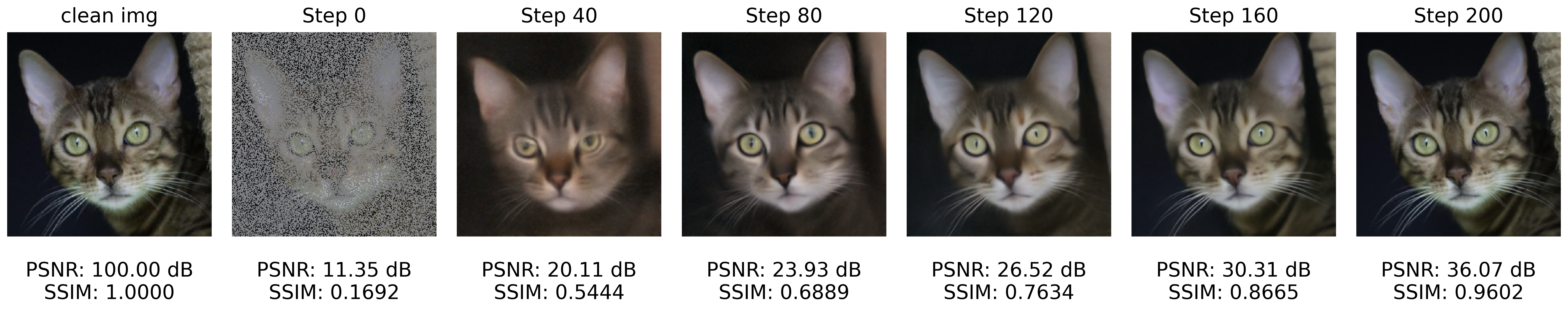}
        \caption{PnP-Flow}
    \end{subfigure}
    \\[6pt]
    \begin{subfigure}[b]{0.85\textwidth}
        \centering
        \includegraphics[width=\textwidth]{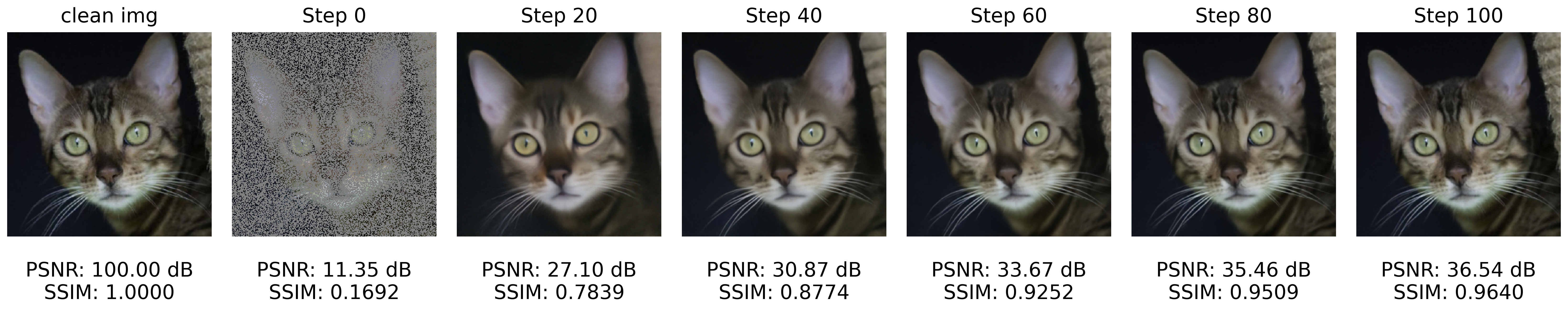}
        \caption{HP-PnP-Flow}
    \end{subfigure}
    \caption{Random inpainting on AFHQ-Cat (70\% masked, $\sigma=0.01$).}
    \label{fig:app_randinp_afhq}
\end{figure}
\FloatBarrier

\begin{figure}[htp]
    \centering
    \begin{subfigure}[b]{0.85\textwidth}
        \centering
        \includegraphics[width=\textwidth]{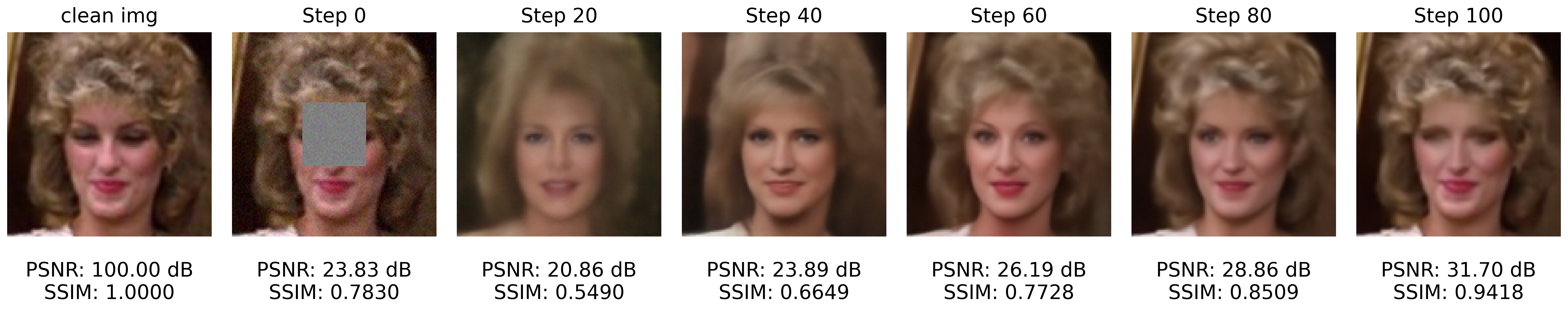}
        \caption{PnP-Flow}
    \end{subfigure}
    \\[6pt]
    \begin{subfigure}[b]{0.85\textwidth}
        \centering
        \includegraphics[width=\textwidth]{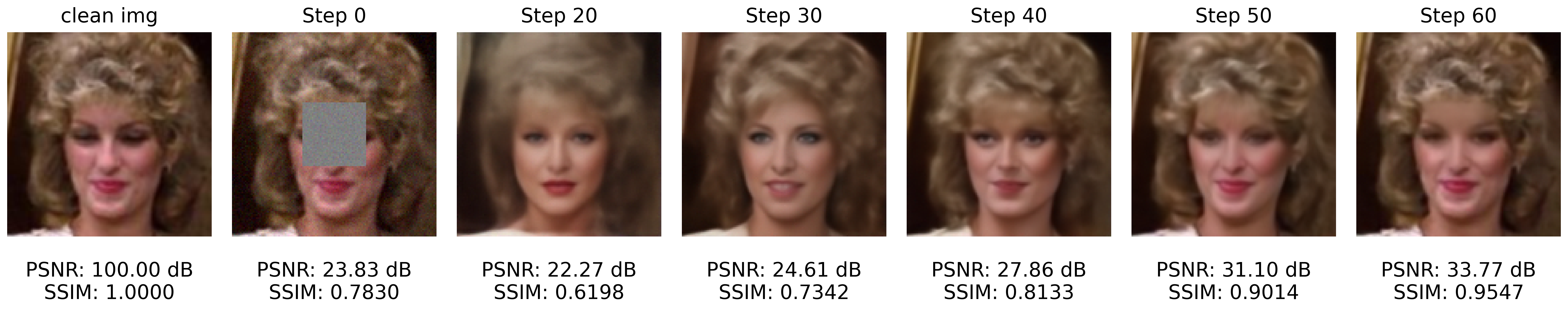}
        \caption{HP-PnP-Flow}
    \end{subfigure}
    \caption{Box inpainting on CelebA ($40\times40$, $\sigma=0.05$).}
    \label{fig:app_boxinp_celeba}
\end{figure}
\FloatBarrier

\begin{figure}[htp]
    \centering
    \begin{subfigure}[b]{0.85\textwidth}
        \centering
        \includegraphics[width=\textwidth]{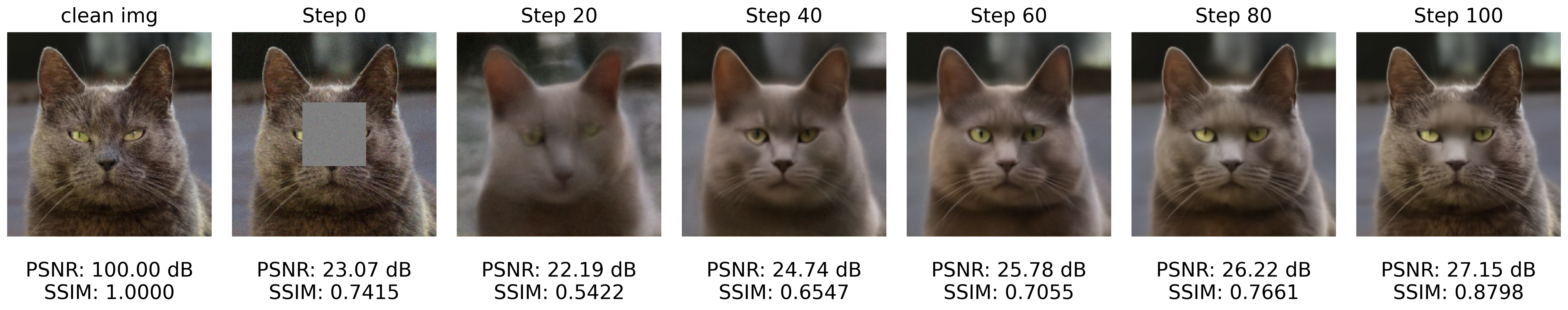}
        \caption{PnP-Flow}
    \end{subfigure}
    \\[6pt]
    \begin{subfigure}[b]{0.85\textwidth}
        \centering
        \includegraphics[width=\textwidth]{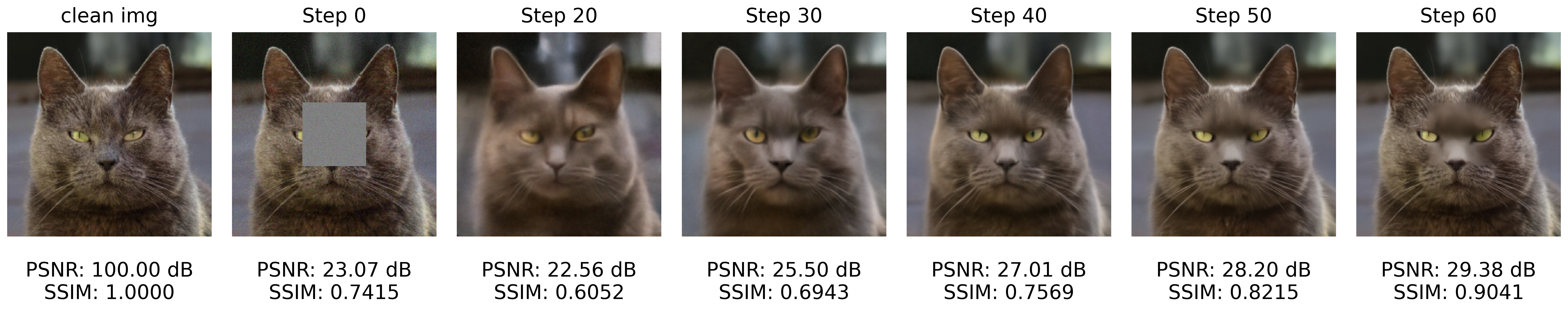}
        \caption{HP-PnP-Flow}
    \end{subfigure}
    \caption{Box inpainting on AFHQ-Cat ($80\times80$, $\sigma=0.05$).}
    \label{fig:app_boxinp_afhq}
\end{figure}
\FloatBarrier

\end{document}